\documentclass{article}
\usepackage{graphicx} % Required for inserting images

\usepackage{amsmath,amssymb,amsfonts}
\usepackage{amsthm}
\usepackage{booktabs}
\usepackage{graphicx}
\usepackage{float}
\usepackage{pgfplots}
\usepackage{xcolor}
\usepackage{subcaption}
\usepackage{hyperref}
\usepackage{natbib}
\usepackage{pgfplots}
\pgfplotsset{compat=1.18}
\usepackage[margin=1in]{geometry}
\usetikzlibrary{arrows.meta,calc,positioning,decorations.pathreplacing,shapes.geometric}

\newcommand{\E}{\mathbb{E}}

\newcommand{\phase}{\sigma}
\newcommand{\Wone}{W_1}

\newtheorem{definition}{Definition}
\newtheorem{proposition}{Proposition}
\newtheorem{corollary}{Corollary}

\definecolor{sharedColor}{RGB}{52,101,164}
\definecolor{headColor}{RGB}{117,80,145}
\definecolor{specialistColor}{RGB}{190,42,42}
\definecolor{transientColor}{RGB}{232,170,145}
\definecolor{settledColor}{RGB}{157,199,143}
\definecolor{navyColor}{RGB}{32,74,135}

\title{Single or Multiple Policies for Phase-Structured Reinforcement Learning?}
\author{
  Guilhem Loussouarn, Nancy Nayak, and Kin K. Leung \\
  \small Department of Electrical and Electronic Engineering \\
  \small Imperial College London, UK \\
}
\date{}

\begin{document}

\maketitle
\begin{abstract}
    Many reinforcement-learning (RL) problems are non-stationary yet structured and can be decomposed into phases, each with its own transition probabilities and reward functions. When the phase sequence is known, the common solution augments the state with information to satisfy the Markovian property and applies standard RL techniques. However, prior work finds that the multi-policy approach for different phases can outperform a single state-augmented policy shared among the phases, for reasons that remain unclear. In this work, we first show that the shared policy can theoretically achieve performance of any multi-policy solution. However, whether a multi-policy solution can perform better than the corresponding single shared policy in practice depends on function approximation, learning and optimization processes, as well as, for multi-policy solutions, the sample efficiency and loss of continuity from one policy to another. We propose a regime-based phase decomposition method to identify which policy can provide better performance. The method is based on consideration of the duration of transient system dynamics relative to the duration of the quasi-stationary period. Numerical experiments are conducted with different non-stationary RL problems to validate our four major hypotheses: (a) longer phase durations favor multi-policies, (b) the heterogeneity between phases increases the burden on single policy, (c) multi-policies need sufficient data for each phase, and (d) environment-specific transition dynamics between phases can affect which policy is preferable.

\end{abstract}

\section{Introduction}

Many sequential decision-making problems are non-stationary but exhibit recurring structure. Their transition dynamics or rewards may change between a finite set of phases. Within a phase, the interactions are described by a stationary Markov decision process (MDP) with phase-specific transitions and rewards; across phases, these MDPs are coupled, since the actions taken in one phase determine the state from which the next begins. We call the resulting sequence of coupled MDPs a \emph{phase-structured MDP}. Examples include daily traffic and demand cycles \citep{oliveira2017mobile,shi2018mobile,haffey2018periodic,dan2014periodic}, robotic contact modes \citep{kroemer2021review}, and changing computational workloads \citep{amvrosiadis2018workloads}. We focus on deterministic scheduled phases, allowing both periodic and arbitrary fixed non-periodic structures, in which the active phase is observed. A latent or history-dependent phase process must instead be inferred from interaction, making the problem partially observable, and we leave that case to future work.

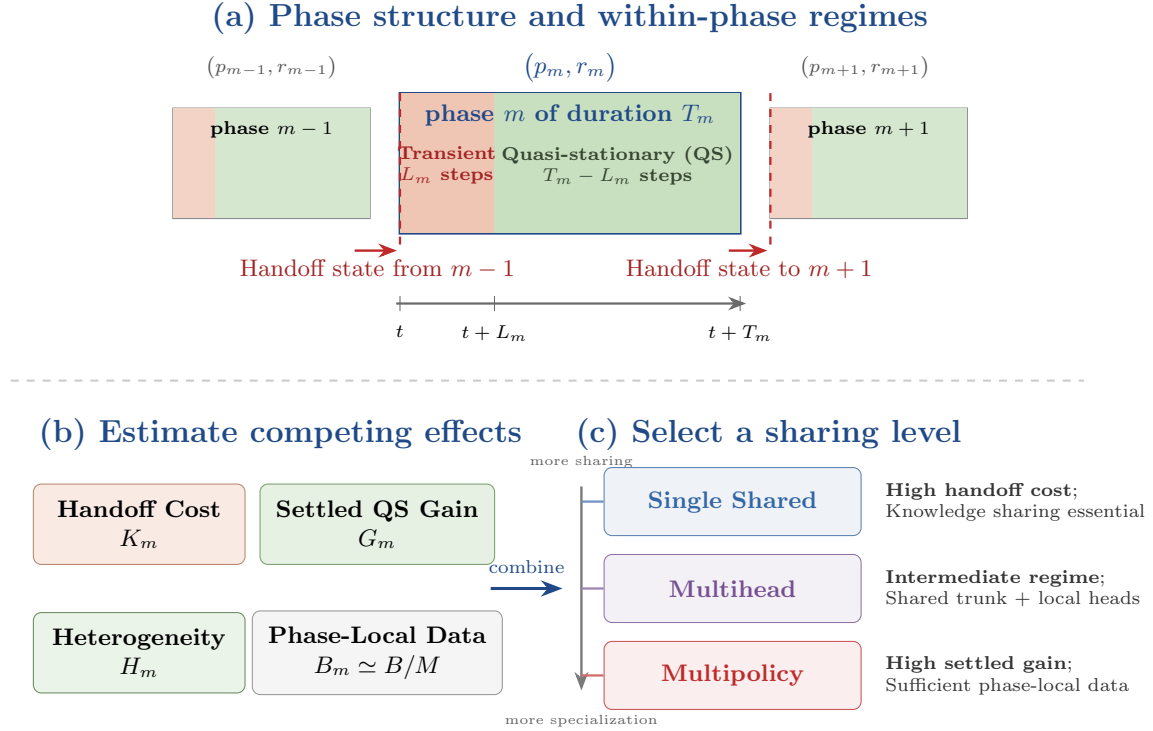
\begin{figure*}[t]
  \centering
  % Force the whole thing to shrink to the text width
    \begin{tikzpicture}[
  x=1cm, y=1cm,
  font=\sffamily,
  >=Stealth,
  title/.style={font=\large\bfseries, text=navyColor},
  box/.style={draw=black!40, rounded corners=3pt, inner sep=6pt, align=center},
  diagnostic/.style={box, minimum width=2.8cm, minimum height=1.0cm, font=\small},
  architecture/.style={box, minimum width=3.4cm, minimum height=0.9cm, font=\small\bfseries},
  flow/.style={->, line width=1.0pt, draw=black!60}
]

% ===========================================================================
% TOP LEVEL: (a) PHASE STRUCTURE AND WITHIN-PHASE REGIMES (CENTERED AT x=6.9)
% ===========================================================================
\begin{scope}[xshift=0.55cm]
\node[title, anchor=center] at (6.85, 5.2) {(a) Phase structure and within-phase regimes};

% Dynamics labels
\node[font=\scriptsize,text=black!68] at (2.90,4.55) {$\bigl(p_{m-1},r_{m-1}\bigr)$};
\node[font=\small\bfseries,text=navyColor] at (6.85,4.55) {$\bigl(p_m,r_m\bigr)$};
\node[font=\scriptsize,text=black!68] at (10.75,4.55) {$\bigl(p_{m+1},r_{m+1}\bigr)$};

% Neighboring occurrences
\draw[black!45,fill=black!2,line width=.75pt] (1.60,2.55) rectangle +(2.60,1.45);
\fill[transientColor!52] (1.60,2.55) rectangle +(.55,1.45);
\fill[settledColor!43] (2.15,2.55) rectangle +(2.05,1.45);
\node[font=\scriptsize\bfseries] at (2.90,3.72) {phase $m-1$};

\draw[navyColor,fill=black!1,line width=1.05pt] (4.60,2.35) rectangle +(4.50,1.85);
\fill[transientColor!68] (4.60,2.35) rectangle +(1.25,1.85);
\fill[settledColor!56] (5.85,2.35) rectangle +(3.25,1.85);
\node[font=\small\bfseries,text=navyColor] at (6.85,3.92) {phase $m$ of duration $T_m$};
\node[font=\scriptsize\bfseries,text=specialistColor!85!black,align=center] at (5.22,3.25) {Transient\\$L_m$ steps};
\node[font=\scriptsize\bfseries,text=settledColor!35!black,align=center] at (7.48,3.25) {Quasi-stationary (QS)\\$T_m-L_m$ steps};

\draw[black!45,fill=black!2,line width=.75pt] (9.50,2.55) rectangle +(2.60,1.45);
\fill[transientColor!52] (9.50,2.55) rectangle +(.55,1.45);
\fill[settledColor!43] (10.05,2.55) rectangle +(2.05,1.45);
\node[font=\scriptsize\bfseries] at (10.80,3.72) {phase $m+1$};

% Handoff boundary states
\draw[specialistColor,dashed,line width=.85pt] (4.60,2.20) -- (4.60,4.28);
\draw[->,specialistColor,line width=.9pt] (4.00,2.12) -- (4.56,2.12);
\node[font=\small,text=specialistColor!85!black,align=center] at (4.30,1.88) {Handoff state from $m-1$};

\draw[specialistColor,dashed,line width=.85pt] (9.50,2.20) -- (9.50,4.28);
\draw[->,specialistColor,line width=.9pt] (8.95,2.12) -- (9.46,2.12);
\node[font=\small,text=specialistColor!85!black,align=center] at (9.22,1.88) {Handoff state to $m+1$};

% Timeline axis
\draw[->,line width=.85pt,draw=black!58] (4.55,1.42) -- (9.18,1.42);
\foreach \x/\lab in {4.60/$t$,5.85/$t+L_m$,9.10/$t+T_m$} {
  \draw[black!58] (\x,1.32) -- (\x,1.52);
  \node[font=\scriptsize,below=2pt] at (\x,1.32) {\lab};
}
\end{scope}

% Divider between Level 1 and Level 2
\draw[black!20, line width=0.8pt, dashed] (0, 0.4) -- (14.2, 0.4);

% ===========================================================================
% BOTTOM LEVEL (LEFT): (b) ESTIMATE COMPETING EFFECTS
% ===========================================================================
\node[title, anchor=west] at (0.25, -0.3) {(b) Estimate competing effects};

\node[diagnostic, fill=transientColor!25, draw=transientColor!80!black] (handoff) at (1.7, -1.5)
  {\textbf{Handoff Cost}\\$K_m$};

\node[diagnostic, fill=settledColor!25, draw=settledColor!80!black] (hetero) at (4.85, -1.5)
  {\textbf{Settled QS Gain}\\$G_m$};

\node[diagnostic, fill=settledColor!20, draw=settledColor!70!black] (gain) at (1.7, -3.2)
  {\textbf{Heterogeneity }\\$H_m$};

\node[diagnostic, fill=black!4] (budget) at (4.85, -3.2)
  {\textbf{Phase-Local Data}\\$B_m \simeq B/M$};

% ===========================================================================
% BOTTOM LEVEL (RIGHT): (c) SELECT A SHARING LEVEL
% ===========================================================================
\node[title, anchor=west] at (7.35, -0.3) {(c) Select a sharing level};

% One connector summarizes panel (b) before entering the architecture scale.
% Keeping it inside the inter-panel gap prevents it from crossing either panel.
\draw[->,line width=1.15pt,draw=navyColor] (6.35, -2.35) -- (7.30, -2.35)
  node[midway,above=2pt,font=\scriptsize,text=navyColor] {combine};

\draw[->,draw=black!52,line width=.9pt] (7.55, -1.0) -- (7.55, -3.7);
\node[font=\tiny,text=black!62,anchor=south] at (7.55, -0.87) {more sharing};
\node[font=\tiny,text=black!62,anchor=north,align=center] at (7.55, -3.9) {more specialization};

\node[architecture, draw=sharedColor!80, text=sharedColor, fill=sharedColor!8]
  (shared) at (9.55, -1.2) {Single Shared};

\node[architecture, draw=headColor!80, text=headColor, fill=headColor!8]
  (head) at (9.55, -2.35) {Multihead};

\node[architecture, draw=specialistColor!80, text=specialistColor, fill=specialistColor!8]
  (multi) at (9.55, -3.5) {Multipolicy};

\draw[sharedColor!70,line width=.85pt] (7.55, -1.2) -- (shared.west);
\draw[headColor!70,line width=.85pt] (7.55, -2.35) -- (head.west);
\draw[specialistColor!70,line width=.85pt] (7.55, -3.5) -- (multi.west);

% Architecture Descriptions
\node[font=\scriptsize, align=left, anchor=west, text=black!80] at (11.45, -1.2)
  {\textbf{High handoff cost};\\Knowledge sharing essential};

\node[font=\scriptsize, align=left, anchor=west, text=black!80] at (11.45, -2.35)
  {\textbf{Intermediate regime};\\Shared trunk + local heads};

\node[font=\scriptsize, align=left, anchor=west, text=black!80] at (11.45, -3.5)
  {\textbf{High settled gain};\\Sufficient phase-local data};

\end{tikzpicture}
 \caption{Conceptual framework for architecture selection in phase-structured MDPs. 
  \textbf{(a)} Each phase $m$ of duration $T_m$ exhibits local dynamics $(p_m, r_m)$ and can be split into a transient regime of $L_m$ steps (dominated by handoff states inherited from the preceding phase) and a quasi-stationary (QS) regime of $T_m - L_m$ steps. 
  \textbf{(b)} The net trade-off between specialization and sharing depends on four diagnostic mechanisms: handoff cost $K_m$ who favours coordinated policies; settled gain $G_m$ representing the advantage of specialization learning; task heterogeneity $H_m$, which represents the difference between faces and can affect shared learning; and phase-local data budget $B_m$, which predominantly affects multipolicy learning. 
  \textbf{(c)} Combining these diagnostics places the problem on a continuum from full parameter sharing (\textit{Single Shared}) to full specialization (\textit{Multipolicy}), with modular architectures (\textit{Multihead}) occupying the intermediate regime.}
  \label{fig:intro_roadmap}
\end{figure*}

State augmentation restores a time-homogeneous MDP only when the appended statistic is sufficient for both the physical dynamics and the evolution of the schedule \citep{puterman1994mdp}. For a periodic schedule, the phase index and local clock are sufficient. More general schedules may require additional scheduler information, such as global time or the position within a known sequence. This resolves how the process is represented, but not how policy parameters should be shared across phases. We distinguish three primary architectures along a parameter-sharing spectrum (Figure~\ref{fig:intro_roadmap}c): a \emph{shared policy} uses one phase-conditioned network for all phases; a \emph{multi-policy} architecture uses an independent specialist for each phase; and a \emph{multi-head policy} shares a representation but uses phase-specific output heads.

Recent work on periodic network resource allocation found that multipolicy PPO can outperform a shared state-augmented policy \citep{chen2025pmdp}. If the active phase is already observed by the shared policy, why should separate policies help? At the level of unrestricted policy classes, they should not: a shared phase-conditioned policy can reproduce any collection of phase-specific policies, so the advantage of specialization cannot follow from expressivity alone, nor does it imply that specialization is always preferable. How much information and how many parameters are shared nonetheless affects learning performance, sample efficiency, and computational cost. We therefore ask: \emph{when should a policy share parameters across phases, and when should it specialize?}

\textbf{Central claim.}
Under finite learning resources, specialization is favoured when its phase-local learning benefit exceeds the costs of sample fragmentation and boundary handoff. Sharing is favoured when transfer, data pooling, and cross-phase coordination dominate.

We relate this trade-off to three properties of a phase-structured MDP. First, \emph{phase heterogeneity} determines how different the behaviours required in each phase are. Second, the \emph{training budget} determines whether every specialist receives enough data to learn its phase. Third, the \emph{within-phase regime structure} determines the importance of coordination. Immediately after a phase change, performance may depend strongly on the state inherited from the previous phase. We call this the transient regime. Once this dependence has weakened, the phase enters a quasi-stationary regime in which phase-local specialization may be more beneficial. The duration of the transient and its sensitivity to the inherited state determine the handoff cost of independent policies.

We study these mechanisms under observable phase structure, leaving latent phase inference outside our scope. Across our three environments, changing the phase count, duration, heterogeneity, budget, or boundary dependence can reverse the ranking of shared and specialized architectures. Our contributions are fivefold. First, we formalize phase-structured MDPs and identify the scheduler state that restores a time-homogeneous MDP. Second, we show that the shared phase-conditioned policy class is equivalent to the multipolicy class, so any performance difference between them is a learning effect rather than a representational one. Third, we split each phase occurrence into transient and quasi-stationary regimes and decompose the learned return gap into a settled specialization gain and a transient handoff cost, yielding a phase-wise rule for when specialization pays and tying the gain to phase heterogeneity and the per-phase sample budget. Fourth, we derive diagnostics linking architecture choice to phase heterogeneity, phase-local sample budgets, transient duration, and handoff sensitivity. Finally, we test these predictions with single shared, multihead, and multipolicy PPO on HAPS coverage control, PhasePendulum, and PhaseGridWorld, and probe the proposed mechanisms through capacity controls, gradient measurements, warm starts, distillation, and controlled handoff interventions.

 \textbf{Paper roadmap.}
The remainder of this paper follows the conceptual flow outlined in Figure~\ref{fig:intro_roadmap}. Section~\ref{sec:related_work} situates our approach within the broader literature on non-stationary RL, multi-task learning, and capacity. Section~\ref{sec:phase-structured-mdps} formalizes phase-structured MDPs and the policy class equivalence.   Section~\ref{sec:regime-framework} proposes the regime decomposition into transient and quasi-stationary phases (Figure~\ref{fig:intro_roadmap}a) and derives the diagnostic metrics quantifying handoff costs, settled gains, and sample budgets (Figure~\ref{fig:intro_roadmap}b). Finally, Section~\ref{sec:experiments} validates these predictions experimentally across three environments, demonstrating how competing mechanisms propose an appropriate level of sharing of architectural components across different architectures(Figure~\ref{fig:intro_roadmap}c).
\section{Related Work}
\label{sec:related_work}
\textbf{Structured time-varying MDPs.} Prior work studies MDPs whose transition and reward functions change according to an underlying structure. Periodic MDPs assume that these changes repeat on a fixed cycle and commonly include the period index in the state to use state-augmented policies \citep{riis1965periodic, veugen1983periodicity, podtsykin1991periodic, osogami2014periodic, hu2014periodic, aniket2023periodic}. Other models consider stochastic or hidden changes, such as switching non-stationary MDPs governed by a latent Markov chain \citep{amiri2026switching}. We use \emph{phase-structured MDP} to refer broadly to the cases when the change structure is deterministic and observed. Sec.~\ref{sec:phase-structured-mdps} provides the formal definition. As for phase-structured RL, it represents the full range of RL methods applied to phase-structured MDPs.

\textbf{Periodic multi-policy reinforcement learning.}
\citep{chen2025pmdp} propose multi-policy RL for periodic network resource allocation and report gains over single-policy baselines. A related multi-policy architecture has been applied to high-altitude platform system (HAPS) control in a phase-structured telecommunications environment \citep{loussouarn2026haps} with similar improvements. In both cases, no explanations on the gap between state-augmented policy and multi-policy performances are given. We take these empirical observations as motivation and study the architectural choice between parameter sharing and phase specialization across controlled domains. Our HAPS benchmark retains the
coverage-control component of the latter environment, enabling systematic ablations of the factors governing this choice.

\textbf{Multi-task learning, modularity, and transfer.}
Multi-task learning aims to solve related tasks jointly by sharing parameters or representations \citep{caruana1997multitask}. Phase-structured RL admits a similar interpretation: the phase acts as a task descriptor, a state-augmented policy shares all parameters across phases, and a multipolicy uses independent phase specialists. Sharing can improve data efficiency through transfer, but it can also cause negative transfer when tasks require incompatible updates. PCGrad and CAGrad modify task gradients to reduce such conflicts
\citep{yu2020pcgrad,liu2021cagrad}; we evaluate them as baselines and use phase-specific gradient geometry as a local diagnostic.

Modular and transfer methods address limitations of complete sharing or
separation. Mixture-of-experts models divide complex mappings among specialized
modules \citep{jacobs1991adaptive,jordan1994hierarchical}. Divide-and-conquer RL uses local policies to simplify difficult optimization, while distillation, Distral, progressive networks, and generalist--specialist learning transfer or preserve knowledge across policies and tasks \citep{ghosh2017divide,rusu2015distillation,teh2017distral, rusu2016progressive,jia2022generalist}. These ideas motivate our multihead and warm-start baselines, which we define in Sec.~\ref{sec:experiments}.

The analogy remains incomplete because standard multi-task formulations
generally sample tasks separately \citep{zeng2021decentralized}. Phases are instead consecutive parts of one trajectory: the policy used in one phase determines the state passed to the next. Existing multi-task methods therefore address sharing and transfer, but not necessarily cross-phase coordination.

\textbf{Capacity and irrelevant features.}
Network size and training compute can affect deep-RL performance, although their benefits depend on the algorithm, environment, and data regime
\citep{hilton2023scaling,schwarzer2023bbf,rybkin2025value,fu2025compute}. Moreover, observations that are irrelevant to the optimal decision can still increase the difficulty of finite-sample learning with function approximation \citep{jong2004irrelevant}. These observations motivate our capacity and distractor controls, which test whether multi-policy gains can be explained by parameter count or by the difficulty of extracting phase information.

\section{Phase-Structured MDPs}
\label{sec:phase-structured-mdps} 

We begin by recalling the standard MDP formulation. An MDP is a tuple $\mathcal{P} = (\mathcal{S},\mathcal{A},p,r,\mu_0,\gamma)$, where $\mathcal{S}$ and $\mathcal{A}$ are the state and action spaces, $p(\cdot\mid s, a)$ is the state-transition probability distribution (or transition kernel), $r(s, a)$ is the expected one-step reward, $\mu_0$ is the initial-state distribution, and $\gamma\in(0,1]$ is the discount factor.

\textbf{Phases and schedules.}
Let $\mathcal{M}=\{0,\ldots,M-1\}$ be a finite set of phases. Each phase $m\in\mathcal{M}$ identifies a stationary mode of the environment, described by a transition probability distribution $p_m$ and reward function $r_m$. A phase is therefore a model of the local dynamics and objective. A deterministic phase schedule is a map $\sigma:\mathbb{N}\rightarrow\mathcal{M},$ and $\sigma(t)$ denotes the phase active at time $t$. For example, consider the sequence of phases $m$ in Table \ref{tab:generalsch} consisting of four phases $m\in\mathcal{M}=\{0,1,2,3\}$ that may occur multiple times in an episode. The total duration of the occurrence of each phase is arbitrary, i.e., the duration of the first occurrence $(j=0)$ of phase $m=0$ is two time steps, i.e. $T_{m,j}|_{m=0,j=0}=2$, whereas the second occurrence $(j=1)$ of phase $m=0$ is only one time step, i.e., $T_{m,j}|_{m=0,j=1}=1$. $\tau(t)$ denotes the local time step within a phase.

\begin{table}[h]
\centering
\begin{tabular}{l|ccccccccccccccccccc}
\hline
$t$ & 0 & 1 & 2 & 3 & 4 & 5 & 6 & 7 & 8 & 9 & 10 & 11 & 12 & 13 & 14 & 15 & 16 &\dots\\
\hline
$m=\sigma(t)$ & 0 & 0 & 1 & 1 & 1 & 2 & 2 & 3 & 3 & 3 & 3 & 0 & 2 & 2 & 0 & 0 & 0  &\dots\\
$\tau(t)$ & 0 & 1 & 0 & 1 & 2 & 0 & 1 & 0 & 1 & 2 & 3 & 0 & 0 & 1 & 0 & 1 & 2 &\dots\\
\hline
\end{tabular}
\caption{A phase schedule with phase $m\in \mathcal{M}=\{0,1,2,3\}$.}
\label{tab:generalsch}
\end{table}

\begin{definition}[Phase-structured MDP]
\label{def:psmdp}
A deterministic phase-structured MDP is a tuple \\ $\mathcal{P}_{\sigma}
=
\left(
\mathcal{S},
\mathcal{A},
\{p_m\}_{m\in\mathcal{M}},
\{r_m\}_{m\in\mathcal{M}},
\mu_0,
\gamma,
\sigma
\right).$ At time $t$, the active phase is $\sigma(t)$ and starting from an initial state $s_0\sim\mu_0$, the process evolves at time $t$ according to\[
s_{t+1}\sim p_{\phase(t)}(\cdot\mid s_t,a_t),
\qquad
r_t=r_{\phase(t)}(s_t,a_t).
\]

For a finite horizon $N$, the objective is to find an optimal policy $\pi^*$ such that:
\[
\pi^\star
\in
\arg\max_{\pi} J_N(\pi), \quad
\text{ where } J_N(\pi)
=
\mathbb{E}_{\mu_0,\pi,\sigma}
\left[
\sum_{t=0}^{N-1}\gamma^t
r_{\sigma(t)}(s_t,a_t)
\right].
\]
The expectation is taken over the initial state, policy actions, and state
transitions under the fixed schedule $\sigma$. For an infinite horizon, we
take $N=\infty$ and assume $\gamma<1$ and bounded rewards.
\end{definition}
\textbf{Scheduler state and Markovianity.}
The state $s_t$ alone may not be Markov. For the same state, the transition or reward can differ depending on the active phase $\sigma(t)$. 
Let $z_t = (\sigma(t), i_t)$ denote the scheduler state, where $\sigma(t)$ is the phase index and $i_t$ is auxiliary schedule context. We call $z_t$ \emph{sufficient} if transition dynamics depend on history solely through $(s_t, a_t, z_t)$. In practice, observing only $\sigma(t)$ is insufficient because phase dynamics vary internally over time. While appending global time ($i_t = t$) technically restores Markovianity, it transforms the state space into an acyclic directed graph that prevents data re-use across cycles. To preserve both Markovian sufficiency and sample efficiency, $i_t$ must be mapped to a compact quotient space; for periodic phases of length $T_m$, the within-phase clock $\tau(t) = t \bmod T_m$ suffices ($z_t = (\sigma(t), \tau(t))$; see App.~\ref{app:markov-sufficiency}).

\begin{proposition}[Scheduler-state augmentation]
\label{prop:scheduler-state-markov}
If $z_t$ is sufficient, then the augmented state $\widetilde{s}_t=(s_t,z_t)$ defines a time-homogeneous MDP and recovers Markovianity.
\end{proposition}

The formal Markovian sufficiency conditions and a periodic phase schedule (Corr.~\ref{cor:phase-clock-markov}) are detailed in App.~\ref{app:markov-sufficiency}.
% extended non-periodic examples, 
The proofs of Cor.~\ref{cor:phase-clock-markov} and Prop.~\ref{prop:scheduler-state-markov} are detailed in App.~\ref{app:proofs}).

% The periodic result does not automatically extend to general schedules. 

\textbf{Phase-local models and handoffs.}
For a phase $m$ and an entry-state distribution $\mu$, define the stationary
phase-local model
\begin{equation}
\mathcal{P}_m(\mu)
=
(\mathcal{S},\mathcal{A},p_m^{\pi},r_m,\mu,\gamma),
\label{eq:phase-local-mdp}
\end{equation}
which describes the local dynamics and rewards. In a phase-structured MDP, however, $\mu$ is not fixed, i.e., the distribution $\mu_{m,j}^{\pi}$ of the state at the start of the $j^{th}$ occurrence of phase $m$ is induced by the actions $\pi$ in the preceding phases. The state transition probabilities $p_m^{\pi}$ therefore are also dependent on previous policies $\pi$ such that $p_m^{\pi}=\sum_{a\in \mathcal{A}}p_m(s_{t+1}|a_t, s_t)\pi(a_t|s_t)$. The phases therefore cannot in general be treated as independent. This cross-phase dependence is the handoff effect studied in Sec.~\ref{sec:regime-framework}.

\textbf{Shared policy and specialized policies.}
Since the augmented state $\Tilde{s}_t$ yields a time-homogeneous MDP, standard RL methods can learn a single policy $\pi$ across all phases. We call any such policy a \textit{shared policy}. Instead, \citep{chen2025pmdp}  train a separate \textit{specialized policy} for each phase, an architecture shown to outperform shared policies in phase-structured problems. Yet, specialized policies hold no representational advantage, since any set of them can be expressed by a single shared policy.

\begin{proposition}[Shared--multipolicy equivalence]
\label{prop:policy-class-equivalence}
There is a bijection between shared policies $\pi^\text{sh}(\cdot\mid s, i, m)$ and an ensemble of specialized policies $\boldsymbol{\pi}^\textbf{sp} =(\pi_0(\cdot\mid s, i),\ldots,\pi_{M-1}(\cdot\mid s, i))$ expressed through the gated map $\Gamma$ such that
\begin{equation}
\begin{aligned}
\pi^{\rm sp}(a\mid s,i ,m)
= \Gamma(\boldsymbol{\pi}^\textbf{sp})(a\mid s, i, m) =
\pi_m(a\mid s, i),\\
\pi^\text{sh}(a\mid s, i, m)=\pi_m(a\mid s, i) = \pi^\text{sp}(a\mid s, i, m).
\end{aligned}
\label{eq:shared-multipolicy-bijection}
\end{equation}
It is always possible to represent a shared policy with an ensemble of specialized policies and vice versa.
\end{proposition}

This is a policy-class equivalence: with unrestricted function classes, shared and specialized policies attain the same optimal return. However, it says nothing about where each architecture reaches when trained. With finite networks, finite data, and gradient-based optimization, the two can converge to different solutions, which is why either may outperform the other in practice. The next section characterizes when specialization is preferable to sharing, and when the reverse holds. Illustrations and proof of Prop.~\ref{prop:policy-class-equivalence} can be found inside App.~\ref{app:proofs}

\section{When to Share and When to Specialize}
\label{sec:regime-framework}
Each phase-local model $\mathcal{P}_m(\mu)$ introduced in \eqref{eq:phase-local-mdp} is stationary within its phase horizon but inherits its initial entry distribution $\mu$ from the preceding phases. We now ask how long that inheritance matters within a phase, and how the answer, together with the heterogeneity of consecutive phases, determines whether a shared or a specialized policy is better suited.

\subsection{Transient and Quasi-Stationary Regimes}
At a phase boundary, an inherited condition such as position, velocity, network load, or other initial-state condition can strongly affect subsequent performance. We call the period during which performance remains sensitive to this inherited state the \emph{transient regime}, and the remaining steps, once the sensitivity has become small,  the \emph{quasi-stationary regime}, in which behaviour is governed mainly by the current phase's dynamics and objective. Quasi-stationarity does not require convergence to an invariant state distribution. It requires only that the effect of the entry state on the return quantities under study has become small.

Let $L_{m,j}\leq T_{m,j}$ denote the duration of the transient regime of the $j^{th}$ occurrence of phase $m$, leaving a quasi-stationary regime of duration $T_{m,j}-L_{m,j}$ time steps. Occurrences of the same phase may differ in $L_{m,j}$ because they inherit different entry distributions, and in slowly mixing or disconnected systems the transient may span the whole phase ($L_{m,j}=T_{m,j}$), so periodicity alone does not guarantee recovery before the next boundary. We estimate $L_{m,j}$ from phase-age behaviour and properties of the environment. App.~\ref{sec:transient-stationary} gives the formal recovery definition and the occurrence-level aggregation.

The two regimes expose different architectural trade-offs. In the quasi-stationary regime, independent training may make a phase-local mapping easier to approximate and optimize than a joint mapping across all phases. In the transient regime, policies remain coupled through the state passed across the boundary. A shared policy trained on complete trajectories can account for this continuation value, whereas independently trained phase-local specialized policies may not. Sharing also pools data and reuses representations, whereas a balanced specialized policy receives only about $B/M$ of a total budget $B$. We now quantify these effects.

\subsection{Settled Gains and Transient Handoff Costs}
The regime decomposition explains why the preferred architecture may vary across phase-structured problems; we now make that comparison quantitative. We measure the gain phase-specific learning obtains in the quasi-stationary regime against the cost specialization incurs in the transient regime, which together form the architecture diagnostic developed in the next subsection.

\textbf{A comparison between learned shared and learned specialized policies.}
% Fix learned shared and specialist controllers 
Let the \emph{learned} shared and the \emph{learned} specialized policies are denoted by $\widehat\pi^{\rm sh}$ and $\widehat\pi^{\rm sp}$, respectively. For occurrence $j$ of phase $m$, with a transient regime duration of $L_{m,j}\leq T_{m,j}$, define 
% $ Z_{m,j}=\sum_{t=0}^{T_{m,j}-1}\gamma^t$, and 
$\lambda_{m,j}=\sum_{t=0}^{L_{m,j}-1}\gamma^t/\sum_{t=0}^{T_{m,j}-1}\gamma^t$.
Thus $\lambda_{m,j}$ is the fraction of the occurrence's discounted mass assigned to the transient regime. Let $\bar J_{m,j,A}^{\rm tr}$ and $\bar J_{m,j,A}^{\rm qs}$ denote the discounted average rewards of architecture $A\in\{\mathrm{sh},\mathrm{sp}\}$ for single shared and specialized policies, in the transient and quasi-stationary windows. For each occurrence, let
\begin{equation}
G_{m,j}
=
\bar J_{m,j,\mathrm{sp}}^{\rm qs}
-
\bar J_{m,j,\mathrm{sh}}^{\rm qs},
\qquad
K_{m,j}
=
\bar J_{m,j,\mathrm{sh}}^{\rm tr}
-
\bar J_{m,j,\mathrm{sp}}^{\rm tr}.
\label{eq:main-regime-terms}
\end{equation}
A positive $G_{m,j}$ is a quasi-stationary specialization gain, while a positive $K_{m,j}$ is a transient advantage for sharing. 
% For readability, we use the occurrence index $j$ only to introduce the construction. 
For readability, we omit the index $j$, and the remainder of the paper uses phase-level quantities $\lambda_m$, $G_m$, and $K_m$, obtained by aggregating over occurrences of phase $m$. This does not assume that all occurrences have identical entry states or transient lengths. The advantage of using learned specialized policies over a learned single shared policy is given by
\begin{equation}
\Delta\bar J
:=
\bar J(\widehat\pi^{\rm sp})
-
\bar J(\widehat\pi^{\rm sh})
=
\sum_{m=0}^{M-1}
\omega_m
\left[
(1-\lambda_m)G_m-\lambda_mK_m
\right],
\label{eq:main-regime-decomposition}
\end{equation}
where $\omega_m$ is the fraction of discounted evaluation mass belonging to phase $m$.  The derivation of the above formulation is given in detail in App.~\ref{sec:transient-stationary}. Specialized policies for phase-local models win over a single shared policy in phase $m$ exactly when its quasi-stationary advantage exceeds its transient deficit, i.e., 
\begin{equation}
(1-\lambda_m)G_{m}>\lambda_mK_{m}.
\label{eq:phasewise-transient-rule}
\end{equation}

\textbf{Transient hand off cost $K_m$.}
The entry law $\mu_{m,A}$ is produced by the policy acting in the phase before $m$, and no architecture removes this coupling. Standard multi-policy training optimizes each specialized policy on phase-local experience, so a specialist may miss how its late actions shift the entry state of the next phase, whereas a shared policy trained on complete trajectories receives that cross-phase credit. Accordingly, $K_m$ is large when achievable return is sensitive to the entry state and recovery is slow; and small when entry distributions match, dynamics contract quickly, or phases reset independently.

\textbf{Quasi-stationary gain $G_m$.}
In the quasi-stationary regime, phase $m$ resembles its phase-local MDP, a specialized policy may approximate the local optimum $\pi_m^\star$ better than a finite shared policy trained across all phases, giving $G_m>0$. This is a finite learning effect, not an expressivity advantage. Since, better phase-local policies need not achieve a better global return once handoffs are included, we isolate the settled comparison by evaluating both architectures on the same settled reference distribution $\mu_m^{\mathrm{ref}}$.

Let $\bar J_m^\star$ and $\Pi_m^\star$ denote the optimal quasi-stationary return in phase $m$ and its set of optimal policies. We measure the heterogeneity of phase $m$ by the loss incurred when policies optimized for other phases are transferred to it:
\begin{equation}
H_m
=
\sum_{k\neq m}a_{mk}
\left[
\bar J_m^\star
-
\sup_{\pi\in\Pi_k^\star}
\bar J_m^{\mathrm{qs}}(\pi;\mu_m^{\mathrm{ref}})
\right],
\qquad
\sum_{k\neq m}a_{mk}=1,
\label{eq:phase-heterogeneity}
\end{equation}
so $H_m$ is large when behaviour optimal elsewhere performs poorly in phase $m$; the supremum avoids dependence on an arbitrary choice of policy when a phase has several optima. For the shared and specialized architectures, define the phase-local learning gaps 
\begin{align}
\varepsilon_{m,\mathrm{sh}}(B;H_m)
=
\bar J_m^\star-
\mathbb E\!\left[
\bar J_m^{\mathrm{qs}}
(\widehat\pi_{\mathrm{sh},B};\mu_m^{\mathrm{ref}})
\right],\nonumber
\quad
\varepsilon_{m,\mathrm{sp}}(B_m)
=
\bar J_m^\star-
\mathbb E\!\left[
\bar J_m^{\mathrm{qs}}
(\widehat\pi_{m,\mathrm{sp},B_m};\mu_m^{\mathrm{ref}})
\right].
\label{eq:phase-local-learning-gaps}
\end{align}
which absorb finite capacity, data, exploration, optimization, and architectural bias. In particular, the shared $\varepsilon_{m,\mathrm{sp}}$ accounts for joint optimization challenges such as multi-task gradient dilution (App.~\ref{app:local-optimization}), whereas independent specialists avoid cross-phase gradient interference entirely. Writing $H_m$ as an argument of $\varepsilon_{m,\mathrm{sh}}$ emphasizes that the shared learner must fit phase-$m$ behaviour jointly with other phases, without assuming any particular functional dependence. The expected settled specialization gain is then exactly
\begin{equation}
G_{m,\mathrm{ref}}^{\mathrm{qs}}
=
\varepsilon_{m,\mathrm{sh}}(B;H_m)
-
\varepsilon_{m,\mathrm{sp}}(B_m).
\label{eq:phase-local-gain-accounting}
\end{equation}

Heterogeneity $H_m$ widens $\varepsilon_{m,\mathrm{sh}}$, while data pooling transferable representations across phases narrows it. Specialization eliminates this joint capacity bottleneck, but each specialist receives only a fraction $\zeta_m \in (0,1)$ of the total sample budget $B$, i.e., $B_m = \zeta_m B$, or $B/M$ under uniform allocation $\zeta_m = 1/M$. A specialist ensemble therefore achieves $G_{m,\mathrm{ref}}^{\mathrm{qs}} > 0$ only when the local sample allocation $B_m = \zeta_m B$ avoids data starvation by exceeding the phase-local sample threshold $B_{0,m}(\delta)$ required for a target suboptimality tolerance $\delta > 0$ (App.~\ref{sec:capacity-budget}, Prop.~\ref{prop:adequate-specialist-budget}).

\subsection{From Regime Quantities to Architecture Choice}

The advantage presented in \eqref{eq:main-regime-decomposition} becomes predictive when the regime quantities are bounded, either by environment analysis or by short pilot runs.
% ; \eqref{eq:main-handoff-cost-bound} places $K_m$ in $[K_{m,\mathrm{ref}} - E_m^{\rm tr}, \, K_{m,\mathrm{ref}} + E_m^{\rm tr}]$. 
Substituting bounds $\underline G_m \leq G_m \leq \overline G_m$ and $\underline K_m \leq K_m \leq \overline K_m$, yields certified bounds on relative performance:
\begin{equation}
\underline R := \sum_m \omega_m \left[ (1-\lambda_m)\underline G_m - \lambda_m \overline K_m \right] \leq \Delta\bar J \leq \sum_m \omega_m \left[ (1-\lambda_m)\overline G_m - \lambda_m \underline K_m \right] =: \overline R.
\label{eq:architecture-score-bounds}
\end{equation}
Hence $\underline R>0$ certifies specialization, i.e., $\bar J(\widehat\pi^{\rm sp})>\bar J(\widehat\pi^{\rm sh})$, and $\overline R<0$ certifies sharing. An interval containing zero leaves the choice unresolved and favours an intermediate architecture (i.e., a multihead network, although \eqref{eq:architecture-score-bounds} compares only the shared and specialized endpoints and does not certify that option. For a specific phase $m$ with $\underline G_m > 0$ and $\overline K_m \geq 0$, specialization yields a positive local lower bound whenever the transient fraction satisfies:
\begin{equation}
\lambda_m < \frac{\underline G_m}{\underline G_m + \overline K_m}.
\label{eq:phasewise-predictive-threshold}
\end{equation}  
so the settled gain must persist long enough to amortize the transient handoff cost. Table~\ref{tab:architecture-regimes} summarizes the resulting regimes. The specialized bound dominates when the settled gain outweighs the handoff and
budget penalties, and the shared bound when transfer, scarce phase-local data, or boundary sensitivity is decisive; multihead suits the unresolved region because it preserves shared features while allowing phase-specific action mappings. 

\begin{table}[t]
\centering
\caption{Architecture selection regimes derived from the transient--quasi-stationary framework.}
\label{tab:architecture-regimes}
\small
\begin{tabular}{p{0.29\linewidth}p{0.27\linewidth}p{0.34\linewidth}}
\toprule
Regime & Governing Condition & Recommended Architecture \\
\midrule
Insufficient $B_m$ & Specialist data starvation & \textbf{Shared}; multihead if outputs differ \\
Enough $B_m$, low $H_m$ & High representation transfer & \textbf{Shared} or \textbf{Multihead} \\
Enough $B_m$, high $H_m$, low $\overline K_m$ & High local learning burden & \textbf{Specialized Multi-Policy} \\
Enough $B_m$, high $H_m$ and high $\overline K_m$ & Severe boundary handoff cost & \textbf{Shared} or \textbf{Coordinated Specialists} \\
Uncertain $[\underline R, \overline R]$ & Uncertified bounds & \textbf{Multihead} / Model Selection \\
\bottomrule
\end{tabular}
\end{table}

\long\def\regimediagnosticstable{
\begin{table}[t]
\centering
\caption{Regime quantities and empirical diagnostics.}
\label{tab:regime-diagnostics}
\small
\begin{tabular}{p{0.18\linewidth}p{0.30\linewidth}p{0.42\linewidth}}
\toprule
Quantity & Meaning & Empirical estimate \\
\midrule
$\lambda_m$ & mean transient exposure & occurrence-level recovery times or phase-age splits \\
$G_m$ & settled specialization gain & common-entry evaluation; \eqref{eq:estimable-regime-bounds} \\
$K_m$ & sharing's handoff advantage & early common-entry evaluation and entry-law shift \\
$H_m$ & return-scaled heterogeneity & cross-phase policy regret \\
$V_g,C_g$ & update heterogeneity & phase-restricted gradients at one shared checkpoint \\
$\zeta_mB/B_{0,m}$ & budget adequacy & phase-local learning curves \\
$\overline K_m$ & handoff-risk bound & controlled entry states and reachability \\
\bottomrule
\end{tabular}
\end{table}
}

\section{Experiments and Results}
\label{sec:experiments}

To test the regime model of Sec.~\ref{sec:regime-framework}, we evaluate PPO \citep{schulman2017ppo} on three phase-structured environments: high-altitude platform systems (HAPS) coverage optimization \cite{loussouarn2026haps}, PhasePendulum, and PhaseGridWorld, all described in detail in App.~\ref{app:hap-details}, App.~\ref{app:pendulum-details}, and App.~\ref{app:gridworld-details}, respectively. In the HAPS network, the idea is to find the optimal position of the HAPS given that the user density and traffic vary over time but in a periodic manner. In particular, different phases inside a period are differentiated by different user densities and traffic loads. The entire service region of the HAPS network is divided into multiple cells; only a few selected cells act as hotspots where the user density and traffic are higher than the rest of the cells. PhasePendulum environment consists of the standard torque-controlled pendulum dynamics while changing the target angle $g_m$ between phases. PhaseGridWorld environment is a $9\times9$ grid with four deterministic actions: up, down, left, and right. Each phase has its own goal and obstacle map. All environments are periodic, with the number of phases and cycles varying by experiment. We compare three architectures: a single policy with the phase and clock augmentation, a multipolicy of $M$ specialists, and a multihead policy with a shared trunk and $M$ heads. All methods use the same PPO hyperparameters unless stated otherwise, and we report raw reward per environment step throughout. We organize the results around four hypotheses (H1 to H4) obtained from Sec.~\ref{sec:regime-framework} and presented in Table \ref{tab:architecture-regimes}.

Our three main methods span different degrees of parameter sharing. \emph{Single+clock} is a single actor--critic taking physical state, phase index, and within-phase clock as inputs. \emph{Multihead} receives the physical state and phase index and combines a shared actor body and a shared critic with $M$ phase-specific linear action heads; at each step the head indexed by the active phase emits the action, so the shared body is updated from every transition while each head is updated only from transitions of its own phase. The phase code is also supplied to the body, allowing it to form phase-conditioned features rather than only phase-agnostic ones. \emph{Multipolicy} uses $M$ independent actor--critics, the active phase selecting which one acts and is updated; its input is the physical state, whereas the clock-aware variant is denoted \emph{multipolicy+clock}. The phase code is not supplied here, since selecting the network already identifies the phase. More details on the methods can be found in App.~\ref{app:architecture-training}

\begin{figure}[ht]
\centering
\begin{minipage}{0.48\linewidth}
\centering
\begin{tikzpicture}
\begin{axis}[
width=\linewidth,
height=1.35in,
xlabel={$T$},
ylabel={Reward/step},
xmode=log,
log basis x={10},
xmin=80,xmax=2500,
ymin=15,ymax=50,
xtick={100,200,500,2000},
xticklabels={$100$,$200$,$500$,$2000$},
tick label style={font=\scriptsize},
label style={font=\scriptsize},
grid=both,
grid style={gray!15}
]
\addplot+[mark=o,thick,error bars/.cd,y dir=both,y explicit]
table[x=T,y=single_phase_tau,y error=single_phase_tau_ci95,col sep=comma] {figures_data/hap_t_sweep.csv};
\addplot+[mark=square,thick,error bars/.cd,y dir=both,y explicit]
table[x=T,y=multihead,y error=multihead_ci95,col sep=comma] {figures_data/hap_t_sweep.csv};
\addplot+[mark=triangle,thick,error bars/.cd,y dir=both,y explicit]
table[x=T,y=multipolicy,y error=multipolicy_ci95,col sep=comma] {figures_data/hap_t_sweep.csv};
\end{axis}
\end{tikzpicture}

\end{minipage}\hfill
\begin{minipage}{0.48\linewidth}
\centering
\begin{tikzpicture}
\begin{axis}[
width=\linewidth,
height=1.35in,
xlabel={$M$},
ylabel={Reward/step},
xmin=1.5,xmax=20.5,
ymin=18,ymax=46,
xtick={2,5,10,20},
legend to name=hap-regime-legend,
legend columns=-1,
legend style={draw=none,font=\scriptsize,
  /tikz/every even column/.append style={column sep=0.8em}},
tick label style={font=\scriptsize},
label style={font=\scriptsize},
grid=both,
grid style={gray!15}
]
\addplot+[mark=o,thick,error bars/.cd,y dir=both,y explicit]
table[x=M,y=single_phase_tau,y error=single_phase_tau_ci95,col sep=comma] {figures_data/hap_m_sweep.csv};
\addlegendentry{Single+clock}
\addplot+[mark=square,thick,error bars/.cd,y dir=both,y explicit]
table[x=M,y=multihead,y error=multihead_ci95,col sep=comma] {figures_data/hap_m_sweep.csv};
\addlegendentry{Multihead}
\addplot+[mark=triangle,thick,error bars/.cd,y dir=both,y explicit]
table[x=M,y=multipolicy,y error=multipolicy_ci95,col sep=comma] {figures_data/hap_m_sweep.csv};
\addlegendentry{Multipolicy}
\end{axis}
\end{tikzpicture}
\end{minipage}

\vspace{0.2em}
\ref{hap-regime-legend}
\caption{Main HAPS regime evidence. Left: at fixed $M=20$, phase-clock sharing leads at
$T=100$, multihead and multipolicy are close at $T=200$, and multipolicy leads
at $T\in\{500,2000\}$. Right: increasing the number of phases preserves the specialist advantage at $T=500$ and exposes the degradation of full sharing at $M=20$. This is consistent with, but does not prove, amortization of boundary costs.}
\label{fig:hap-regime}
\end{figure}
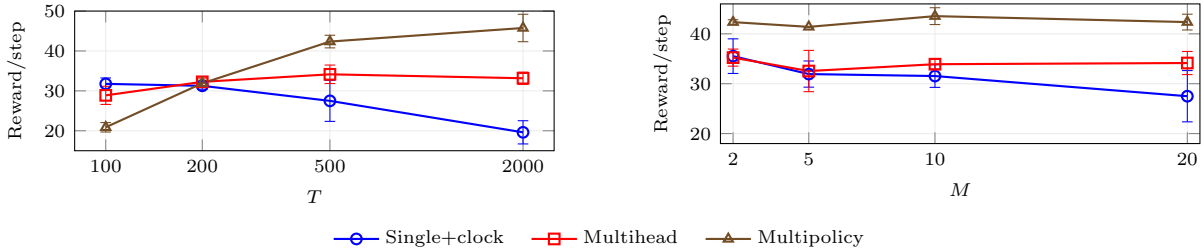

\subsection{H1: Longer Phases Favour Specialization with multi-policy}
In the HAPS environment, the transient regime is governed by platform relocation and by user persistence across phase boundaries. A direct transit takes at most $100$ steps when both spatial coordinates of the HAPS change maximally, and mean user lifetimes of $5$ to $25$ steps are much shorter, so we take $L_m\approx 100$. Reward differences arise mainly when one or two platforms fail to reach a high-density zone, so any difference between two policies should be in the same order. Following this observation, $G_m\approx K_m$ reduces \eqref{eq:phasewise-predictive-threshold} to $\lambda_m<1/2$, and with $\gamma\approx 1$ and $\lambda_m\approx L_m/T_m$ this predicts a crossover at $T\approx 2L_m=200$. We treat $200$ as a diagnostic horizon rather than a strict value of $T_m$, since arrivals are stochastic and learned trajectories need not be straight. The prediction is that shared policies win below this dwell time, where the transient dominates, and specialized policies above it, where the quasi-stationary regime does. Fig.~\ref{fig:hap-regime} (left), at $M=20$, shows exactly this: multipolicy reward increases with $T$, Single+Clock decreases, the two curves cross near $T=200$, and multihead is largely insensitive to $T$, confirming its position as an in-between solution. PhasePendulum reproduces this pattern (Table~\ref{tab:pendulum-m-t-controls}). For both $M=5$ and $M=10$, the single policy leads at $T\in\{100,200\}$ while multi-policy attains the best mean at $T=500$; multi-head performs very close to multi-policy and even surpasses in some cases. Longer dwell times therefore favour specialization wherever a quasi-stationary regime is physically reachable.

\subsection{H2: Behavioural Diversity Increases the Burden on Single Shared Policy}
Heterogeneity can be increased by adding distinct phases, which leaves each specialist's problem unchanged while requiring the shared policy to represent more behaviours. Fig.~\ref{fig:hap-regime}(right) reports the performance of the HAPS for different phase counts at $T=500$: Single+Clock degrades as $M$ grows while multi-policy is flat, so the effect is a degradation of complete sharing rather than a benefit of more phases. The snapshots in Fig.~\ref{fig:hap-selected-snapshots} in the appendix show the mechanism, with the shared policy failing to reach the hotspots of some phases. Widening the shared network or using four, six or eight layers does not close the gap, and a network with eight layers performs worse than one with six layers (Table~\ref{tab:hap-capacity-depth}). The shared policy therefore does not fail for lack of parameters: adding them does not help and can hurt. Its limitation lies in training rather than in representational capacity, as Prop.~\ref{prop:policy-class-equivalence} implies, so the remaining gap is an architectural advantage for multipolicy. PhasePendulum varies behavioral diversity more directly, setting target angles to $g_m(h)=h\,g_m^{\max}$ (App.~\ref{app:pendulum-details}). Table~\ref{tab:het} reports $h$ from $0$ to $1$. The single policy is best at $h=0$, where all phases are identical, and at $h=0.25$; from $h=0.5$ the ordering reverses, with multihead first and multipolicy second. Heterogeneity therefore widens the shared learner's gap as in \eqref{eq:phase-local-gain-accounting}, far enough to reverse the advantage of sharing. In PhaseGridWorld (Fig.~\ref{fig:gridworld-boundary-intervention}), where handoff costs are negligible, shared and specialized architectures perform comparably when adjacent phases share an optimal policy, while multipolicy clearly outperforms the shared baseline when they require distinct ones. Heterogeneity alone therefore suffices to separate the architectures.

We test whether the single-policy degradation observed above is caused by destructive shared updates in Table~\ref{tab:corrected-gradient-diagnostic} and Fig.~\ref{fig:grad-main}. The orthogonal-null alignment statistic $A_g$ presented (App.~\ref{app:local-optimization}) remains close to one for any $M$. Thus the main reason behind learning degradation is not the destructive interference. It explains why gradient surgery solutions such as PCGrad/CAGrad did not close the gap with multipolicy in our environments, as gradient destruction is not the biggest burden (Table~\ref{tab:hap-main}).
The cosine similarity falls from $0.675$ to $0.195$ for complete sharing as $M$ grows from $2$ to $20$ (Table~\ref{tab:corrected-gradient-diagnostic}), close to the $1/\sqrt M$ orthogonal benchmark derived in App.~\ref{app:local-optimization}. 
 
With more distinct phases in the shared policy, the cosine similarity is going to decrease for larger $M$. We call this effect \emph{gradient dilution}. Together with the capacity controls, it shifts the explanation toward finite approximation and shared optimization burden.

\subsection{H3: Specialization Requires Sufficient Phase-Local Data}
Under a fixed total interaction budget $B$, a balanced multipolicy allocates only $B_m\simeq B/M$ samples to specialist $m$. Shared PPO instead updates one model from all $B$ interactions. We call failure to reach a high-return local solution within $B_m$ the \emph{learning starvation}; this concerns finite-budget performance, not optimizer convergence.PhasePendulum at $h=0$ isolates this effect because all ten phases are identical and any performance gap must reflect optimization and sample efficiency rather than representation capacity. With a total budget of $B=5\text{M}$ timesteps, each redundant specialist receives a partitioned allocation of only $B_m \approx 500\text{k}$ interactions. As shown in Fig.~\ref{fig:pendulum-h0-learning-specialists}, Single+clock strongly outperforms Multipolicy because local specialists suffer from severe sample starvation ($B_m < B_{0,m}$ App.~\ref{sec:capacity-budget}), failing to reach local convergence. Fig.~\ref{fig:pendulum-h0-learning} confirms this threshold behavior: Multipolicy requires a total budget scaling to $B \ge 10\text{M}$ for each local allocation to cross $B_{0,m}$ and approach the asymptotic performance level of Single+clock. A single-phase GridWorld calibration provides consistent evidence (Tab.~\ref{tab:gridworld-one-phase-budget}): the clock-aware shared model reaches the target learning threshold within $500\text{k}$ interactions across all three seeds, whereas the multipolicy model requires a total $750\text{k}$ steps to reliably match it. It is important to note the intrinsic advantage of adding clock information to any solution (Sec.~\ref{sec:gridworld-results}). Also, Tab.~\ref{tab:gridworld-specialist-controls} illustrates that increasing the global budget can invert the performance ranking between single and multipolicy models.

\subsection{H4: Handoff Dependence Can Reverse the Ranking}
Studying handoff dependence is challenging in general because transient regimes are rarely cleanly delimited in complex tasks. We address this by using controlled GridWorld environments where the underlying causality is fully understood. Because independent specialists optimize phase-local objectives from local resets, they lack awareness of downstream continuation values, incurring a boundary coordination loss. Full-trajectory shared policies can instead sacrifice immediate phase return to ensure the overall highest return (Prop.~\ref{prop:two-goal-handoff}). 

Fig.~\ref{fig:gridworld-handoff-schematic} formalizes this trade-off in a two-phase GridWorld. Selecting an immediate goal ($g_{00}$, reward $1.0$) leaves the agent trapped behind an impassable wall in Phase 1, whereas selecting a safe goal ($g_{01}$, reward $\rho < 1.0$) preserves Phase 1 reachability. As shown in Fig.~\ref{fig:gridworld-handoff} ($\rho \to 0.9$), Single+clock and Multihead adapt to choose $g_{01}$ in almost all episodes. In contrast, independent specialists (Multipolicy) remain myopic, selecting $g_{00}$ in $90\%$ of episodes and failing Phase 1 entirely. Here, the transient regime length equals the full phase duration ($L_m = T_m$), completely annihilating specialization gains due to uncoordinated local policies.

A second GridWorld experiment (App.~\ref{app:gridworld-details}) evaluates the transient regime across two dimensions: (1) varying phase goals and wall layouts to adjust phase heterogeneity and coordination demands, and (2) comparing physical state carryover against resampling from reachable entry states (Fig.~\ref{fig:gridworld-handoff-topologies}). As shown in Fig.~\ref{fig:gridworld-boundary-intervention}, physical carryover severely penalizes Multipolicy in the irrecoverable Closed Trap ($L_m = T_m$). Resetting phase entry to a common reachable state eliminates boundary costs ($K_m \approx 0$), reversing to a positive gap for Multipolicy. Conversely, in the recoverable Bottleneck ($L_m < T_m$), Multipolicy dominates under both physical carryover and reachable resets.

\section{Conclusion}
We have investigated when phase-structured reinforcement learning should employ a single shared policy versus phase-specific multiple-policies. Given sufficient scheduler information, single and multiple policies have the same representational capability, so any difference in performance must come from the environment and the learning dynamics rather than from representation. The proposed regime decomposition of the phases helps to understand the difference between single and multiple policies under constraints such as limited samples, phase-heterogeneity, and the mutual coupling between phases. We observe that sharing is favoured when phases exhibit behavioral overlap, learning budgets are constrained, or boundary-state coordination consumes a large fraction of the phase duration. Specialization is favored when phases are heterogeneous, per-specialist data volume is sufficient, and the quasi-stationary gain exceeds the handoff cost. Partially shared designs, e.g., multihead networks offer a intermediate solution between these two extremes.

Future work should look beyond the binary choice between single shared and specialized multi policy models. The key challenge lies in designing adaptive architectures that share transferable features, specialize where local optimization is hard. and propagate continuation values across phase boundaries. Extending this framework to unknown phase structures and develop a coordinated multiple policies provide a promising foundation for a new class of phase-structured RL algorithms.
% \section*{Acknowledgments}
% This work was supported by [Sponsor/Grant Name, e.g., the Engineering and Physical Sciences Research Council (EPSRC) under Grant No. XXXX], [Sponsor 2, e.g., European Research Council (ERC)], and [Company/Industry Partner Name]. Guilhem Loussouarn was supported by [Scholarship/Fellowship Name if applicable]. The authors also acknowledge the computational resources provided by [Imperial College London HPC / Cluster Name].
\bibliographystyle{unsrt}
\bibliography{references.bib}

\newpage
\appendix
\section{Appendix}
% \section*{Appendix Overview}
The appendix is organized in four parts. Sec.~\ref{app:supporting-theory} defines the optimization, heterogeneity, capacity, budget, and handoff quantities used by the regime framework, and then assembles the predictive certificate. Sec.~\ref{app:proofs} groups the complete derivations by their role in the argument. Sec.~\ref{app:experiments} specifies the environments, training protocols, and reporting conventions. Sec.~\ref{app:extended-results} contains additional quantitative results and qualitative diagnostics.

\subsection{Supporting Mechanisms and Bounds}
\label{app:supporting-theory}

This section provides the definitions and results used by the different sections of this work. 

%\subsection{Settled Learning: Heterogeneity, Capacity, and Budget}
%\label{sec:heterogeneity}

%Phases can differ in their transition laws, rewards, useful policies, and training updates. Model differences alone do not determine architecture choice: two distinct phase models may require the same behavior. The main text therefore defines $H_m$ through the return lost when behavior suited to another phase is reused in phase $m$. Exact optimal policies are unavailable in our deep-RL domains, so learned-policy cross-evaluation is only a proxy and can include residual learning error \citep{barreto2017successor,barreto2019transfer}.

%Large behavioral heterogeneity can increase the burden on a finite shared network, but it does not guarantee that specialization wins: an expressive shared model may implement an inexpensive phase gate. Conversely, specialists split the sample budget and remain coupled through boundary states. We next separate finite-class capacity from finite-budget learning. Gradient geometry is treated independently in Appendix~\ref{app:local-optimization} because different behaviors need not imply harmful aggregate updates.

\subsubsection{Schedule Sufficiency and Markovian Augmentation}
\label{app:markov-sufficiency}

To model phase transitions formally within the MDP framework, the physical state $s_t \in \mathcal{S}$ must be augmented with scheduler information $z_t \in \mathcal{Z}$. Let $z_t = (\sigma(t), i_t)$, where $\sigma(t) \in \{0, \dots, M-1\}$ is the active phase index and $i_t$ represents auxiliary schedule context (e.g., elapsed phase time or global step index).

\paragraph{Formal Markovian Sufficiency Condition.}
We define the scheduler state $z_t$ as \emph{sufficient} if the joint transition dynamics satisfy:
\begin{equation}
\Pr(s_{t+1} \in A, z_{t+1} \in B \mid s_{0:t}, a_{0:t}, z_{0:t}) = P_{\sigma(t)}(A \mid s_t, a_t) \, q(B \mid z_t),
\label{eq:markov-sufficiency-condition}
\end{equation}
for all measurable sets $A \subseteq \mathcal{S}$ and $B \subseteq \mathcal{Z}$, where $P_m(\cdot \mid s, a)$ is the transition kernel of phase $m$, and $q(\cdot \mid z_t)$ is the exogenous schedule update kernel. When Equation~\eqref{eq:markov-sufficiency-condition} holds, the augmented state $\tilde s_t = (s_t, z_t)$ forms a time-homogeneous Markov Decision Process $\widetilde{\mathcal{M}} = (\mathcal{S} \times \mathcal{Z}, \mathcal{A}, \widetilde P, \widetilde r, \gamma)$.
To preserve sample efficiency, state augmentation must minimize the scheduler state space $\mathcal{Z}$. For periodic or cyclic phase structures, folding global time $t$ into a compact within-phase clock $\tau(t) = t \bmod T_m$ quotient space preserves state recurrence across cycles while satisfying Equation~\eqref{eq:markov-sufficiency-condition}:
\begin{corollary}[Periodic Phase-Clock Augmentation]
\label{cor:phase-clock-markov}
Suppose phases follow a deterministic cycle with phase durations $(T_0, T_1, \dots, T_{M-1})$. Let $T_{\mathrm{period}} = \sum_{m=0}^{M-1} T_m$ denote the total cycle length, and let $S_m = \sum_{k=0}^{m-1} T_k$ denote the entry step offset of phase $m$ within a cycle (with $S_0 = 0$). For global step $t$, let $\bar t = t \bmod T_{\mathrm{period}}$. The phase index $\sigma(t)$ and within-phase clock $\tau(t)$ are uniquely defined by:
\begin{equation}
\sigma(t) = m \quad \text{such that } S_m \le \bar t < S_m + T_m, \qquad \tau(t) = \bar t - S_{\sigma(t)}.
\end{equation}
Then $z_t = (\sigma(t), \tau(t))$ is sufficient. Consequently, the augmented state $(s_t, \sigma(t), \tau(t))$ is Markovian, time-homogeneous, and compact.
\end{corollary}

If the same phase $m$ occurs multiple times within a single macro-cycle with different durations or successor phases, the auxiliary schedule context $i_t$ must additionally include the macro-cycle phase occurrence index $j$ to maintain unique state identification. Table~\ref{tab:periodic} illustrates the unique sequence of pairs $(\sigma(t), \tau(t))$ for a periodic 4-phase schedule.

\begin{table}[h]
\centering
\caption{\textbf{Periodic phase schedule and unique state augmentation pairs $(\sigma(t), \tau(t))$.}}
\label{tab:periodic}
\small
\begin{tabular}{l|ccccccccccc|cccc}
\toprule
$t$ & 0 & 1 & 2 & 3 & 4 & 5 & 6 & 7 & 8 & 9 & 10 & 11 & 12 & 13 & $\dots$ \\
\midrule
$m = \sigma(t)$ & 0 & 0 & 1 & 1 & 1 & 2 & 2 & 3 & 3 & 3 & 3 & 0 & 0 & 1 & $\dots$ \\
$\tau(t)$ & 0 & 1 & 0 & 1 & 2 & 0 & 1 & 0 & 1 & 2 & 3 & 0 & 1 & 0 & $\dots$ \\
\bottomrule
\end{tabular}
\end{table}

The formal proof of Corollary~\ref{cor:phase-clock-markov} is provided in Appendix~\ref{app:proofs}

\subsubsection{A Transient--Quasi-Stationary View}
\label{sec:transient-stationary}

We now split each phase into an early transient period and a later settled quasi-stationary period. Let $p_{m}^{\pi}(t)$ be the state-transition probabilities induced by policy $\pi$ at time step $t$. Let $\mu_{m,j}^{\pi}$ be the entry distribution at occurrence $j$ produced by the earlier policies, and let $\rho_m^{\rm ref}$ be a possible near-invariant law. Let $W_1$ be the Wasserstein-1 distance that is used to measure how far the current distribution is from the settled one in units of distance on the state space. For $\varepsilon>0$, define the distance at time step $t$ inside the phase $m$
\begin{equation}
d_{m,j}^{\pi}(t)=
\Wone\!\left(
\mu_{m,j}^{\pi}p_{m}^{\pi}(t),
\rho_m^{\rm ref}p_{m}^{\pi}(t)
\right).
\label{eq:recovery-discrepancy}
\end{equation}
Let us define the \emph{$\varepsilon$-occurrence-recovery time} $\ell_{m,j}^\pi(\varepsilon)$ as the first step after which this distance remains at or below  $\varepsilon>0$ for all later steps such that 
\begin{equation}
\ell_{m,j}^\pi(\varepsilon)=
\inf\left\{\ell\in\{0,\ldots,T_{m,j}\}:
d_{m,j}^{\pi}(t)\leq\varepsilon
\text{ for all }t=\ell,\ldots,T_{m,j}\right\}.
\label{eq:recovery-time}
\end{equation}

If no such $\ell$ exists, we set $\ell_m^\pi(\varepsilon)=T_{m,j}$. The $\varepsilon$-occurrence-recovery time depends on both the policy and the entry law, and therefore is not a fixed property of the phase. 

For any comparison window $0\leq L_{m,j}\leq T_{m,j}$, the phase-local return  $J_{m,j}(\pi;\mu_{m,j}^\pi)$ for phase $m$ is given by
\begin{equation}
\begin{aligned}
    J_{m,j}(\pi;\mu_{m,j}^\pi)
=J_{m,j}^{\rm tr}(\pi;L_{m,j})+&J_{m,j}^{\rm qs}(\pi;L_{m,j}),\\ \text{where, }
J_{m,j}^{\rm tr}(\pi;L_{m,j})
=\E\!\left[\sum_{j=0}^{L_{m-1,j}}\gamma^j r_{m,j}(s_j,a_j)\right],& \text{ and }
J_{m,j}^{\rm qs}(\pi;L_{m,j})
=\E\!\left[\sum_{j=L_{m,j}}^{T_{m-1,j}}\gamma^j r_{m,j}(s_j,a_j)\right].
\label{eq:transient-return-split}
\end{aligned}
\end{equation}

Let $\pi_m^{\rm sp}$ and $\pi^{\rm sh}$ be the learned specialized and shared policies, each evaluated under the cyclic entry law it induces, so that the differences in handoff are included in the comparison. For a split point $L_{m,j}$ between transient and quasi-stationary phase inside phase $m$ at occurrence $j$, define the average regime differentials
\begin{equation}
G_{m,j}=\bar J_{m,j,\mathrm{sp}}^{\rm qs}
-\bar J_{m,j,\mathrm{sh}}^{\rm qs},
\qquad
K_{m,j}=\bar J_{m,j,\mathrm{sh}}^{\rm tr}
-\bar J_{m,j,\mathrm{sp}}^{\rm tr}.
\label{eq:occurrence-regime-terms}
\end{equation}
 
For occurrence $(m,j)$, let $Z_{m,j}=\sum_{t=0}^{T_{m,j}-1}\gamma^t$ and
\begin{equation}
\lambda_{m,j}=
\frac{\sum_{t=0}^{L_{m,j}-1}\gamma^t}{Z_{m,j}}.
\label{eq:occurrence-transient-exposure}
\end{equation}

Decomposing the normalized return difference $\bar J_{m,j,\mathrm{sp}} - \bar J_{m,j,\mathrm{sh}}$ over occurrence $(m,j)$ yields:
\begin{align}
\bar J_{m,j,\mathrm{sp}} - \bar J_{m,j,\mathrm{sh}} &= (1 - \lambda_{m,j})\left( \bar J_{m,j,\mathrm{sp}}^{\mathrm{qs}} - \bar J_{m,j,\mathrm{sh}}^{\mathrm{qs}} \right) + \lambda_{m,j}\left( \bar J_{m,j,\mathrm{sp}}^{\mathrm{tr}} - \bar J_{m,j,\mathrm{sh}}^{\mathrm{tr}} \right) \nonumber \\
&= (1 - \lambda_{m,j}) G_{m,j} - \lambda_{m,j} K_{m,j},
\end{align}

 Let $\omega_{m,j} \ge 0$ denote the normalized evaluation weight assigned to occurrence $(m,j)$, with $\sum_{m,j} \omega_{m,j} = 1$. Multiplying by $\omega_{m,j}$ and summing over all phase occurrences gives:
\begin{equation}
\Delta\bar J = \bar J(\widehat\pi^{\mathrm{sp}}) - \bar J(\widehat\pi^{\mathrm{sh}}) = \sum_{m=0}^{M-1} \sum_{j} \omega_{m,j} \left[ (1 - \lambda_{m,j}) G_{m,j} - \lambda_{m,j} K_{m,j} \right].
\end{equation}

For compact phase-level reporting, let
$\omega_m=\sum_{j\in\mathcal O_m}\omega_{m,j}$ and define
\begin{align}
\lambda_m
&=\frac{1}{\omega_m}
\sum_{j\in\mathcal O_m}\omega_{m,j}\lambda_{m,j},\nonumber\\
G_m
&=\frac{\sum_j\omega_{m,j}(1-\lambda_{m,j})G_{m,j}}
{\omega_m(1-\lambda_m)},\qquad
K_m
=\frac{\sum_j\omega_{m,j}\lambda_{m,j}K_{m,j}}
{\omega_m\lambda_m}.
\label{eq:phase-level-regime-aggregates}
\end{align}
The values of $K_m$ and $G_m$ are set to zero when the duration of the transient and quasi-stationary periods is zero, respectively. Now we can write the advantage of using learned specialized policies over a learned single shared policy as \begin{equation}
\Delta\bar J
:=
\bar J(\widehat\pi^{\rm sp})
-
\bar J(\widehat\pi^{\rm sh})
=\sum_m\omega_m[(1-\lambda_m)G_m-\lambda_mK_m].
\end{equation}

This compares learned shared and learned specialized policies, $\widehat\pi^{\rm sh}$ and $\widehat\pi^{\rm sp}$, respectively, not the best policies in unrestricted classes, i.e., $\pi^{\rm sh}$ and $\pi^{\rm sp}$. This framework yields a simple three-step empirical diagnostic: we measure the performance gap between shared and specialized models across phase age, measure the decay rate of the entry-state distribution post-boundary, and test whether recovery times remain invariant to changes in $T_m$. Ultimately, longer phase horizons favor specialists only when they expose a settled quasi-stationary region where local gains dominate ($G_m > 0$); phase length alone does not guarantee a specialist advantage.

\subsubsection{Hand off Costs with Two Goals}
\label{sec:two-goal-handoff}
\begin{proposition}[Hand off threshold for two-goal problem]
\label{prop:two-goal-handoff}
Consider two consecutive phases. Phase \(0\) has duration \(T_0\) and offers two possible choices of goals: $i \in \{00, 01\}$ one with a higher reward than other, i.e., \(g_{00}\) gives the larger phase-\(0\) reward, while \(g_{01}\) produces a more favourable entry state for phase \(1\).

Let \(F_i(s)\) be the expected discounted return collected during phase \(0\), starting from state \(s\), under goal \(i\). Let \(\mu_i\) be the phase-\(1\) entry-state distribution induced by this choice of goal. Define the continuation value:
\begin{equation}
V_1^{\mathrm{cont}}(\mu_i) = \mathbb{E}_{s_1 \sim \mu_i} \left[V_1^{\mathrm{cont}}(s_1)\right].
\end{equation}

Assume that a phase-local specialized policy optimizes only the phase-\(0\) return \(F_i(s)\), whereas a coordinated full-trajectory controller also accounts for the continuation value in phase \(1\).

If $F_{00}(s) > F_{01}(s)$ then the phase-local specialized policy prefers the immediate-reward goal \(g_{00}\). By contrast, the full-trajectory controller prefers the safe-handoff goal \(g_{01}\) if and only if
\begin{equation}
\gamma^{T_0} \left[ V_1^{\mathrm{cont}}(\mu_{01}) - V_1^{\mathrm{cont}}(\mu_{00}) \right] > F_{00}(s) - F_{01}(s).
\label{eq:two-goal-handoff-threshold}
\end{equation}

When both inequalities hold, the return loss caused by the uncoordinated phase-local choice is
\begin{equation}
\Delta_{\mathrm{coord}}(s) = \gamma^{T_0} \left[ V_1^{\mathrm{cont}}(\mu_{01}) - V_1^{\mathrm{cont}}(\mu_{00}) \right] - \left[ F_{00}(s) - F_{01}(s) \right] > 0.
\label{eq:two-goal-coordination-gap}
\end{equation}
\end{proposition}

\subsubsection{Sample Allocation and Sample Thresholds}
\label{sec:capacity-budget}
% {sec:capacity-budget}

Let $\bar J_m^{\mathrm{qs}}(\widehat\pi_{m, b})$ denote the expected settled return of specialist $m$ trained on $b$ local interactions, and let $\bar J_m^{\mathrm{qs}}(\widehat\pi_{m, \infty})$ represent its asymptotic performance under unlimited samples ($b \to \infty$). %In the asymptotic limit, phase $m$ achieves its maximum specialization gain $G_{m, \infty}^{\mathrm{qs}} > 0$.

Because an ensemble of $M$ specialists partitions a total budget $B$ into local allocations $B_m = \zeta_m B$ (where $\sum \zeta_m = 1$, and $\zeta_m = 1/M$ for balanced sampling), each specialist receives significantly fewer environment steps than a shared model trained on the full budget $B$.

To guarantee that local learning starvation does not destroy the settled specialization gain, we define the local sample threshold $B_{0, m}(\delta)$ as the minimum allocation required for specialist $m$ to recover its asymptotic return within a tolerance $\delta > 0$:
\begin{equation}
B_{0, m}(\delta) = \inf \left\{ b \in \mathbb{N} : \bar J_m^{\mathrm{qs}}(\widehat\pi_{m, \infty}) - \bar J_m^{\mathrm{qs}}(\widehat\pi_{m, b}) \le \delta \right\}.
\label{eq:simple-sample-threshold}
\end{equation}

\begin{proposition}[Sufficient Budget for Positive Specialization Gain]
\label{prop:adequate-specialist-budget}
Assume that under unlimited samples, specialist $m$ achieves a net specialization advantage with the quasi-stationary gain $G_{m, \infty} > \delta$. If the total interaction budget $B$ satisfies:
\begin{equation}
B \ge \max_{m} \frac{B_{0, m}(\delta)}{\zeta_m},
\label{eq:simple-budget-condition}
\end{equation}
then every specialist receives sufficient samples ($B_m \ge B_{0, m}(\delta)$) to maintain a strictly positive local gain:
\begin{equation}
G_{m, \mathrm{ref}}^{\mathrm{qs}}(B_m) \ge G_{m, \infty}^{\mathrm{qs}} - \delta > 0.
\label{eq:positive-gain-guarantee}
\end{equation}
For balanced sampling ($\zeta_m = 1/M$), condition \eqref{eq:simple-budget-condition} simplifies to $B \ge M \max_m B_{0, m}(\delta)$.
\end{proposition}

\subsubsection{Gradient Geometry as an Optimization Mechanism Test}
\label{app:local-optimization}

We use phase-restricted gradients to test one potential optimization mechanism driving shared architecture performance. Let $\mathcal L_m(\theta)$ denote a differentiable phase-local objective evaluated at a common parameter vector $\theta$ under a fixed reference data distribution. Given normalized phase weights $w_m \ge 0$ with $\sum_{m=0}^{M-1} w_m = 1$, define:
\begin{equation}
\mathcal L(\theta) = \sum_{m=0}^{M-1} w_m \mathcal L_m(\theta),
\qquad
g_m(\theta) = \nabla_\theta \mathcal L_m(\theta),
\qquad
g(\theta) = \sum_{m=0}^{M-1} w_m g_m(\theta).
\label{eq:phase-local-objectives}
\end{equation}
In balanced diagnostic sweeps, we set $w_m = 1/M$. Under a shared gradient ascent step $\theta^+ = \theta + \eta g(\theta)$, local objective updates admit a first-order Taylor expansion:

\begin{proposition}[First-Order Decomposition of Shared Updates]
\label{prop:shared-phase-update}
For a differentiable local objective $\mathcal L_m$, a shared gradient update $\theta^+ = \theta + \eta g(\theta)$ yields:
\begin{equation}
\mathcal L_m(\theta^+) - \mathcal L_m(\theta) = \eta w_m \|g_m\|^2 + \eta \sum_{k \neq m} w_k \langle g_m, g_k \rangle + o(\eta).
\end{equation}
\end{proposition}

The first term represents phase $m$'s self-improvement, while the second term captures cross-phase transfer. A shared update is locally destructive to phase $m$ if and only if the aggregate inner product is strictly negative: $\langle g_m, g(\theta) \rangle < 0$. Negative pairwise inner products $\langle g_m, g_k \rangle < 0$ alone do not imply local objective degradation.

To measure aggregate directional alignment without introducing mechanical scaling artifacts with respect to phase count $M$, we introduce the orthogonal-null normalized statistics:
\begin{equation}
A_g = \frac{\|\sum_{m=0}^{M-1} w_m g_m\|^2}{\sum_{m=0}^{M-1} w_m^2 \|g_m\|^2},
\qquad 
C_g = 1 - A_g.
\label{eq:aggregate-alignment}
\end{equation}
Under mutually orthogonal phase gradients ($\langle g_m, g_k \rangle = 0$ for $k \neq m$), $A_g = 1$ identically for any phase count $M$. A positive conflict statistic $C_g > 0$ indicates that negative cross-phase terms dominate the aggregate update.

Phase-to-aggregate cosine similarity also exhibits a fundamental orthogonal baseline, as formalized below.

\begin{proposition}[Orthogonal Gradient Alignment Scaling]
\label{prop:orthogonal-cosine-proof}
Let $g_0, \dots, g_{M-1} \in \mathbb{R}^d$ be non-zero mutually orthogonal phase gradients ($\langle g_m, g_k \rangle = 0$ for all $m \neq k$) and let $g = \frac{1}{M} \sum_{k=0}^{M-1} g_k$ denote the balanced aggregate update ($w_m = 1/M$). If all gradients share equal norms $\|g_m\| = C > 0$, then:
\begin{equation}
\cos(g_m, g) = \frac{1}{\sqrt M}.
\end{equation}
More generally, for arbitrary gradient norms $\|g_m\|$, the orthogonal alignment baseline scales as:
\begin{equation}
\cos(g_m, g) = \frac{\|g_m\|}{\sqrt{\sum_{k=0}^{M-1} \|g_k\|^2}}.
\end{equation}
\end{proposition}

Therefore, a decay rate of $1/\sqrt M$ in cosine similarity indicates that the shared update vector is being spread across orthogonal task directions—a phenomenon we term \textbf{gradient dilution}—rather than experiencing destructive multi-task interference \citep{bengio2020interference}.

Empirically, we differentiate the phase-restricted clipped PPO actor surrogate at commonly trained checkpoints using global advantage normalization. As reported in Appendix~\ref{app:HAP-Res}, empirical evaluation rejects destructive interference in HAPS: $A_g \approx 1.0$, and fewer than $2\%$ of phases exhibit $\langle g_m, g \rangle < 0$. The observed drop in cosine alignment closely tracks $1/\sqrt M$, confirming gradient dilution as the primary optimization bottleneck.

\subsection{Proofs and Derivations}
\label{app:proofs}

This appendix provides complete formal proofs for all theoretical propositions and identities presented in the main text and Appendix~\ref{app:supporting-theory}. The notation is inherited directly from the corresponding statements.

% ============================================================
% STRUCTURAL RESULTS
% ============================================================
\subsubsection{Structural Results}

% \subsubsection{Scheduler-State Markovization (Proposition~\ref{prop:scheduler-state-markov})}

\textbf{Proposition~\ref{prop:scheduler-state-markov}} (Scheduler-state augmentation) \textit{If $z_t$ is sufficient, then the augmented state $\widetilde{s}_t=(s_t,z_t)$ defines a time-homogeneous MDP and recovers Markovianity.}

\begin{proof}
Let $\mathcal{H}_t = (s_0, a_0, z_0, s_1, a_1, z_1, \dots, s_t, a_t, z_t)$ denote the complete execution trajectory history up to time step $t$. To prove that the augmented state $\tilde s_t = (s_t, z_t) \in \mathcal{S} \times \mathcal{Z}$ defines a time-homogeneous Markov Decision Process (MDP), we must establish three distinct properties:
\begin{enumerate}
    \item \textbf{Markovian Memorylessness:} $\Pr(\tilde s_{t+1} \in C \mid \mathcal{H}_t) = \Pr(\tilde s_{t+1} \in C \mid \tilde s_t, a_t)$.
    \item \textbf{Time-Homogeneity:} The transition probability distribution depends on $\tilde s_t$ and $a_t$, but has no explicit dependence on the step index $t$.
    \item \textbf{Reward Consistency:} Expected single-step rewards depend solely on $(\tilde s_t, a_t)$.
\end{enumerate}

\paragraph{Step 1: Proving the Markov Memoryless Property.}
Let $C = A \times B$ be any product measurable set in the augmented state space $\mathcal{S} \times \mathcal{Z}$, where $A \subseteq \mathcal{S}$ and $B \subseteq \mathcal{Z}$. By the definition of conditional probability, we factorize the joint transition of physical state $s_{t+1}$ and phase scheduler state $z_{t+1}$ given complete history $\mathcal{H}_t$:
\begin{equation}
\begin{aligned}
\Pr(\tilde s_{t+1} \in A \times B \mid \mathcal{H}_t) &= \Pr(s_{t+1} \in A, \, z_{t+1} \in B \mid \mathcal{H}_t) \\ &= \Pr(s_{t+1} \in A \mid \mathcal{H}_t, z_{t+1} \in B) \cdot \Pr(z_{t+1} \in B \mid \mathcal{H}_t).
\label{eq:joint-factorization-step1}
\end{aligned}
\end{equation}

By the definition of \emph{scheduler state sufficiency}, the exogenous scheduler update law is conditionally independent of past physical states and actions given $z_t$:
\begin{equation}
\Pr(z_{t+1} \in B \mid \mathcal{H}_t) = q(B \mid z_t).
\label{eq:scheduler-independence}
\end{equation}
Similarly, conditioned on the active phase index $\sigma(t) = h(z_t)$, the physical transition law depends only on the current physical state $s_t$ and action $a_t$:
\begin{equation}
\Pr(s_{t+1} \in A \mid \mathcal{H}_t, z_{t+1} \in B) = p_{h(z_t)}(A \mid s_t, a_t).
\label{eq:physical-independence}
\end{equation}
Substituting \ref{eq:scheduler-independence} and \ref{eq:physical-independence} into \ref{eq:joint-factorization-step1} yields:
\begin{equation}
\Pr(\tilde s_{t+1} \in A \times B \mid \mathcal{H}_t) = p_{h(z_t)}(A \mid s_t, a_t) \, q(B \mid z_t).
\label{eq:final-markov-kernel}
\end{equation}
Because the right-hand side of \ref{eq:final-markov-kernel} depends strictly on $\tilde s_t = (s_t, z_t)$ and $a_t$, the historical terms \\$(s_{0:t-1}, a_{0:t-1}, z_{0:t-1})$ drop out entirely, proving the Markov property.

\paragraph{Step 2: Proving Time-Homogeneity.}
Notice that neither the phase-local transition kernel $p_{h(z_t)}(\cdot \mid s_t, a_t)$ nor the scheduler update kernel $q(\cdot \mid z_t)$ contains an explicit parameter $t$. The time dependency in the environment is entirely localized inside $z_t$. Consequently, for any two step indices $t_1 \neq t_2$, if $\tilde s_{t_1} = \tilde s_{t_2} = (s, z)$ and $a_{t_1} = a_{t_2} = a$, then:
\begin{equation}
\begin{aligned}
    \Pr(\tilde s_{t_1 + 1} \in C \mid \tilde s_{t_1} = (s,z), a_{t_1} = a) &= \Pr(\tilde s_{t_2 + 1} \in C \mid \tilde s_{t_2} = (s,z), a_{t_2} = a) \\&= \widetilde P(C \mid (s,z), a),
\end{aligned}
\end{equation}
where $\widetilde P(A \times B \mid (s,z), a) := p_{h(z)}(A \mid s, a) \, q(B \mid z)$ defines a time-invariant transition kernel on $\mathcal{S} \times \mathcal{Z}$.

\paragraph{Step 3: Reward Consistency.}
The expected single-step reward in phase $m = h(z_t)$ is given by $r_m(s_t, a_t)$. Defining the augmented reward function $\tilde r(\tilde s_t, a_t):= r_{h(z_t)}(s_t, a_t)$, we see that $\tilde r: (\mathcal{S} \times \mathcal{Z}) \times \mathcal{A} \to \mathbb{R}$ is a well-defined, stationary function on the augmented state space.

Combining Steps 1, 2, and 3 establishes that $\widetilde{\mathcal{M}} = (\mathcal{S} \times \mathcal{Z}, \mathcal{A}, \widetilde P, \widetilde r, \gamma)$ is a formally valid, time-homogeneous Markov Decision Process.
\end{proof}

% \subsubsection{Proof of Corollary~\ref{cor:phase-clock-markov} (Periodic Phase-Clock Augmentation)}
% \label{app:proof-cor-clock}
\textbf{Corollary \ref{cor:phase-clock-markov}} (Periodic Phase-Clock Augmentation) \textit{Suppose phases follow a deterministic cycle with phase durations $(T_0, T_1, \dots, T_{M-1})$. Let $T_{\mathrm{period}} = \sum_{m=0}^{M-1} T_m$ denote the total cycle length, and let $S_m = \sum_{k=0}^{m-1} T_k$ denote the entry step offset of phase $m$ within a cycle (with $S_0 = 0$). For global step $t$, let $\bar t = t \bmod T_{\mathrm{period}}$. The phase index $\sigma(t)$ and within-phase clock $\tau(t)$ are uniquely defined by:
\begin{equation}
\sigma(t) = m \quad \text{such that } S_m \le \bar t < S_m + T_m, \qquad \tau(t) = \bar t - S_{\sigma(t)}.
\end{equation}
Then $z_t = (\sigma(t), \tau(t))$ is sufficient. Consequently, the augmented state $(s_t, \sigma(t), \tau(t))$ is Markovian, time-homogeneous, and compact.}

\begin{proof}
Let $(T_0, T_1, \dots, T_{M-1})$ be a sequence of positive integer phase durations. The total cycle period is $T_{\mathrm{period}} = \sum_{m=0}^{M-1} T_m$, and the cumulative entry offset for phase $m$ within a cycle is $S_m = \sum_{k=0}^{m-1} T_k$, with $S_0 = 0$. 

By definition of Corollary~\ref{cor:phase-clock-markov}, for any global step index $t \in \mathbb{N}_0$, let $\bar t = t \bmod T_{\mathrm{period}} \in \{0, \dots, T_{\mathrm{period}} - 1\}$. The phase index $\sigma(t) \in \{0, \dots, M-1\}$ and within-phase clock $\tau(t) \in \{0, \dots, T_{\sigma(t)} - 1\}$ are uniquely determined by:
\begin{equation}
\sigma(t) = m \iff S_m \le \bar t < S_m + T_m, \qquad \tau(t) = \bar t - S_{\sigma(t)}.
\label{eq:clock-offset-definition}
\end{equation}

\paragraph{Step 1: Deterministic One-Step Update Operator $F$.}
We now evaluate the one-step temporal evolution of $z_t = (\sigma(t), \tau(t))$ from step $t$ to step $t+1$. Consider two cases based on the value of the within-phase clock $\tau(t)$:

\begin{enumerate}
    \item \textbf{Intra-Phase Step ($\tau(t) < T_{\sigma(t)} - 1$):} The step $t+1$ remains inside the same phase $m = \sigma(t)$. The within-phase clock advances by 1, yielding:
    \begin{equation}
    \sigma(t+1) = \sigma(t), \qquad \tau(t+1) = \tau(t) + 1.
    \end{equation}
    
    \item \textbf{Phase Boundary Transition ($\tau(t) = T_{\sigma(t)} - 1$):} Step $t$ is the final step of phase $m = \sigma(t)$. At step $t+1$, the system transitions to the next phase in the macro-cycle, $(m + 1) \bmod M$, resetting the within-phase clock:
    \begin{equation}
    \sigma(t+1) = (\sigma(t) + 1) \bmod M, \qquad \tau(t+1) = 0.
    \end{equation}
\end{enumerate}

Thus, the schedule state update $z_{t+1} = (\sigma(t+1), \tau(t+1))$ is governed by the deterministic function $F: \mathcal{Z} \to \mathcal{Z}$:
\begin{equation}
F(m, \tau) = \begin{cases}
(m, \, \tau + 1), & \text{if } \tau < T_m - 1, \\
((m + 1) \bmod M, \, 0), & \text{if } \tau = T_m - 1,
\end{cases}
\label{eq:deterministic-clock-operator}
\end{equation}
where $z_{t+1} = F(z_t)$ holds for any arbitrary phase lengths $(T_0, \dots, T_{M-1})$.

\paragraph{Step 2: Kernel Factorization and Sufficiency.}
Because $z_{t+1} = F(z_t)$ is a deterministic mapping, the exogenous scheduler kernel $q(B \mid z_t)$ in Proposition~\ref{prop:scheduler-state-markov} reduces to a point-mass Dirac delta measure:
\begin{equation}
q(B \mid z_t) = \delta_{F(z_t)}(B) = \begin{cases} 1, & \text{if } F(z_t) \in B, \\ 0, & \text{if } F(z_t) \notin B, \end{cases}
\end{equation}
for any measurable set $B \subseteq \mathcal{Z}$. Substituting $q(B \mid z_t) = \delta_{F(z_t)}(B)$ into \ref{eq:final-markov-kernel} gives:
\begin{equation}
\Pr(s_{t+1} \in A, \, z_{t+1} \in B \mid \mathcal{H}_t) = p_{\sigma(t)}(A \mid s_t, a_t) \, \delta_{F(z_t)}(B),
\end{equation}
which is conditionally independent of past trajectory history $\mathcal{H}_{:t-1}$ given $(s_t, a_t, z_t)$ and also time-invariant.

\paragraph{Step 3: Compactness and Recurrence.}
The domain of $z_t$ is the finite discrete union:
\begin{equation}
\mathcal{Z} = \bigcup_{m=0}^{M-1} \left( \{m\} \times \{0, \dots, T_m - 1\} \right),
\end{equation}
which has finite cardinality $|\mathcal{Z}| = T_{\mathrm{period}} = \sum_{m=0}^{M-1} T_m$. Unlike global time $t \in \mathbb{N}_0$, which grows infinitely and prevents state reuse, $z_t$ recurs periodically every $T_{\mathrm{period}}$ steps ($z_{t + T_{\mathrm{period}}} = z_t$). This establishes that the augmented state $(s_t, \sigma(t), \tau(t))$ forms a compact, time-homogeneous Markov Decision Process.
\end{proof}

% \subsubsection{Shared--Multipolicy Equivalence (Proposition~\ref{prop:policy-class-equivalence})}

\textbf{Proposition~\ref{prop:policy-class-equivalence}} (Shared--multipolicy equivalence) \textit{There is a bijection between shared policies $\pi^\text{sh}(\cdot\mid s, i, m)$ and an ensemble of specialized policies $\boldsymbol{\pi}^\textbf{sp} =(\pi_0(\cdot\mid s, i),\ldots,\pi_{M-1}(\cdot\mid s, i))$ expressed through the gated map $\Gamma$ such that
\begin{equation}
\begin{aligned}
\pi^{\rm sp}(a\mid s,i ,m)
= \Gamma(\boldsymbol{\pi}^\textbf{sp})(a\mid s, i, m) =
\pi_m(a\mid s, i),\\
\pi^\text{sh}(a\mid s, i, m)=\pi_m(a\mid s, i) = \pi^\text{sp}(a\mid s, i, m).
\end{aligned}
\label{eq:shared-multipolicy-bijection1}
\end{equation}
It is always possible to represent a shared policy with an ensemble of specialized policies and vice versa.}

\begin{proof}
Let $x = (s, i)$ denote the shared input observation (physical state and context). Given an ensemble of specialized policies $(\pi_0, \dots, \pi_{M-1})$, define the canonical concatenation mapping $\Phi$:
\begin{equation}
\Phi(\pi_0, \dots, \pi_{M-1})(a \mid x, m) = \pi_m(a \mid x).
\end{equation}
Conversely, given a shared phase-conditioned policy $\pi$, define the phase slicing operator $\Psi$:
\begin{equation}
\Psi(\pi)_m(a \mid x) = \pi(a \mid x, m).
\end{equation}
For any specialist ensemble, $\Psi(\Phi(\pi_0, \dots, \pi_{M-1}))_m = \pi_m$ for all $m \in \{0, \dots, M-1\}$. For any shared policy, $\Phi(\Psi(\pi))(a \mid x, m) = \pi(a \mid x, m)$. Thus, $\Phi$ and $\Psi$ define bijective inverse mappings between the unrestricted shared policy space $\Pi_{\mathrm{sh}}$ and the product specialist space $\Pi_0 \times \dots \times \Pi_{M-1}$.

Consequently, corresponding policies assign identical action probability distributions $\pi(a \mid x, m) = \pi_m(a \mid x)$ across all physical states and phase indices. Under identical initial state distributions and transition dynamics, induction over time steps $t$ proves identical trajectory laws and identical expected returns. Taking supremums over both policy spaces establishes $\bar J_{\mathrm{sh}}^\star = \bar J_{\mathrm{sp}}^\star$.
\end{proof}

\begin{figure}
    \centering
    \includegraphics[width=1\linewidth]{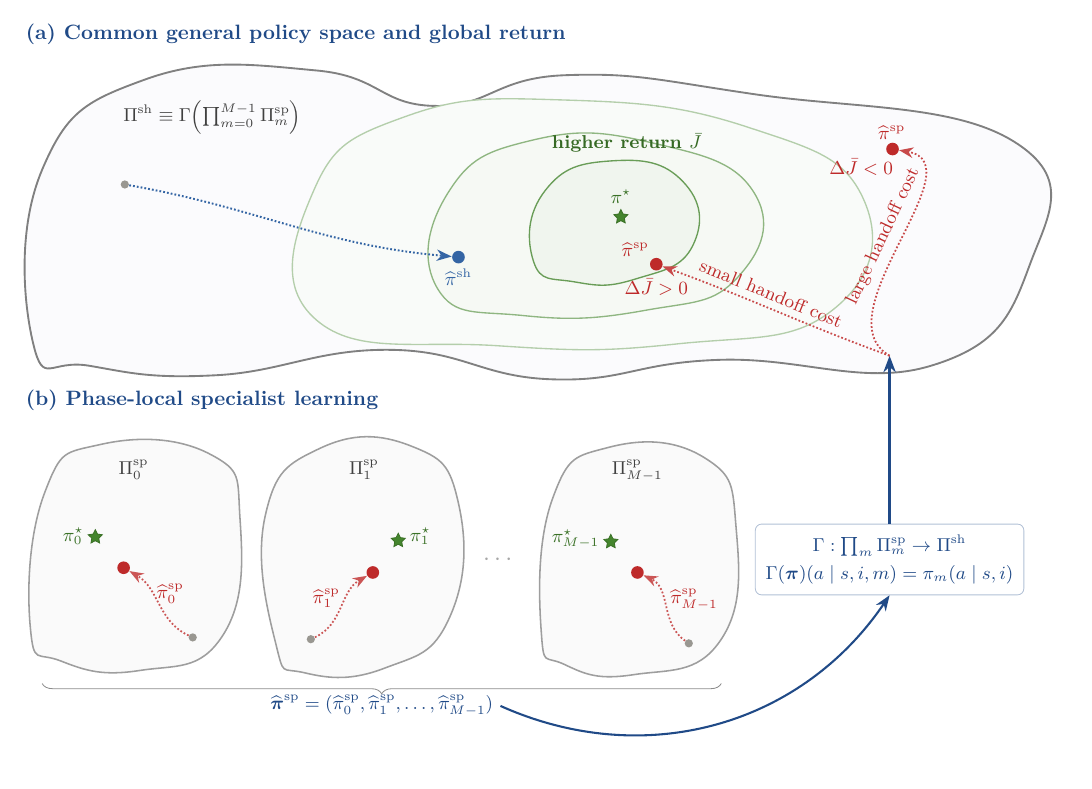}
    \caption{\textbf{Geometric intuition of policy class equivalence and local specialist learning.} 
  \textbf{(a)}The common achievable policy space $\Pi^{\mathrm{sh}}$ encompasses both shared and multi-policy solutions, as established by Proposition~\ref{prop:policy-class-equivalence}. Depending on the alignment of phase-local entry distributions with deployment handoffs, exact multi-policy assembly yields either a specialization gain ($\Delta\bar{J} > 0$) or a handoff mismatch penalty ($\Delta\bar{J} < 0$).
  \textbf{(b)} Individual phase specialists $\widehat{\pi}_m^{\mathrm{sp}}$ are trained independently within local policy spaces $\Pi_m^{\mathrm{sp}}$ under phase-local entry distributions. The tuple $\widehat{\boldsymbol{\pi}}^{\mathrm{sp}}$ is assembled via the exact gating map $\Gamma$ with zero projection loss. }
    \label{fig:policy_geometry}
\end{figure}
\subsubsection{Geometric Intuition of Policy Class Equivalence}
\label{sec:geometric_intuition}

Figure~\ref{fig:policy_geometry} provides a geometric visualization of the equivalence result established in Proposition~\ref{prop:policy-class-equivalence}. While shared and multi-policy architectures possess identical functional expressivity in the common achievable space $\Pi^{\mathrm{sh}}$ (Figure~\ref{fig:policy_geometry}a), their distinct optimization dynamics lead to different realized solutions under finite data, leading to different average global return $\bar J$. 

As depicted in Figure~\ref{fig:policy_geometry}b, each specialist policy $\widehat{\pi}_m^{\mathrm{sp}}$ in the multi-policy ensemble is optimized independently within its local space $\Pi_m^{\mathrm{sp}}$. The exact gating map $\Gamma$ embeds this product tuple $\widehat{\boldsymbol{\pi}}^{\mathrm{sp}} = (\widehat{\pi}_0^{\mathrm{sp}}, \dots, \widehat{\pi}_{M-1}^{\mathrm{sp}})$ into $\Pi^{\mathrm{sh}}$ without projection loss. However, because local training is decoupled across phases, specialists optimize for phase-local performance without accounting for cross-phase coordination. 

    Consequently, the global performance of the assembled multi-policy depends heavily on boundary handoffs: when phase-local optimization yields terminal states that do not compromise initialization in subsequent phases,specialization can yield strict return gains over a single shared policy ($\Delta\bar{J} > 0$). Conversely, when handoff mismatch or data starvation occurs, uncoordinated local optimization degrades overall trajectory performance ($\Delta\bar{J} < 0$). Note that for visual clarity, this representation assumes a unique global optimal policy $\pi^\star$.
%
% ============================================================
% HANDOFF AND RECOVERY RESULTS
% ============================================================
\subsubsection{Transient Handoff, Recovery, and Regime Bounds}

% \subsubsection{Proof of Proposition~\ref{prop:two-goal-handoff} (Two-Goal Handoff Threshold)}
% \label{app:proof-two-goal-handoff}
\textbf{Proposition \ref{prop:two-goal-handoff}} (Hand off threshold for two-goal problem) \textit{Consider two consecutive phases. Phase \(0\) has duration \(T_0\) and offers two possible choices of goals: $i \in \{00, 01\}$ one with a higher reward than other, i.e., \(g_{00}\) gives the larger phase-\(0\) reward, while \(g_{01}\) produces a more favourable entry state for phase \(1\).}

\textit{Let \(F_i(s)\) be the expected discounted return collected during phase \(0\), starting from state \(s\), under goal \(i\). Let \(\mu_i\) be the phase-\(1\) entry-state distribution induced by this choice of goal. Define the continuation value:
\begin{equation}
V_1^{\mathrm{cont}}(\mu_i) = \mathbb{E}_{s_1 \sim \mu_i} \left[V_1^{\mathrm{cont}}(s_1)\right].
\end{equation}}

\textit{Assume that a phase-local specialized policy optimizes only the phase-\(0\) return \(F_i(s)\), whereas a coordinated full-trajectory controller also accounts for the continuation value in phase $1$.}

\textit{If $F_{00}(s) > F_{01}(s)$ then the phase-local specialized policy prefers the immediate-reward goal \(g_{00}\). By contrast, the full-trajectory controller prefers the safe-handoff goal \(g_{01}\) if and only if
\begin{equation}
\gamma^{T_0} \left[ V_1^{\mathrm{cont}}(\mu_{01}) - V_1^{\mathrm{cont}}(\mu_{00}) \right] > F_{00}(s) - F_{01}(s).
\label{eq:two-goal-handoff-threshold2}
\end{equation}
}

\textit{When both inequalities hold, the return loss caused by the uncoordinated phase-local choice is
\begin{equation}
\Delta_{\mathrm{coord}}(s) = \gamma^{T_0} \left[ V_1^{\mathrm{cont}}(\mu_{01}) - V_1^{\mathrm{cont}}(\mu_{00}) \right] - \left[ F_{00}(s) - F_{01}(s) \right] > 0.
\label{eq:two-goal-coordination-gap1}
\end{equation}
}

\begin{proof}
Fix an initial state $s \in \mathcal{S}$ at step $t = 0$ of phase $0$. For each choice of goals $i \in \{00, 01\}$, the total expected trajectory return across both phases decomposes additively at the phase boundary $t = T_0$:
\begin{equation}
Q_i(s) := F_i(s) + \gamma^{T_0} V_1^{\mathrm{cont}}(\mu_i),
\label{eq:two-goal-total-value}
\end{equation}
where $F_i(s) = \mathbb{E}_{\pi_i} \left[ \sum_{t=0}^{T_0-1} \gamma^t r_0(s_t, a_t) \,\middle|\, s_0 = s \right]$ denotes the expected phase-$0$ return, and $\mu_i(\cdot) = \Pr(s_{T_0} \in \cdot \mid s_0 = s, \pi_i)$ is the induced phase-$1$ entry state distribution under goal choice $i$.

By definition, an uncoordinated phase-local specialist $\widehat{\pi}_{0,\mathrm{sp}}$ optimizes strictly the phase-$0$ return $F_i(s)$. Consequently, under condition $F_{00}(s) > F_{01}(s)$, the specialist uniquely selects commitment $g_{00}$.

Conversely, a coordinated full-trajectory controller selects commitment $g_{01}$ over $g_{00}$ if and only if $Q_{01}(s) > Q_{00}(s)$. Substituting \eqref{eq:two-goal-total-value} yields:
\begin{align}
Q_{01}(s) > Q_{00}(s) &\iff F_{01}(s) + \gamma^{T_0} V_1^{\mathrm{cont}}(\mu_{01}) > F_{00}(s) + \gamma^{T_0} V_1^{\mathrm{cont}}(\mu_{00}) \nonumber \\
&\iff \gamma^{T_0} \left[ V_1^{\mathrm{cont}}(\mu_{01}) - V_1^{\mathrm{cont}}(\mu_{00}) \right] > F_{00}(s) - F_{01}(s),
\end{align}
which establishes the threshold condition in \eqref{eq:two-goal-handoff-threshold2}.

Finally, when $F_{00}(s) > F_{01}(s)$ and \eqref{eq:two-goal-handoff-threshold2} simultaneously hold, the specialist selects $g_{00}$ while the optimal coordinated policy selects $g_{01}$. The resulting suboptimality gap $\Delta_{\mathrm{coord}}(s) := Q_{01}(s) - Q_{00}(s)$ evaluates to:
\begin{equation}
\Delta_{\mathrm{coord}}(s) = \gamma^{T_0} \left[ V_1^{\mathrm{cont}}(\mu_{01}) - V_1^{\mathrm{cont}}(\mu_{00}) \right] - \left[ F_{00}(s) - F_{01}(s) \right] > 0,
\end{equation}
where strict positivity follows directly from \eqref{eq:two-goal-handoff-threshold2}, completing the proof of \eqref{eq:two-goal-coordination-gap1}.
\end{proof}  

\subsubsection{Sample Allocation and Sample Thresholds}
\textbf{Propositon  \ref{prop:adequate-specialist-budget}} (Sufficient Budget for Positive Specialization Gain) \textit{Assume that under unlimited samples, specialist $m$ achieves a net specialization advantage with the quasi-stationary gain $G_{m, \infty} > \delta$. If the total interaction budget $B$ satisfies:
\begin{equation}
B \ge \max_{m} \frac{B_{0, m}(\delta)}{\zeta_m},
\label{eq:simple-budget-condition1}
\end{equation}
then every specialist receives sufficient samples ($B_m \ge B_{0, m}(\delta)$) to maintain a strictly positive local gain:
\begin{equation}
G_{m, \mathrm{ref}}^{\mathrm{qs}}(B_m) \ge G_{m, \infty}^{\mathrm{qs}} - \delta > 0.
\end{equation}
For balanced sampling ($\zeta_m = 1/M$), condition \eqref{eq:simple-budget-condition1} simplifies to $B \ge M \max_m B_{0, m}(\delta)$.
}

\begin{proof}
By definition of the infimum in \eqref{eq:simple-sample-threshold}, for any phase-local allocation $b \ge B_{0, m}(\delta)$, the learning gap relative to the asymptotic limit satisfies:
\begin{equation}
\bar J_m^{\mathrm{qs}}(\widehat\pi_{m, \infty}) - \bar J_m^{\mathrm{qs}}(\widehat\pi_{m, b}) \le \delta \implies \bar J_m^{\mathrm{qs}}(\widehat\pi_{m, b}) \ge \bar J_m^{\mathrm{qs}}(\widehat\pi_{m, \infty}) - \delta.
\end{equation}
Under the budget condition $B \ge \max_{m} \frac{B_{0, m}(\delta)}{\zeta_m}$, multiplying by $\zeta_m > 0$ yields:
\begin{equation}
B_m = \zeta_m B \ge \zeta_m \left( \frac{B_{0, m}(\delta)}{\zeta_m} \right) = B_{0, m}(\delta) \quad \forall m \in \{0, \dots, M-1\}.
\end{equation}
Substituting $b = B_m$ into the suboptimality bound gives:
\begin{equation}
\bar J_m^{\mathrm{qs}}(\widehat\pi_{m, \mathrm{sp}, B_m}) \ge \bar J_m^{\mathrm{qs}}(\widehat\pi_{m, \infty}) - \delta.
\end{equation}
Subtracting the return of the shared policy $\bar J_m^{\mathrm{qs}}(\widehat\pi_{\mathrm{sh}, B})$ trained on total budget $B$ yields the expected settled specialization gain $G_{m, \mathrm{ref}}^{\mathrm{qs}}(B_m)$:
\begin{equation}
G_{m, \mathrm{ref}}^{\mathrm{qs}}(B_m) = \bar J_m^{\mathrm{qs}}(\widehat\pi_{m, \mathrm{sp}, B_m}) - \bar J_m^{\mathrm{qs}}(\widehat\pi_{\mathrm{sh}, B}) \ge \bar J_m^{\mathrm{qs}}(\widehat\pi_{m, \infty}) - \delta - \bar J_m^{\mathrm{qs}}(\widehat\pi_{\mathrm{sh}, B}).
\end{equation}
Because the shared policy's return at finite budget $B$ is upper bounded by its asymptotic return ($\bar J_m^{\mathrm{qs}}(\widehat\pi_{\mathrm{sh}, B}) \le \bar J_m^{\mathrm{qs}}(\widehat\pi_{\mathrm{sh}, \infty})$), we obtain:
\begin{equation}
G_{m, \mathrm{ref}}^{\mathrm{qs}}(B_m) \ge \bar J_m^{\mathrm{qs}}(\widehat\pi_{m, \infty}) - \bar J_m^{\mathrm{qs}}(\widehat\pi_{\mathrm{sh}, \infty}) - \delta = G_{m, \infty}^{\mathrm{qs}} - \delta.
\end{equation}
Since $\delta < G_{m, \infty}^{\mathrm{qs}}$ by assumption, $G_{m, \infty}^{\mathrm{qs}} - \delta > 0$, proving that every local specialist achieves a strictly positive specialization gain when $B \ge \max_m \frac{B_{0, m}(\delta)}{\zeta_m}$.
\end{proof}

%
%============================================================
% GRADIENT GEOMETRY RESULTS
% ============================================================
\subsubsection{Gradient Geometry and Optimization Proofs}

% \subsubsection{First-Order Shared Update Decomposition (Proposition\ref{prop:shared-phase-update})}
\textbf{Proposition \ref{prop:shared-phase-update}} (First-Order Decomposition of Shared Updates) \textit{For a differentiable local objective $\mathcal L_m$, a shared gradient update $\theta^+ = \theta + \eta g(\theta)$ yields:
\begin{equation}
\mathcal L_m(\theta^+) - \mathcal L_m(\theta) = \eta w_m \|g_m\|^2 + \eta \sum_{k \neq m} w_k \langle g_m, g_k \rangle + o(\eta).
\end{equation}
}

\begin{proof}
Because the local objective $\mathcal L_m$ is differentiable at the parameter vector $\theta$, its first-order Taylor expansion in the direction of the aggregate update $g(\theta) = \sum_k w_k g_k(\theta)$ is:
\begin{equation}
\mathcal L_m(\theta + \eta g) - \mathcal L_m(\theta) = \eta \langle \nabla \mathcal L_m(\theta), g(\theta) \rangle + o(\eta).
\end{equation}
Substituting $g_m = \nabla \mathcal L_m(\theta)$ and expanding $g(\theta)$ gives:
\begin{equation}
\langle g_m, g \rangle = \left\langle g_m, \, \sum_{k=0}^{M-1} w_k g_k \right\rangle = w_m \|g_m\|^2 + \sum_{k \neq m} w_k \langle g_m, g_k \rangle.
\end{equation}
Multiplying by step size $\eta$ proves Prop.~\ref{prop:shared-phase-update}.

Define $c_m = \langle g_m, g \rangle$. If $c_m < 0$, then because $\lim_{\eta \to 0^+} \frac{o(\eta)}{\eta} = 0$, there exists a step size threshold $\eta_0 > 0$ such that $|o(\eta)| < \eta \frac{|c_m|}{2}$ for all $\eta \in (0, \eta_0)$. Consequently:
\begin{equation}
\mathcal L_m(\theta + \eta g) - \mathcal L_m(\theta) \le \eta c_m + |o(\eta)| < \eta c_m - \frac{\eta c_m}{2} = \frac{\eta c_m}{2} < 0,
\end{equation}
proving that a sufficiently small shared gradient step strictly decreases phase $m$'s local objective.
\end{proof}

% \subsubsection{Orthogonal Gradient Alignment Scaling (Proposition\ref{prop:orthogonal-cosine-proof})}
\textbf{Proposition \ref{prop:orthogonal-cosine-proof}} (Orthogonal Gradient Alignment Scaling) \textit{Let $g_0, \dots, g_{M-1} \in \mathbb{R}^d$ be non-zero mutually orthogonal phase gradients ($\langle g_m, g_k \rangle = 0$ for all $m \neq k$) and let $g = \frac{1}{M} \sum_{k=0}^{M-1} g_k$ denote the balanced aggregate update ($w_m = 1/M$). If all gradients share equal norms $\|g_m\| = C > 0$, then:
\begin{equation}
\cos(g_m, g) = \frac{1}{\sqrt M}.
\end{equation}
More generally, for arbitrary gradient norms $\|g_m\|$, the orthogonal alignment baseline scales as:
\begin{equation}
\cos(g_m, g) = \frac{\|g_m\|}{\sqrt{\sum_{k=0}^{M-1} \|g_k\|^2}}.
\end{equation}
}

\begin{proof}
Let $g_0, \dots, g_{M-1} \in \mathbb{R}^d$ satisfy $\|g_m\| = C > 0$ and $\langle g_m, g_k \rangle = 0$ for all $m \neq k$. Let $g = \frac{1}{M} \sum_{k=0}^{M-1} g_k$ denote the aggregate gradient under balanced weights $w_m = 1/M$.

\paragraph{Step 1: Aggregate Gradient Norm.}
Expanding the squared norm $\|g\|^2$:
\begin{equation}
\|g\|^2 = \left\langle \frac{1}{M} \sum_{k=0}^{M-1} g_k, \, \frac{1}{M} \sum_{j=0}^{M-1} g_j \right\rangle = \frac{1}{M^2} \sum_{k=0}^{M-1} \sum_{j=0}^{M-1} \langle g_k, g_j \rangle.
\end{equation}
By mutual orthogonality, cross-terms $\langle g_k, g_j \rangle = 0$ for $k \neq j$. Substituting $\|g_k\|^2 = C^2$:
\begin{equation}
\|g\|^2 = \frac{1}{M^2} \sum_{k=0}^{M-1} \|g_k\|^2 = \frac{1}{M^2} \cdot (M C^2) = \frac{C^2}{M} \implies \|g\| = \frac{C}{\sqrt M}.
\end{equation}

\paragraph{Step 2: Cosine Similarity.}
The inner product $\langle g_m, g \rangle$ evaluates to:
\begin{equation}
\langle g_m, g \rangle = \left\langle g_m, \, \frac{1}{M} \sum_{k=0}^{M-1} g_k \right\rangle = \frac{1}{M} \langle g_m, g_m \rangle = \frac{C^2}{M}.
\end{equation}
Substituting into the cosine similarity definition:
\begin{equation}
\cos(g_m, g) = \frac{\langle g_m, g \rangle}{\|g_m\| \|g\|} = \frac{\frac{C^2}{M}}{C \cdot \frac{C}{\sqrt M}} = \frac{\frac{C^2}{M}}{\frac{C^2}{\sqrt M}} = \frac{1}{\sqrt M}.
\end{equation}

If gradient norms are unequal ($\|g_m\| \neq \|g_k\|$), Step 1 generalizes to $\|g\| = \frac{1}{M} \sqrt{\sum_{k=0}^{M-1} \|g_k\|^2}$, yielding the general orthogonal baseline:
\begin{equation}
\cos(g_m, g) = \frac{\|g_m\|}{\sqrt{\sum_{k=0}^{M-1} \|g_k\|^2}}.
\end{equation}
\end{proof}

\subsection{Experimental Details}
\label{app:experiments}

% This section details the domain formulations, network architectures, sample budgets, and evaluation protocols needed to replicate all experiments. The three test environments (HAP, PhasePendulum, and PhaseGridWorld) balance diagnostic rigor with low computational overhead, running entirely on local CPU resources. Commands, PBS files, raw seed outputs, and analysis scripts are included in the supplementary repository.

This appendix gives the environment definitions, network architectures, sample budgets, and evaluation protocols needed to reproduce every experiment. Appendix~\ref{app:architecture-training} describes the architectures and the training protocol common to all runs; Appendix~\ref{app:hap-details}, \ref{app:pendulum-details} and \ref{app:gridworld-details} describe the three environments, and Appendix~\ref{app:eval} describes how results are aggregated and reported. 

% The environments are deliberately small and each isolates a different factor: HAP is a telecommunications coverage-control task with an traffic driven phase schedule \cite{loussouarn2026haps}, 

% PhasePendulum a continuous-control task in which we have customized the dwell time and phase heterogeneity can be varied directly, and PhaseGridWorld a discrete task in which the recovery time can be computed exactly. Every configuration runs on commodity CPU hardware, with a full sweep over $N$~seeds requiring approximately $X$~core-hours and no accelerators. Commands, job scripts, raw per-seed outputs, and analysis scripts are provided in the supplementary material.

\subsubsection{Architectures and Common Training Protocol}
\label{app:architecture-training}

% \paragraph{Policy inputs and parameter sharing.}
% Let $s_t$ denote the original observation state, $e_{m_t}$ the one-hot phase code, and \( c(\tau_t)=[\sin(2\pi\tau_t/T),\cos(2\pi\tau_t/T)] \) the clock encoding. \emph{Single w/ clock} is one PPO actor--critic receiving $[s_t,e_{m_t},c(\tau_t)]$. \emph{Multihead} receives $[s_t,e_{m_t}]$ and uses one shared actor body, one phase-specific linear action head, and one shared critic. \textcolor{red}{what do you mean by this? (The active head is selected from the phase code.)} \emph{Multipolicy} uses $M$ independent PPO actor--critics selected externally by phase. A standard specialist receives $s_t$; \emph{multipolicy w/ clock} receives $[s_t,c(\tau_t)]$. It does not receive the phase code because selecting the network already identifies the phase. Unless an ablation states otherwise, all actor and critic multilayer perceptrons have two hidden layers of width $256$, Tanh activations, and the Stable-Baselines \citep{stable-baselines3} initialization defaults.

\paragraph{Policy inputs and parameter sharing.}
Let $s_t$ be the environment observation, $e_{m_t}\in\{0,1\}^M$ the one-hot code of the active phase, and $c(\tau_t)=[\sin(2\pi\tau_t/T),\cos(2\pi\tau_t/T)]$ the encoding of the local clock within the current block. All experiments use equal block durations $T_m\equiv T$. All architectures are trained with PPO. \emph{Single+clock} is a single actor--critic taking $[s_t,e_{m_t},c(\tau_t)]$. \emph{Multihead} takes $[s_t,e_{m_t}]$ and combines a shared actor body and a shared critic with $M$ phase-specific linear action heads; at each step the head indexed by the active phase emits the action, so the shared body is updated from every transition while each head is updated only from transitions of its own phase. The phase code is also supplied to the body, allowing it to form phase-conditioned features rather than only phase-agnostic ones. \emph{Multipolicy} uses $M$ independent actor--critics, the active phase selecting which one acts and is updated; its input is $s_t$, or $[s_t,c(\tau_t)]$ in the \emph{multipolicy+clock} variant. The phase code is not supplied here, since selecting the network already identifies the phase. Unless an ablation states otherwise, all actor and critic MLPs have two hidden layers of width $256$ with Tanh activations and the Stable-Baselines3 \citep{stable-baselines3} default initialization.

\paragraph{Secondary controls in ablation studies.}
Two input ablations isolate the value of the phase and clock signals: the phase-blind baseline observes only $s_t$, and the phase-conditioned baseline observes $[s_t,e_{m_t}]$, that is, the shared architecture without the clock encoding. Capacity controls vary width, using $[128,128]$ or $[512,512]$ in place of the default $[256,256]$, and depth, using four or eight hidden layers of width $256$ in place of two; the same setting is applied to every architecture. In the warm-start protocol with nominal per-phase budget $B_0$, a single phase-blind policy is first trained on the full schedule for $B_0$ interactions, its parameters are copied into all $M$ specialists, and each copy is fine-tuned for $B_0(M-1)/M$ interactions, so the total is $MB_0$, matching standard multipolicy training. PCGrad \citep{yu2020pcgrad} and CAGrad \citep{liu2021cagrad} split each shared-policy PPO minibatch by phase, compute one complete PPO-loss gradient per represented phase, and combine those gradients by the respective rule before the optimizer step. The two distillation controls transfer policies in opposite directions. The specialist-to-shared control compresses the trained specialist ensemble into a single phase-conditioned student by minimizing the KL divergence to each specialist's action distribution on that phase's rollouts, following policy distillation \citep{rusu2015distillation}. The shared-to-specialist control trains each specialist with reinforcement learning while adding a KL penalty toward a pretrained shared teacher, with the coefficient decayed from $\beta_0$ to zero over training, following the teacher-guided principle of kickstarting \citep{schmitt2018kickstarting}. The latter is related to Distral \citep{teh2017distral}, although we do not jointly learn a shared policy centroid.
% The phase-blind baseline observes only $s_t$, while the phase-conditioned baseline observes $[s_t,e_{m_t}]$ without a clock. Capacity controls replace the default hidden widths by $[128,128]$ or $[512,512]$, or use four to height hidden layers of width $256$ compared to the original two layers. In the warm-start protocol with nominal per-phase budget $B_0$, one non-augmented policy is first trained for $B_0$ interactions, copied into all $M$ specialists, and each copy is then fine-tuned for $B_0(M-1)/M$ interactions. The complete procedure therefore uses the same total budget $MB_0$ as standard multipolicy training. PCGrad and CAGrad from \cite{yu2020pcgrad,liu2021cagrad} split each shared-policy PPO minibatch by phase, compute one complete PPO-loss gradient per represented phase, and combine those gradients before the optimizer step. Our distillation controls follow two complementary directions of policy transfer. The specialist-to-shared control compresses a phase-specific policy ensemble into one phase-conditioned student, following policy distillation \citep{rusu2015distillation}. The shared-to-specialist control adds a decaying KL penalty toward a pretrained shared teacher while each specialist continues reinforcement learning, following the teacher-guided training principle of kickstarting \citep{schmitt2018kickstarting}. It is also related to Distral, although we do not jointly learn a shared policy centroid\citep{teh2017distral}.

\paragraph{Training and evaluation distinction.}
Shared and multihead agents train on complete cyclic trajectories and retain the physical state at phase boundaries. A standard multi-policy instead trains in a phase-fixed environment whose episode lasts $T$ steps and begins from the domain's phase-local reset law. Specialized multi-policies are updated separately and receive approximately $B_m=B/M$ interactions under a balanced total budget $B$. Crucially, all architectures are evaluated on the same complete cyclic environment: selecting a new specialized policy does not reset the physical state. The resulting score therefore includes any mismatch between phase-local training resets and deployment-time entry states.

\paragraph{PPO hyperparameters.}
Hyperparameters are held fixed across architectures within each reported comparison. We use the domain-specific settings in Table~\ref{tab:ppo-hyperparameters}. The different rollout sizes reflect the substantially different simulation horizon scales of HAPS and the two diagnostic domains.

\begin{table}[H]
\centering
\small
\caption{Default PPO settings. %HAP normalizes observations and training rewards; evaluation uses frozen observation statistics and raw rewards. PhasePendulum and PhaseGridWorld do not use \texttt{VecNormalize}.
}
\label{tab:ppo-hyperparameters}
\begin{tabular}{lccc}
\toprule
Setting & HAPS & PhasePendulum & PhaseGridWorld \\
\midrule
Hidden layers & $[256,256]$ & $[256,256]$ & $[256,256]$ \\
Parallel environments & 8 & 8 & 8 \\
Rollout steps per environment & 2048 & 256 & 256 \\
Minibatch size & 1024 & 256 & 256 \\
Epochs per update & 10 & 10 & 10 \\
Discount $\gamma$ & 0.99 & 0.99 & 0.99 \\
GAE parameter & 0.95 & 0.95 & 0.95 \\
Learning rate & $3\!\times\!10^{-4}$ & $3\!\times\!10^{-4}$ & $3\!\times\!10^{-4}$ \\
Entropy coefficient & 0.02 & 0.01 & 0.01 \\
PPO clip range & 0.2 & 0.2 & 0.2 \\
Target KL & 0.02 & 0.03 & 0.03 \\
Observation/reward normalization & yes/yes & no/no & no/no \\
\bottomrule
\end{tabular}
\end{table}

\subsubsection{HAPS Coverage Environment}
\label{app:hap-details}

\paragraph{Scope and geometry.}
We consider only the mobility and coverage component of the simulation environment introduced by \cite{loussouarn2026haps}. The idea is to find the optimal position of the HAPS given that the user density and traffic vary over time but in a periodic manner. In particular, different phases inside a period are differentiated by different user densities and traffic loads. The entire service region of the HAPS network is divided into multiple cells; only a few selected cells act as hotspots where the user density and traffic are higher than the rest of the cells. 

The service region is a square of side $A=300\,\mathrm{km}$, partitioned into $5\times5$ equal cells. Unless an ablation states otherwise, $H=2$ platforms operate at altitudes $z_h\in[18,20]\,\mathrm{km}$ with $\theta_{\max}=60^\circ$. Their footprint radii therefore lie between $31.2$ and $34.6\,\mathrm{km}$. Each step permits horizontal increments of $3\,\mathrm{km}$ per coordinate and a vertical increment of $3(20/300)=0.2\,\mathrm{km}$, followed by clipping to the service and altitude bounds. The per-HAPS service maximum capacity is $200$ users.
For grid cell $c=(i,j)$, the number of new users is sampled independently at timestep $t$ with phase $m=\sigma(t)$ as $N_c^t\sim\mathrm{Poisson}(\lambda_c^{m})$. Their horizontal coordinates are uniform within the cell. Their lifetimes are independent geometric random variables with parameter $\Xi_c$, supported on the positive integers. The base arrival-rate and departure-parameter arrays are
\begin{equation}
\Lambda=\begin{bmatrix}
1.60&1.20&1.40&1.10&1.30\\
1.30&1.19&1.70&1.40&1.20\\
1.20&1.50&1.30&1.60&1.40\\
1.10&1.40&1.20&1.30&1.50\\
1.40&1.30&1.60&1.20&1.10
\end{bmatrix},\quad
\Xi=\begin{bmatrix}
0.10&0.08&0.20&0.12&0.09\\
0.12&0.07&0.15&0.14&0.11\\
0.08&0.10&0.05&0.06&0.04\\
0.11&0.09&0.18&0.13&0.10\\
0.09&0.11&0.12&0.08&0.07
\end{bmatrix}.
\label{eq:hap-traffic-matrices}
\end{equation}
In phase $m$, the two selected cells for being hotspots use $\lambda_c^m=2\Lambda_c$, and for all other cells, we use $\lambda_c^m=0.1\Lambda_c$. The lifetime parameters $\Xi_c$ do not change with phase. The mean lifetime, $1/\Xi_c\in[5,25]$ steps, determines how long demand generated in the previous phase should persist after a boundary. The different actions from the HAPS lead to movement in 3D space, and the reward depends only on coverage and potential interference. We perform an ablation study for different phase count, dwell time, budget, and architecture.% without mixing the phase-allocation question with radio-resource allocation.

\paragraph{Phase Structure and Episodes.}
Each episode is composed of two full cycles. A cycle is composed of the $M$ different phases that follow each other in increasing index order. Each phase lasts $T$ time steps, making the whole episode last $2MT$ time steps. While designing the phases, it is possible to select the hotspot cells in such a way that the distance between a hotspot from one phase and the other is high, low, or medium. \emph{neighbour} places consecutive hotspots in adjacent cells, \emph{random} draws them uniformly, and \emph{maxjump} places them as far apart as the grid allows. These three dispositions correspond to increasing heterogeneity between consecutive phases.

% In the experiments, the phases are chosen according to three different dispositions, as such that the hotspots between two phases are neighbours, random, and maximum distance (maxjump). These different dispositions could represent different degrees of heterogeneity between phases. \textcolor{red}{See where to talk about random is maybe the more heterogeneous}

\paragraph{Association and reward.}
%For every user, the simulator first identifies all HAP footprints containing that user. It counts the user in interference i.e., when it is under at least two footprints. The total number of interfering users is noted $I^t$. It then considers covering HAPs in increasing distance order and assigns the user to the first one with remaining capacity. Uncovered users and users for which every covering HAP is full remain unserved. The resulting assigned-user counts are $C_h^t$. We use $\kappa=0.5$ as an overlap is discouraged even if the user is eventually assigned to only one HAP.

The objective of the problem is to maximize the number of covered users over time by the HAPS while minimizing the interference that can be caused by two similar HAPS serving the same users. Let $C_h^t$ be the number of users assigned to HAPS $h$ and let $I^t$ be the total number of connected users. The coverage reward is $ r_t=\sum_{h=1}^{H}C_h^t-\kappa I^t,$ where $\kappa=0.5.$

\paragraph{Observation and action spaces.}
Let $U_t$ be the number of users present in the environment at time $t$; we have $U_t^+=\max\{1,|\mathcal U_t|\}$, and let $G_h^t$ be the number of users geometrically covered by HAPS $h$ before capacity-constrained assignment. We have $x_h, y_h, z_h$ as the coordinates of HAPS $h$. The physical observation is
\begin{equation}
s_t=\left[
\left\{
\frac{x_h^t}{A},\frac{y_h^t}{A},
\frac{z_h^t-z_{\min}}{z_{\max}-z_{\min}},
\frac{C_h^t}{U_t^+},\frac{G_h^t}{U_t^+}
\right\}_{h=1}^{H},
\bar\lambda^{m_t},\frac{I^t}{U_t^+}
\right],
\label{eq:hap-observation}
\end{equation}
where $\bar\lambda^{m_t}$ is the mean arrival rate over the $25$ cells. For $H=2$, this physical vector has dimension $12$. The phase-conditioned policy receives $[s_t,e_{m_t}]$, where $e_{m_t}\in\{0,1\}^{M}$ is one-hot. Clock variants append $[\sin(2\pi\tau_t/T),\cos(2\pi\tau_t/T)]$. The joint action space is $\{-1,0,1\}^{H}$ and controls all HAPS centrally.

\paragraph{Initialization and method-specific training.}
At the start of a full episode, the user set is empty, horizontal HAPS coordinates are sampled uniformly over the service area, and every altitude is set to $19\,\mathrm{km}$. Shared and multihead policies are trained on complete cyclic episodes, so their rollouts preserve HAPS positions and outstanding user sessions across every phase boundary. Standard multipolicy training instead uses one phase-fixed environment per specialist. Each specialist episode starts at its assigned phase with a fresh user set and lasts exactly $T$ steps. The specialists are updated in round-robin order, and each receives approximately $B/M$ of a total interaction budget $B$. All methods are nevertheless evaluated on the same complete cyclic process: at each boundary the evaluator selects the next specialist without resetting the environment. Handoff-aware variants replace the phase-fixed reset law with entry states collected from these full ensemble rollouts; we identify those variants explicitly when reported.

The main phase-count sweep uses $M\in\{2,5,10,20\}$ at $T=500$. The dwell-time sweep fixes $M=20$ and varies $T$. Training budgets denote total environment interactions, including interactions collected across the eight environments. The main phase-count results and the dwell-time sweep $T\in\{100,200,500,2000\}$ use seeds $\{0,1,2,3,4\}$; secondary ablations use seeds $\{0,1,2\}$ unless stated otherwise. Raw HAPS evaluations use three complete cyclic episodes per checkpoint. Because each episode preserves the state across boundaries, these scores include both settled phase behavior and handoff effects.

\subsubsection{PhasePendulum Environment and Protocol}
\label{app:pendulum-details}

\paragraph{Dynamics, observation, and reward.}
PhasePendulum environment consists of the standard torque-controlled pendulum dynamics\footnote{\url{https://gymnasium.farama.org/environments/classic_control/pendulum/}} while changing the target angle $g_m$ between phases. Its physical state is $(\theta_t,\dot\theta_t)$, its action is a scalar torque $u_t\in[-2,2]$, and its integration step is $\Delta t=0.05$. We use $g=10$, mass $1$, length $1$, and clip angular velocity to $[-8,8]$. The physical observation is
\[
s_t=[\cos\theta_t,\sin\theta_t,\dot\theta_t/8].
\]
Writing $\operatorname{ang}(x)=((x+\pi)\bmod 2\pi)-\pi$, the reward in phase $m$ is
\begin{equation}
r_t=-\operatorname{ang}(\theta_t-g_m)^2
     -0.1\dot\theta_t^2-0.001u_t^2.
\label{eq:pendulum-reward}
\end{equation}
Only the target, and therefore the reward, changes at a boundary; angle and angular velocity remain continuous. At a fresh episode reset, $\theta_0\sim\mathrm{Unif}[-\pi,\pi]$ and $\dot\theta_0\sim\mathrm{Unif}[-1,1]$.

\paragraph{Target structure and budgets.}
Each of the $M$ phases has its own target angle, and the phases repeat periodically in an  increasing index order. We control the heterogeneity between phases through the choice of successive targets. The maximum-jump structure greedily orders candidate target angles so as to maximize the circular distance between consecutive targets. For the continuous heterogeneity sweep, we construct one such structure $g^{\rm max}$ and set $g_m(h)=h g_m^{\rm max}$ for $h\in\{0,0.25,0.5,0.75,1\}$. At $h=0$ every phase has the same target; at $h=1$ the complete maximum-jump structure is recovered. We fix $M=10$, $T=1000$, five training seeds, four evaluation cycles, $5$M total interactions, and $20$ deterministic raw-evaluation episodes per checkpoint. Each specialist therefore receives $0.5$M interactions. Scaling targets also changes their absolute locations relative to the pendulum dynamics, so we compare architectures within each $h$ and do not interpret rewards across $h$ as a pure causal effect of heterogeneity.

The phase-count sweep uses $M\in\{2,5,10,20\}$, $T=1000$, the maximum-jump suite, five seeds, and a fixed allocation $B_0=1$M interactions per phase. Its total budget is therefore $MB_0$. This sweep tests scalability while holding specialist data approximately constant, not under a fixed total budget. The identical-phase budget study fixes $M=10, T=1000, h=0$ and varies the specialist allocation from $0.5$M to $2$M interactions. Shared and multihead policies train on full cyclic trajectories; specialists train from the fresh phase-local reset law above. All are evaluated for four complete cycles without resetting $(\theta,\dot\theta)$ at boundaries.

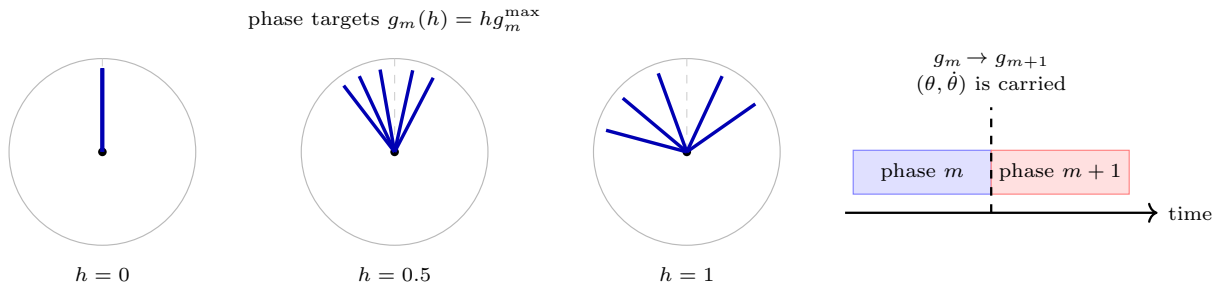
\begin{figure}[H]
\centering
\resizebox{0.98\linewidth}{!}{%
\begin{tikzpicture}[x=1cm,y=1cm,font=\scriptsize]
  % Target-angle contraction used by the continuous heterogeneity sweep.
  \foreach \cx/\hh in {0/0,3.6/0.5,7.2/1} {
    \draw[gray!55] (\cx,0) circle (1.15);
    \fill (\cx,0) circle (1.5pt);
    \draw[gray!35,dashed] (\cx,0) -- +(0,1.15);
    \node[below] at (\cx,-1.30) {$h=\hh$};
  }
  \foreach \dummy in {1,...,5} {
    \draw[very thick,blue!70!black] (0,0) -- +(0,1.03);
  }

  \foreach \a in {-55,-25,20,50,75} {
    \draw[very thick,blue!70!black] (3.6,0) -- ++({90+0.5*\a}:1.03);
    \draw[very thick,blue!70!black] (7.2,0) -- ++({90+\a}:1.03);
  }

  \node[align=center] at (3.6,1.62) {phase targets $g_m(h)=h g_m^{\rm max}$};

  \draw[->,thick] (9.15,-0.75) -- (13.0,-0.75) node[right] {time};
  \draw[fill=blue!12,draw=blue!50] (9.25,-0.52) rectangle (10.95,0.02);
  \draw[fill=red!12,draw=red!50] (10.95,-0.52) rectangle (12.65,0.02);
  \node at (10.10,-0.25) {phase $m$};
  \node at (11.80,-0.25) {phase $m+1$};
  \draw[dashed,thick] (10.95,-0.75) -- (10.95,0.55);
  \node[above,align=center] at (10.95,0.58)
    {$g_m\!\rightarrow g_{m+1}$\\$(\theta,\dot\theta)$ is carried};
\end{tikzpicture}
}
\caption{PhasePendulum construction. }
\label{fig:pendulum-construction}
\end{figure}
Fig.\ref{fig:pendulum-construction} presents the objective angles for different scenarios. On the left of the figure, we have the examples for different heterogeneity $h$ at $M=5$ phases. At $h=0$, all phase objectives coincide; increasing $h$ spreads them around the circle. On the right we have a detail of the between-phase transition where only the target changes. The angle and angular velocity remain continuous, so the new phase inherits the state
produced by the preceding policy.

\subsubsection{PhaseGridWorld Environment and Protocol}
\label{app:gridworld-details}

\paragraph{Base environment.}
The agent moves on a $9\times9$ grid using four deterministic actions: up, down, left, and right. Each phase has its own goal and obstacle map. The physical observation contains only the agent coordinates, each scaled to $[-1,1]$. Shared methods append the phase code, clock variants append the sine--cosine clock, and multipolicy selects a specialist externally. Goal and obstacle coordinates are never observed directly. The per-step reward is
\begin{equation}
r_t=-0.01-0.05\,\mathbf{1}\{\text{collision}\} +0.02(d_t-d_{t+1})+\mathbf{1}\{d_{t+1}=0\},
\label{eq:gridworld-reward}
\end{equation}
where $d_t$ is the Manhattan distance to the active goal. Reaching the goal resamples a reachable non-goal cell within the same phase.

At a normal boundary, the agent position is carried into the next map. If the inherited cell is an obstacle in that map, a safety rule resamples the agent uniformly from free cells connected to the incoming goal. Thus, obstacle collisions are repaired, but a free inherited cell in a disconnected component is retained. The boundary-reset intervention instead performs this reachable-state resampling at every boundary. We call these two boundary laws \emph{carried} and \emph{reset}. The same law is used for every architecture within a comparison.

\paragraph{Ordinary suites and distractors.}
Maximum-jump goal suites greedily maximize Euclidean separation between successive goals. Maximum-jump obstacle suites similarly select successive Bernoulli obstacle fields to maximize their normalized Hamming distance. The main ordinary runs use $M=10$, $T=500$, four cycles, obstacle densities $0.25$ and $0.40$, seeds $\{0,1,2\}$, and $B_0=1.5$M interactions per phase. A separate experiment uses $M=10, T=100$, three seeds, and appends $Z\in\{0,8,32,128\}$ independent $\mathrm{Unif}[-1,1]$ coordinates to every observation. These distractors do not affect transitions or rewards. Figure \ref{fig:gridworld-density-examples} presents examples of environments for the two densities. 

\begin{figure}[H]
\centering
\includegraphics[width=0.96\linewidth]{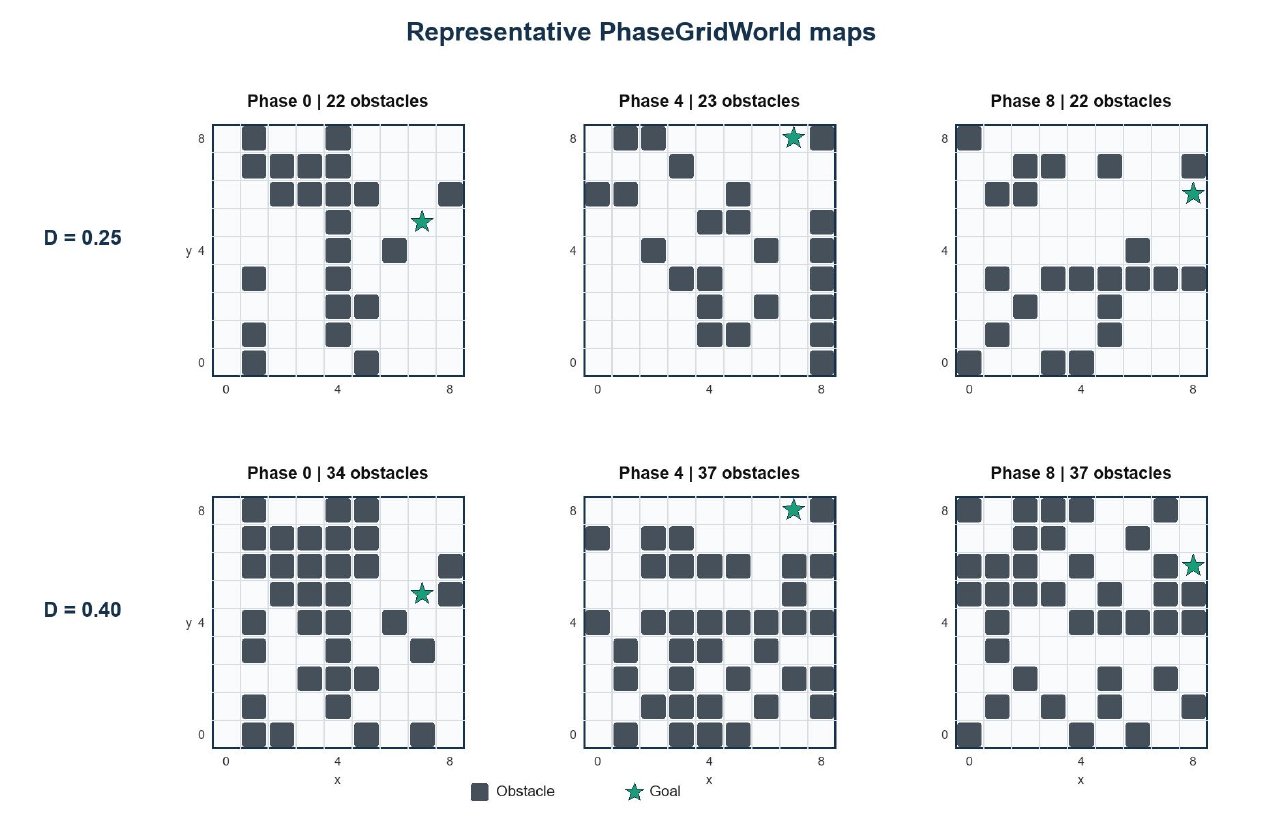}
\caption{Representative maps from the ordinary PhaseGridWorld maximum jump suite used in the reported experiments ($9\times9$, $M=10$, suite and obstacle seed 0). Columns show the same three phase indices at obstacle sampling probabilities $D=0.25$ and $D=0.40$. Obstacles are sampled independently before the maximum-jump ordering is selected, and the active goal cell is always cleared; consequently, the realized obstacle count need not equal exactly $81D$. These panels illustrate environment structure rather than policy performance.}
\label{fig:gridworld-density-examples}
\end{figure}

\paragraph{Controlled boundary topologies.}
To study in detail the handoff between phases, we built the following environment. The aligned, bottleneck, and closed-trap experiments use $M=2$, $T=100$, five training seeds, $3$M total interactions, and 20 evaluation cycles per episode. In the aligned case both phases use goal $(7,4)$. In the other two cases the goal changes from $(1,4)$ to $(7,4)$. Phase 1 adds the vertical wall $x=4$; the bottleneck opens the single gate $(4,7)$, whereas the trap leaves the wall closed. For the boundary intervention, each trained checkpoint is evaluated under both carried and reachable-state-reset laws. The policy parameters are unchanged, so the contrast isolates the effect of the entry distribution. This specific experiment is illustrated in Fig. \ref{fig:gridworld-handoff-topologies}

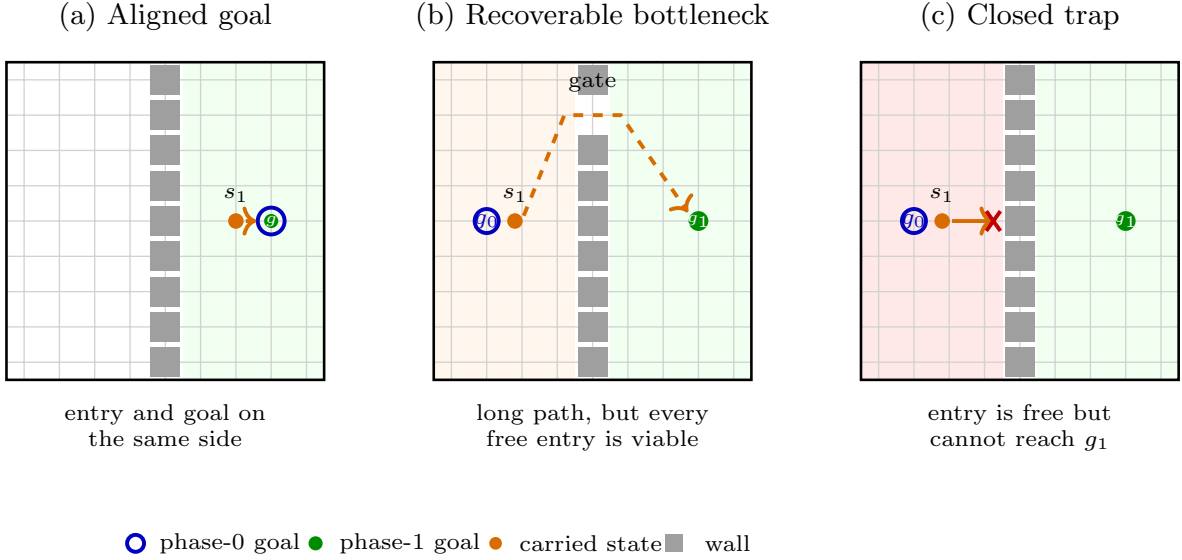
\begin{figure}[H]
\centering
\resizebox{0.95\linewidth}{!}{%
\begin{tikzpicture}[x=0.38cm,y=0.38cm,font=\scriptsize]
  % Aligned: both phases use the right-hand goal; phase 1 adds the wall.
  \begin{scope}[xshift=0cm]
    \fill[green!6] (4.5,-0.5) rectangle (8.5,8.5);
    \draw[step=1,gray!35,thin] (-0.5,-0.5) grid (8.5,8.5);
    \draw[black,thick] (-0.5,-0.5) rectangle (8.5,8.5);
    \foreach \y in {0,...,8} {
      \fill[gray!75] (3.58,\y-0.42) rectangle (4.42,\y+0.42);
    }
    \draw[blue!75!black,very thick] (7,4) circle (0.38);
    \fill[green!55!black] (7,4) circle (0.20);
    \node[white,font=\tiny] at (7,4) {$g$};
    \fill[orange!85!black] (6,4) circle (0.22);
    \node[above] at (6,4.22) {$s_1$};
    \draw[->,orange!85!black,very thick] (6.28,4) -- (6.58,4);
    \node[above,font=\small] at (4,9.05) {(a) Aligned goal};
    \node[below,align=center] at (4,-0.85) {entry and goal on\\the same side};
  \end{scope}

  % Bottleneck: opposite goals, with one gate in the phase-1 wall.
  \begin{scope}[xshift=4.6cm]
    \fill[orange!7] (-0.5,-0.5) rectangle (3.5,8.5);
    \fill[green!6] (4.5,-0.5) rectangle (8.5,8.5);
    \draw[step=1,gray!35,thin] (-0.5,-0.5) grid (8.5,8.5);
    \draw[black,thick] (-0.5,-0.5) rectangle (8.5,8.5);
    \foreach \y in {0,1,2,3,4,5,6,8} {
      \fill[gray!75] (3.58,\y-0.42) rectangle (4.42,\y+0.42);
    }
    \draw[blue!75!black,very thick] (1,4) circle (0.34);
    \node[blue!75!black] at (1,4) {$g_0$};
    \fill[green!55!black] (7,4) circle (0.28);
    \node[white,font=\tiny] at (7,4) {$g_1$};
    \fill[orange!85!black] (1.8,4) circle (0.22);
    \node[above] at (1.8,4.22) {$s_1$};
    \draw[->,orange!85!black,very thick,dashed]
      (2.05,4.15) -- (3.2,7) -- (4.8,7) -- (6.7,4.25);
    \node[above] at (4,7.35) {gate};
    \node[above,font=\small] at (4,9.05) {(b) Recoverable bottleneck};
    \node[below,align=center] at (4,-0.85) {long path, but every\\free entry is viable};
  \end{scope}

  % Trap: the same wall without a gate disconnects the old-goal side.
  \begin{scope}[xshift=9.2cm]
    \fill[red!9] (-0.5,-0.5) rectangle (3.5,8.5);
    \fill[green!6] (4.5,-0.5) rectangle (8.5,8.5);
    \draw[step=1,gray!35,thin] (-0.5,-0.5) grid (8.5,8.5);
    \draw[black,thick] (-0.5,-0.5) rectangle (8.5,8.5);
    \foreach \y in {0,...,8} {
      \fill[gray!75] (3.58,\y-0.42) rectangle (4.42,\y+0.42);
    }
    \draw[blue!75!black,very thick] (1,4) circle (0.34);
    \node[blue!75!black] at (1,4) {$g_0$};
    \fill[green!55!black] (7,4) circle (0.28);
    \node[white,font=\tiny] at (7,4) {$g_1$};
    \fill[orange!85!black] (1.8,4) circle (0.22);
    \node[above] at (1.8,4.22) {$s_1$};
    \draw[->,orange!85!black,very thick] (2.08,4) -- (3.25,4);
    \draw[red!75!black,very thick] (3.05,3.72) -- (3.45,4.28);
    \draw[red!75!black,very thick] (3.05,4.28) -- (3.45,3.72);
    \node[above,font=\small] at (4,9.05) {(c) Closed trap};
    \node[below,align=center] at (4,-0.85) {entry is free but\\cannot reach $g_1$};
  \end{scope}

  \begin{scope}[xshift=1.2cm,yshift=-1.95cm]
    \draw[blue!75!black,very thick] (0,0) circle (0.24);
    \node[right] at (0.35,0) {phase-0 goal};
    \fill[green!55!black] (5.1,0) circle (0.20);
    \node[right] at (5.45,0) {phase-1 goal};
    \fill[orange!85!black] (10.2,0) circle (0.18);
    \node[right] at (10.55,0) {carried state};
    \fill[gray!75] (15.0,-0.25) rectangle (15.5,0.25);
    \node[right] at (15.8,0) {wall};
  \end{scope}
\end{tikzpicture}
}
\caption{Controlled PhaseGridWorld boundary constructions at the start of phase 1. The $9\times9$ geometry, goals, wall, and bottleneck gate match the environment implementation. The orange marker is an illustrative carried entry state, not a sampled trajectory. In the aligned case, the goal remains at $(7,4)$. In the other cases, it switches from $(1,4)$ to $(7,4)$. Opening the gate at $(4,7)$ turns a disconnected trap into a recoverable, but long-path, bottleneck.}
\label{fig:gridworld-handoff-topologies}
\end{figure}

\paragraph{Two-goal handoff choice.}
In addition to the above handoff study environment, we created a two-goal handoff environment based on the theory proposed in Appendix \ref{sec:two-goal-handoff} and illustrated by Fig. \ref{fig:gridworld-handoff-schematic}. The causal choice construction also uses $M=2, T=100$ and a closed phase-1 wall. Phase 0 contains a primary goal $g_{00}=(1,4)$ worth $1$ and a secondary safe goal $g_{01}=(7,4)$ worth $\rho\in\{0.3,0.6,0.9\}$. Phase 1 uses $g_1=(7,6)$. Upon reaching either phase-0 goal, the agent remains there until the scheduled boundary. Hence $g_{00}$ produces a free but unreachable phase-1 entry, while $g_{01}$ preserves reachability. Distance shaping is disabled in this construction so that the comparison is driven by goal reward and continuation value. We train Single w/ clock, multihead w/ clock, and multipolicy w/ clock for $3$M total interactions over five seeds and evaluate ten episodes per checkpoint. Scripted controls additionally enumerate all 79 non-goal initial cells and follow shortest paths to either $g_{00}$ or $g_{01}$; they involve no RL training.

\begin{figure}[H]
\centering
\begin{tikzpicture}[x=0.48cm,y=0.48cm,font=\scriptsize]
  \begin{scope}
    \draw[step=1,gray!35,thin] (-0.5,-0.5) grid (8.5,8.5);
    \draw[black,thick] (-0.5,-0.5) rectangle (8.5,8.5);
    \draw[red!75!black,very thick] (1,4) circle (0.34);
    \node[red!75!black,above] at (1,4.35) {$g_{00}:1$};
    \fill[green!55!black] (7,4) circle (0.28);
    \node[green!40!black,above] at (7,4.35) {$g_{01}:\rho$};
    \draw[->,red!75!black,thick] (4,3.45) -- (1.45,3.85);
    \draw[->,green!45!black,thick] (4,4.55) -- (6.55,4.15);
    \node[align=center,font=\small,text width=4.5cm,anchor=south] at (4,8.85)
      {\textbf{Phase 0}\\Immediate vs. safe handoff};
  \end{scope}
  \begin{scope}[xshift=5.75cm]
    \fill[red!8] (-0.5,-0.5) rectangle (3.5,8.5);
    \fill[green!6] (4.5,-0.5) rectangle (8.5,8.5);
    \draw[step=1,gray!35,thin] (-0.5,-0.5) grid (8.5,8.5);
    \draw[black,thick] (-0.5,-0.5) rectangle (8.5,8.5);
    \foreach \y in {0,...,8} {
      \fill[gray!75] (3.58,\y-0.42) rectangle (4.42,\y+0.42);
    }
    \fill[red!75!black] (1,4) circle (0.22);
    \node[red!75!black,below] at (1,3.65) {trapped};
    \draw[red!75!black,very thick] (2.9,3.7) -- (3.4,4.2);
    \draw[red!75!black,very thick] (2.9,4.2) -- (3.4,3.7);
    \fill[green!55!black] (7,4) circle (0.22);
    \draw[->,green!45!black,thick] (7,4.35) -- (7,5.55);
    \draw[blue!75!black,very thick] (7,6) circle (0.32);
    \node[blue!75!black,right] at (7.35,6) {$g_1$};
    \node[align=center,font=\small,text width=4.5cm,anchor=south] at (4,8.85)
      {\textbf{Phase 1}\\Entry state controls reachability};
  \end{scope}
  \draw[->,very thick] (9.1,4) -- (11.0,4);
  \node[above] at (10.05,4) {boundary};
\end{tikzpicture}
\caption{Controlled two-phase handoff environment. In phase 0, taking the
larger local reward $g_{00}$ places the agent in a component that becomes
disconnected from $g_1$ after the boundary. Taking the lower reward $g_{01}$
preserves phase-1 reachability.}
\label{fig:gridworld-handoff-schematic}
\end{figure}
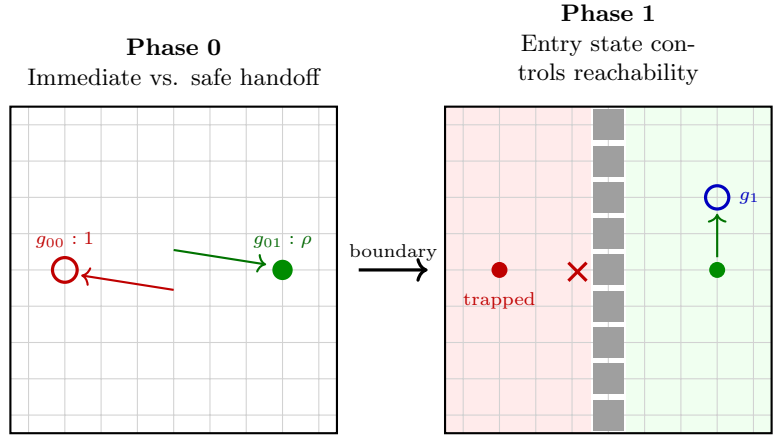

\subsubsection{Evaluation Metrics and Protocol}
\label{app:eval}
Evaluations report unweighted cross-seed mean raw returns at the final training checkpoint, with matched-seed differences and paired $95\%$ percentile bootstrap confidence intervals. Diagnostic gradient evaluations (Appendix~\ref{app:local-optimization}) compute phase-local PPO surrogates using $1024$ environment interactions per phase with global advantage normalization. Sample budgets strictly denote environment interaction steps rather than minibatch updates.
\subsection{Extended Empirical Results and Diagnostics}
\label{app:extended-results}

This section contains the longer empirical comparisons that are summarized in the main text. We keep confirmatory results separate from qualitative checks. All numerical comparisons use the evaluation and uncertainty conventions in Appendix~\ref{app:experiments}.

\subsubsection{HAPS Experiments Results}
\label{app:HAP-Res}

\paragraph{HAPS Main results}
In HAPS, transient dynamics are governed by platform relocation and user persistence across phase boundaries. An HAPS platform moves $3\,\mathrm{km}$ per step per coordinate in a $300\,\mathrm{km} \times 300\,\mathrm{km}$ operational area, meaning direct transit requires at most $100$ steps when both spatial coordinates change simultaneously. Meanwhile, mean user connection lifetimes range from $5$ to $25$ steps. We then take $L_m =100$. Moreover, reward differences are mostly caused when one or two HAPS do not reach a high user density zone. Then any difference between two policies should be in the same order. Following this observation, we propose the following approximation: $G_m \approx K_m$, transforming \eqref{eq:phasewise-predictive-threshold} into $\lambda_m < \frac{1}{2}$. And for $\gamma$ close to one, we have $\lambda_m \approx  \frac{L_m}{T_m}$; this motivates $T=200$ as a diagnostic transient recovery horizon from \eqref{eq:phasewise-predictive-threshold}, though not as a strict value for $L_m$, since user arrivals are stochastic and learned navigation trajectories need not be straight lines.

At fixed $M=20$, varying the dwell time $T$ fundamentally shifts method rankings (Figure~\ref{fig:hap-regime}, right). At short dwell times ($T=100$), the Single+clock shared architecture leads ($31.75$ vs.\ $20.88$ for Multipolicy). At $T=200$, Multihead and Multipolicy perform comparably ($32.26$ and $31.90$, respectively) with no consistent seed-level winner, slightly outperforming Single+clock. As dwell time increases further, Multipolicy establishes a decisive lead at $T=500$ ($42.36$ vs.\ $34.15$ for Multihead) and $T=2000$ ($45.76$ vs.\ $33.16$).

\begin{figure}[t]
\centering
\includegraphics[width=\linewidth]{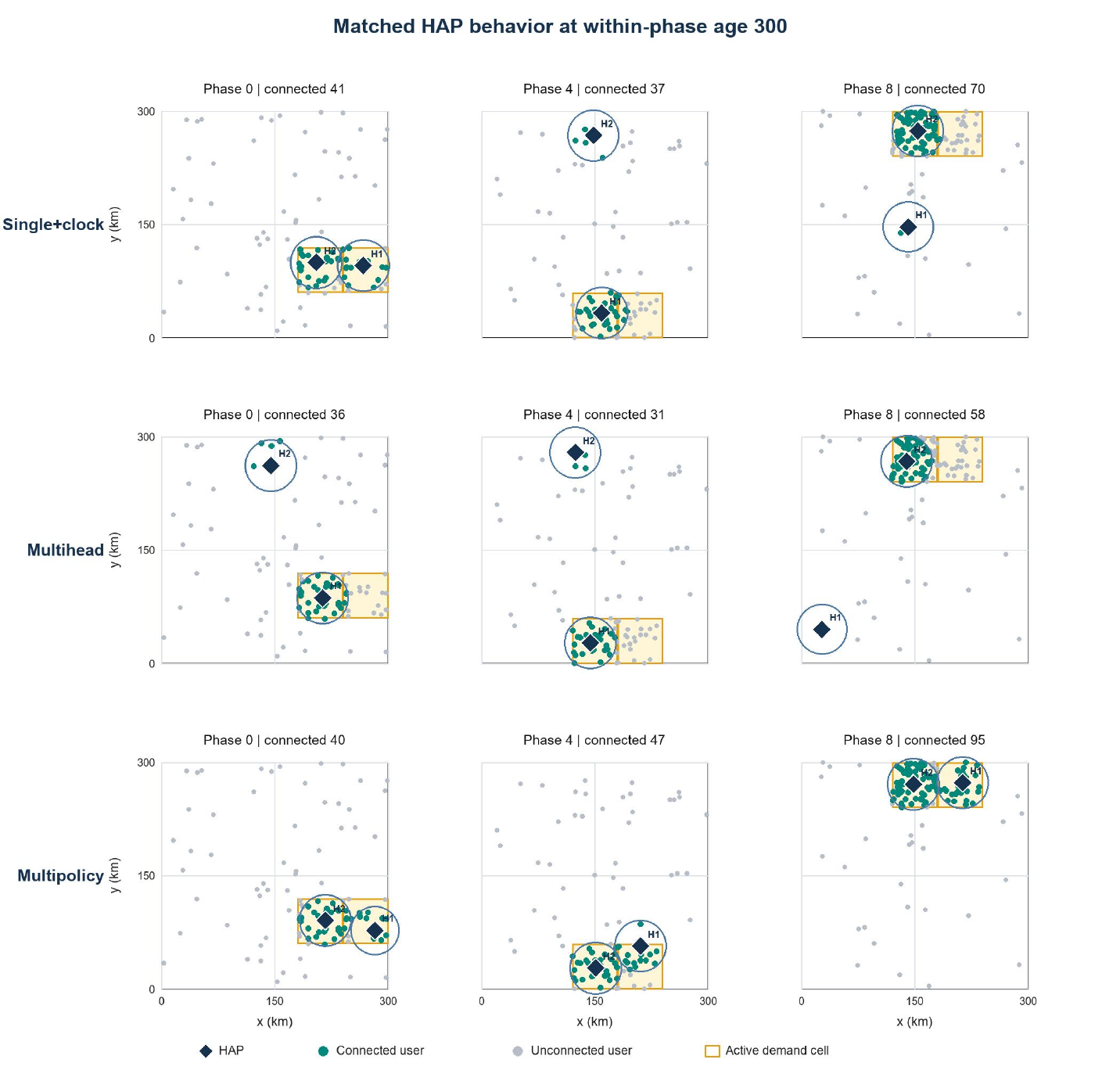}
\caption{Matched HAPS behavior within-phase at age $300$ for five evenly spaced phases. Rows are architectures and columns are phases. The phase subset $\{0,4,8\}$ was fixed by index rather than selected by performance. All methods see the same demand realization. This single-seed visualization makes phase-wise adaptation directly comparable, while Figure~\ref{fig:hap-regime} reports the confirmatory five-seed returns.}
\label{fig:hap-selected-snapshots}
\end{figure}

In addition to the three core architectures, we evaluated multi-task optimization baselines, extended sample budgets, and policy distillation methods on HAPS at $M=20$ and $T=500$. As summarized in Table~\ref{tab:hap-main}, several key insights emerge:

\paragraph{Budget scaling and sample efficiency.}
Increasing the shared-policy training budget up to $20\text{M}$ steps fails to close the high-$M$ performance gap relative to phase specialists. Specialists outperform the shared baseline even at the smallest evaluated per-phase budget ($B_0 = 250\text{k}$ steps, corresponding to $5\text{M}$ total environment interactions) and show only modest gains with additional samples. We observe no late-budget crossover: the advantage of parameter isolation significantly outweighs the data-efficiency cost of splitting sample allocations across specialists.

\paragraph{Multi-task gradient surgery.}
Gradient surgery methods—specifically PCGrad \citep{yu2020pcgrad} and CAGrad \cite{liu2021cagrad} improve upon the standard single-policy baseline at $5\text{M}$ steps ($28.87$ and $28.05$ vs.\ $25.54$). However, both methods remain substantially below the multi-policy score ($43.33$), indicating that gradient conflict mitigation alone cannot resolve the representation bottleneck in high-phase setups.

\paragraph{Asymmetry in policy distillation.}
Policy distillation exhibits a stark structural asymmetry: while compressing a specialist ensemble into a single shared student fails, using a shared teacher to regularize specialist training preserves full performance while guiding early learning. 

Specifically, under \textbf{Multipolicy+clock + teacher KL}, a frozen shared policy pretrained for $1\text{M}$ steps serves as a global teacher $\pi_{\mathrm{teacher}}(a \mid s, m, \tau)$. Each fresh phase specialist $\pi_m(a \mid s, \tau)$ is trained on its phase-fixed environment with PPO, supplemented by an alternating KL-divergence update against the teacher's action distribution. The shared teacher guides early exploration without constraining the specialist's final capacity. Reallocating $1\text{M}$ steps to teacher pretraining leaves $\approx 950\text{k}$ steps per specialist, strictly preserving the $20\text{M}$ total interaction budget. This regularized ensemble achieves $46.58$, matching the standard unconstrained specialist score ($46.72$).

In contrast, distilling the $20\text{M}$-step specialist ensemble back into a single shared student yields only $12.95$. Even when scaling the student architecture and allocating $5\text{M}$ dedicated student steps, performance reaches only $27.73$—remaining below both the un-distilled shared baseline ($30.10$) and the specialist teachers ($46.72$). This demonstrates that the representation bottleneck of a single shared network cannot be overcome simply by imitating optimal specialist behaviors post hoc.

\begin{table}[t]
\centering
\small
\caption{Performance summary on HAPS ($M=20, T=500$). Values denote mean raw reward per step averaged across random seeds $\{0, 1, 2\}$. Bold indicates the highest score.}
\label{tab:hap-main}
\begin{tabular}{lcc}
\toprule
Budget / Setup & Method & Reward/step \\
\midrule
$5\text{M}$ budget & Single+clock / Multihead / Multipolicy+clock & 26.16 / 33.52 / 43.33 \\
$10\text{M}$ budget & Single+clock / Multihead / Multipolicy+clock & 31.16 / 35.41 / 46.39 \\
$20\text{M}$ budget & Single+clock / Multihead / Multipolicy+clock & 30.10 / 37.37 / \textbf{46.72} \\
$20\text{M}$ budget & Multipolicy without clock & 46.57 \\
\midrule
$5\text{M}$ gradient surgery & Single+clock & 25.54 \\
$5\text{M}$ gradient surgery & PCGrad-style & 28.87 \\
$5\text{M}$ gradient surgery & CAGrad-style & 28.05 \\
\midrule
$20\text{M}$ distillation & Multipolicy+clock + teacher KL & 46.58 \\
$20\text{M}$ distillation & Single+clock distilled from specialists & 12.95 \\
$20\text{M}$ larger student & Big single+clock distilled ($5\text{M}$ student steps) & 27.73 \\
\bottomrule
\end{tabular}
\end{table}

\paragraph{Capacity and depth.}
Table~\ref{tab:hap-capacity-depth} summarizes the capacity control experiments supporting Hypothesis H2 of Sec. \ref{sec:experiments}. Increasing the width (``Big'') or depth (``Deep'') of the shared policy architecture fails to close the performance gap relative to the specialist ensemble at $M=20$ and $T=500$. While these results rule out parameter count of the default shared model as a simple explanatory factor, we note that this is not a strict equal-compute comparison: specialist ensembles store significantly more parameters overall, and deeper PPO networks are notoriously prone to optimization instability. Indeed, additional six- and eight-layer depth-stress runs exhibited severe optimization instability across seeds ($29.04$ and $23.30$, respectively) rather than approaching specialist performance.

\begin{table}[t]
\centering
\small
\caption{HAPS capacity and depth controls at $M=20, T=500$. Values denote mean reward per step across three seeds. ``Big'' doubles hidden width; ``Deep'' uses four hidden layers. The six- and eight-layer entries denote depth-stress runs.}
\label{tab:hap-capacity-depth}
\begin{tabular}{lrlr}
\toprule
Architecture & Reward/step & Architecture & Reward/step \\
\midrule
Single, default & 26.04 & Single+clock, default & 27.06 \\
Single, big & 22.48 & Single+clock, big & 24.27 \\
Single, deep & 24.73 & Single+clock, deep & 29.51 \\
Single+clock, 6 layers & 29.04 & Single+clock, 8 layers & 23.30 \\
Multihead, default & 32.76 & Multihead, deep & 27.47 \\
Multipolicy & \textbf{42.34} & & \\
\bottomrule
\end{tabular}
\end{table}

\paragraph{Phase ordering and handoff variants.}
To isolate the effect of temporal regularity versus transition difficulty, we evaluated three distinct phase ordering schedules at $M=20, T=500$ (Table~\ref{tab:hap-order-handoff}). Navigating adjacent phases proves easiest across all methods. Conversely, random and maximum-jump transitions degrade overall returns while widening the performance advantage of the multi-policy ensemble over multi-head baselines. Because altering phase order shifts both boundary state displacement and temporal regularity, this sensitivity highlights that phase count alone does not fully capture domain difficulty. An exploratory $M=30, T=500$ extension yields mean rewards of $23.80$, $32.52$, and $37.63$ for Single+clock, Multihead, and Multipolicy, respectively, preserving the relative performance ordering.

Our preliminary specialist-coordination mechanisms do not resolve the boundary handoff problem. Under a $20\text{M}$-step sample budget, appending a within-phase clock to standard multi-policy leaves performance essentially unchanged, whereas explicit handoff and handoff-warm-start objectives degrade performance across both tested dwell times ($T \in \{200, 500\}$). These negative results demonstrate that naive handoff regularizers can disrupt phase-local learning, highlighting the need for principled continuation-value coordination.

Similarly, a budget-matched warm-start protocol proved unsuccessful. In this setup, a non-augmented policy is pretrained for $B_0$ interactions, after which its weights are cloned into all $M$ specialists for fine-tuning over $B_0(M-1)/M$ steps (maintaining a total interaction budget of $MB_0$). As shown in Table~\ref{tab:hap-warm-handoff}, shared initialization and its handoff-aware extension significantly reduce performance across all tested settings ($M \in \{10, 20\}$, $T \in \{500, 1000\}$). This indicates that initializing specialists with shared parameters introduces negative transfer during fine-tuning, rather than serving as a beneficial warm start.

\begin{table}[t]
\centering
\small
\caption{Exploratory HAPS phase-order and handoff controls (means over three seeds). The phase-order sweep uses $M=20, T=500$ under the original sample budget. The handoff sweep uses $M=20$ under a $20\text{M}$ total interaction budget ($1\text{M}$ per phase).}
\label{tab:hap-order-handoff}
\begin{tabular}{llrrr}
\toprule
Control & Setting & Single+clock & Multihead & Multipolicy \\
\midrule
Phase order & Adjacent & 34.58 & 46.80 & \textbf{48.41} \\
Phase order & Random & 27.47 & 38.17 & \textbf{40.59} \\
Phase order & Maximum jump & 31.17 & 34.88 & \textbf{41.06} \\
\midrule
Clock/handoff, $T=200$ & Standard / clock-aware specialist & -- & -- & 37.47 / \textbf{38.50} \\
Clock/handoff, $T=200$ & Handoff / handoff+warm start & -- & -- & 31.75 / 27.14 \\
Clock/handoff, $T=500$ & Standard / clock-aware specialist & -- & -- & 46.35 / \textbf{46.75} \\
Clock/handoff, $T=500$ & Handoff / handoff+warm start & -- & -- & 36.69 / 27.29 \\
\bottomrule
\end{tabular}
\end{table}

\begin{table}[t]
\centering
\small
\caption{Exploratory HAPS warm-start controls with $B_0=1\text{M}$ interactions per phase under equal total budgets ($MB_0$). Values denote mean reward per step across seeds $\{0, 1, 2\}$. ``Warm'' denotes budget-matched pretraining followed by specialist fine-tuning.}
\label{tab:hap-warm-handoff}
\begin{tabular}{rrcccc}
\toprule
$M$ & $T$ & Multipolicy & Warm & Handoff & Handoff+warm \\
\midrule
10 & 500  & \textbf{44.84} & 16.82 & 32.91 & 17.67 \\
10 & 1000 & \textbf{48.21} & 14.50 & 39.55 & 14.13 \\
20 & 500  & \textbf{46.37} & 13.44 & 37.14 & 11.27 \\
20 & 1000 & \textbf{48.60} & 11.14 & 41.69 & 9.92 \\
\bottomrule
\end{tabular}
\end{table}

Table~\ref{tab:corrected-gradient-diagnostic} reports the actor-gradient diagnostic measurements based on appendix \ref{app:local-optimization}. Across all phase counts $M$, the aggregate alignment statistic $A_g$ remains close to its orthogonal baseline ($A_g = 1$), with every five-seed $95\%$ confidence interval containing unity. Furthermore, the aggregate update harms at most $2\%$ of phases ($f_{\mathrm{harm}} \le 0.02$). Consequently, the declining phase-to-update cosine similarity reflects gradient \emph{dilution} across high-dimensional, distinct task directions rather than an aggregate update vector that actively opposes individual phase objectives. Although $44$--$60\%$ of pairwise gradient cosines are negative ($f_{\mathrm{pair}}^{-} \approx 0.5$), their mean ($\overline{\cos(g_m, g_k)}$) remains close to zero, as positive cross-terms effectively offset negative ones. These findings demonstrate that the performance gap in high-$M$ shared architectures cannot be attributed to destructive aggregate gradient interference at the evaluated training checkpoints.

\begin{table}[t]
\centering
\small
\caption{Corrected HAPS actor-gradient diagnostic at $T=500$. Values denote means over five seeds. $A_g=1$ corresponds to the mutually orthogonal baseline ($C_g = 1 - A_g$), and $f_{\mathrm{harm}}$ is the fraction of phases where $\langle g_m, g \rangle < 0$. Remaining columns report the mean phase-to-aggregate cosine, mean off-diagonal pairwise cosine, and the fraction of negative phase pairs ($f_{\mathrm{pair}}^{-}$).}
\label{tab:corrected-gradient-diagnostic}
\setlength{\tabcolsep}{5pt}
\begin{tabular}{rlrrrrrr}
\toprule
$M$ & Architecture & $A_g$ & $C_g$ & $f_{\mathrm{harm}}$ & $\overline{\cos(g_m,g)}$ & $\overline{\cos(g_m,g_k)}$ & $f_{\mathrm{pair}}^{-}$\\
\midrule
2  & Single+clock & 0.950 &  0.050 & 0.000 & 0.675 & -0.053 & 0.600 \\
5  & Single+clock & 0.970 &  0.030 & 0.000 & 0.399 & -0.012 & 0.520 \\
10 & Single+clock & 0.968 &  0.032 & 0.000 & 0.264 &  0.000 & 0.462 \\
20 & Single+clock & 1.046 & -0.046 & 0.020 & 0.195 &  0.002 & 0.485 \\
\midrule
2  & Multihead    & 1.003 & -0.003 & 0.000 & 0.706 &  0.003 & 0.600 \\
5  & Multihead    & 1.017 & -0.017 & 0.000 & 0.416 &  0.006 & 0.440 \\
10 & Multihead    & 1.041 & -0.041 & 0.020 & 0.279 &  0.007 & 0.462 \\
20 & Multihead    & 1.074 & -0.074 & 0.000 & 0.204 &  0.002 & 0.468 \\
\bottomrule
\end{tabular}
\end{table}

\begin{figure}[t]
\centering
\begin{subfigure}[b]{0.48\linewidth}
  \centering
  \includegraphics[width=\linewidth]{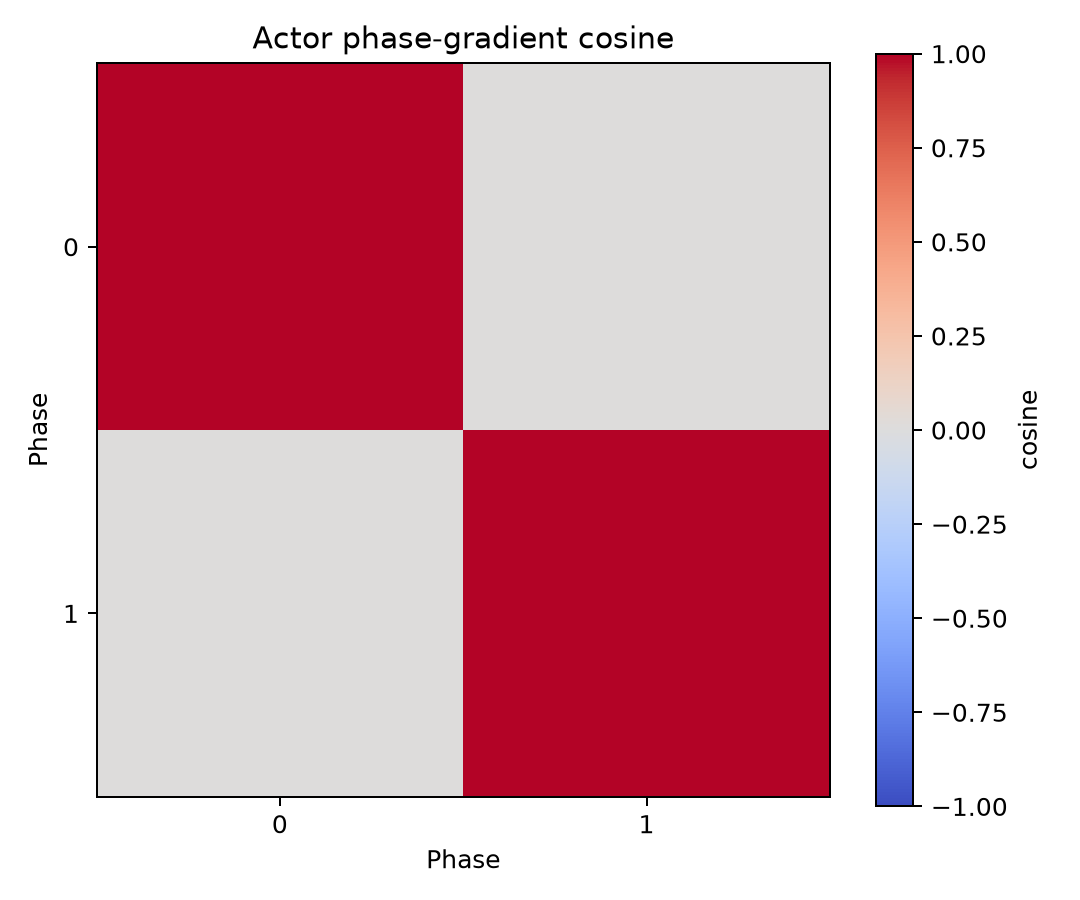}
  \caption{$M=2$, Single+clock}
  \label{fig:grad-m2}
\end{subfigure}\hfill
\begin{subfigure}[b]{0.48\linewidth}
  \centering
  \includegraphics[width=\linewidth]{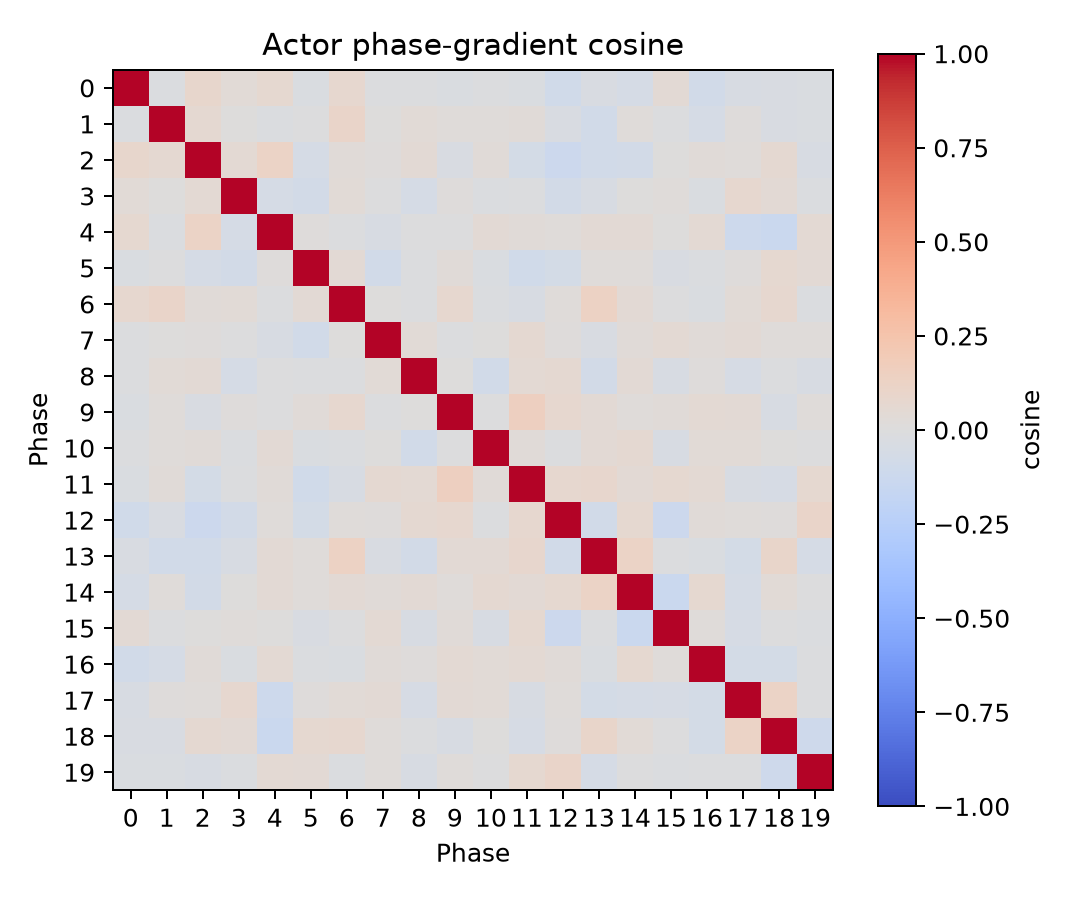}
  \caption{$M=20$, Single+clock}
  \label{fig:grad-m20}
\end{subfigure}
\caption{Representative single-seed actor-gradient cosine matrices for HAPS, illustrating qualitative phase geometry. A high fraction of negative pairwise cosines naturally arises in near-orthogonal high-dimensional gradient spaces and does not indicate that the aggregate update harms individual phases. Quantitative evaluations across five seeds are reported in Table~\ref{tab:corrected-gradient-diagnostic}.}
\label{fig:grad-main}
\end{figure}

\subsubsection{PhasePendulum Experiments Results}

PhasePendulum results confirm that the relative advantage of architectural specialization depends on task heterogeneity and sample availability. We control target separation by interpolating phase targets as $g_m(h) = h \, g_m^{\mathrm{max}}$, where $h \in [0, 1]$ parameterizes phase divergence. 

\paragraph{First Heterogeneity Test}
A first experiment is run with $M=10$ phases of duration $T=1000$ timesteps. The different results are presented in Tab.\ref{tab:het}. At $h = 0$ (identical phase targets), full parameter sharing via Single+clock solves the task efficiently in four of five seeds, whereas all ten specialists remain severely undertrained due to data fragmentation. At $h = 0.5$, shared and specialized architectures achieve comparable returns. At high heterogeneity ($h \in \{0.75, 1.00\}$), Multipolicy outperforms Multihead and Single+Clock across all five matched seeds. Its advantage over complete sharing, while positive in the majority of seeds, exhibits greater variance. This progression reflects a fundamental regime transition from data pooling to specialized parameter allocation as task divergence increases.

\begin{table}[t]
\centering
\small
\caption{PhasePendulum target-separation sweep ($M=10, T=1000$). Values denote mean reward per step across five seeds.}
\label{tab:pendulum-main}
\begin{tabular}{lccc}
\toprule
$h$ & Single+clock & Multihead & Multipolicy \\
\midrule
$0.00$ & \textbf{-1.370} & -2.785 & -6.007 \\
$0.25$ & \textbf{-2.185} & -2.592 & -5.107 \\
$0.50$ & -3.977 & \textbf{-3.844} & -3.977 \\
$0.75$ & -2.666 & -2.693 & \textbf{-2.614} \\
$1.00$ & -1.685 & -1.638 & \textbf{-1.596} \\
\bottomrule
\end{tabular}
\label{tab:het}
\end{table}

\paragraph{Budget allocation and entropy controls.}
To determine whether the failure of multipolicy at $h=0$ stems from fundamental optimization barriers or insufficient sample allocation, we performed budget and entropy sweeps (Fig. \ref{fig:pendulum-h0-learning}). Extending total interactions from $B=5\text{M}$ to $B=20\text{M}$ ($2\text{M}$ per specialist) fully resolves the undertraining deficit, allowing multipolicy to match the shared baseline. Crucially, learning curves show that rapid performance recovery occurs after surpassing the default $500\text{k}$ per-specialist threshold. Conversely, increasing the PPO entropy coefficient ($c_{\mathrm{ent}} = 0.05$) under $B=20\text{M}$ causes policy action variance to explode without restoring performance, confirming that the initial performance gap is a sample-allocation bottleneck rather than premature policy collapse. This is also visible in Fig.\ref{fig:pendulum-h0-learning-specialists} where we see that each local policy/specialist from multipolicy performs well under a single shared policy. The problem seems more linked to learning budget than to nonexistent coordination between phases.

\begin{figure}[t]
\centering
\begin{subfigure}[b]{0.55\linewidth}
\centering
\begin{tikzpicture}
\begin{axis}[
  width=\linewidth,height=1.55in,
  xlabel={Interactions per specialist (millions)},
  ylabel={Reward/step},
  xmin=0,xmax=2,ymin=-10,ymax=1,
  legend style={draw=none,font=\scriptsize,at={(0.98,0.03)},anchor=south east},
  tick label style={font=\scriptsize},label style={font=\scriptsize},
  grid=both,grid style={gray!15}
]
\addplot+[mark=o,thick,error bars/.cd,y dir=both,y explicit]
  table[x=specialist_interactions_millions,y=mean_raw_reward_per_step_mean,
  y error=mean_raw_reward_per_step_ci95,col sep=comma]
  {final_analysis/figure_data/pendulum_h0_budget_B5M_ent0p01.csv};
  \addlegendentry{$B=5\text{M}$}
\addplot+[mark=square,thick,error bars/.cd,y dir=both,y explicit]
  table[x=specialist_interactions_millions,y=mean_raw_reward_per_step_mean,
  y error=mean_raw_reward_per_step_ci95,col sep=comma]
  {final_analysis/figure_data/pendulum_h0_budget_B10M_ent0p01.csv};
  \addlegendentry{$B=10\text{M}$}
\addplot+[mark=triangle,thick,error bars/.cd,y dir=both,y explicit]
  table[x=specialist_interactions_millions,y=mean_raw_reward_per_step_mean,
  y error=mean_raw_reward_per_step_ci95,col sep=comma]
  {final_analysis/figure_data/pendulum_h0_budget_B20M_ent0p01.csv};
  \addlegendentry{$B=20\text{M}$}
\end{axis}
\end{tikzpicture}
\caption{Sample budget allocation ($h=0$)}
\label{fig:pendulum-budget}
\end{subfigure}\hfill
\begin{subfigure}[b]{0.42\linewidth}
\centering
\begin{tikzpicture}
\begin{axis}[
  width=\linewidth,height=1.55in,
  xlabel={Interactions per specialist (millions)},
  ylabel={Action std.\ dev.},
  xmin=0,xmax=2,ymode=log,ymin=0.25,ymax=10000,
  legend style={draw=none,font=\scriptsize,at={(0.02,0.98)},anchor=north west},
  tick label style={font=\scriptsize},label style={font=\scriptsize},
  grid=both,grid style={gray!15}
]
\addplot+[mark=triangle,thick]
  table[x=specialist_interactions_millions,y=mean_action_std_mean,col sep=comma]
  {final_analysis/figure_data/pendulum_h0_budget_B20M_ent0p01.csv};
  \addlegendentry{$c_{\mathrm{ent}}=0.01$}
\addplot+[mark=diamond,thick]
  table[x=specialist_interactions_millions,y=mean_action_std_mean,col sep=comma]
  {final_analysis/figure_data/pendulum_h0_budget_B20M_ent0p05.csv};
  \addlegendentry{$c_{\mathrm{ent}}=0.05$}
\end{axis}
\end{tikzpicture}
\caption{Entropy regularization ($B=20\text{M}$)}
\label{fig:pendulum-entropy}
\end{subfigure}
\caption{Budget scaling and entropy controls for Multipolicy in PhasePendulum under identical phase targets ($h=0$). \textbf{(a)} Learning curves across three total interaction budgets ($B$). Error bars denote two-sided $95\%$ Student-$t$ confidence intervals over five seeds. Significant gains emerge after surpassing $500\text{k}$ interactions per specialist. \textbf{(b)} Under $B=20\text{M}$, increasing the entropy coefficient ($c_{\mathrm{ent}}=0.05$) destabilizes action variance without improving returns (note log scale on right axis).}
\label{fig:pendulum-h0-learning}
\end{figure}
\begin{figure}[t]
\begin{tikzpicture}
\begin{axis}[
width=\linewidth,height=1.55in,
xlabel={Specialist phase $m$},ylabel={Reward/step},
xmin=-0.5,xmax=9.5,ymin=-10,ymax=1,
xtick={0,1,...,9},
legend style={draw=none,font=\scriptsize,at={(0.98,0.03)},anchor=south east},
tick label style={font=\scriptsize},label style={font=\scriptsize},
grid=both,grid style={gray!15}
]
\addplot+[only marks,mark=triangle*,thick,error bars/.cd,y dir=both,y explicit]
table[x=phase,y=mean_reward_per_step,y error=ci95_halfwidth,col sep=comma]
{final_analysis/figure_data/pendulum_h0_specialists.csv};
\addlegendentry{Specialists}
\addplot+[black,dashed,thick] coordinates {(-0.5,-1.3698) (9.5,-1.3698)};
\addlegendentry{Single+clock}
\end{axis}
\end{tikzpicture}
\caption{Learning starvation in PhasePendulum with ten identical phases
($h=0$).  Final phase-wise returns at the original $500$k allocation; every point is a mean over five seeds and the dashed line is the corresponding Single+clock mean. Error bars are two-sided $95\%$ Student-$t$ intervals. Doubling the local budget produces the largest recovery, while the low-budget failure is spread across all ten otherwise redundant specialists.}
\label{fig:pendulum-h0-learning-specialists}
\end{figure}
\paragraph{Phase count, dwell time, and reward weighting dynamics.}
Under a fixed $1$M per-phase interaction allocation ($T=1000$), multipolicy outperforms Single+clock in 19 out of 20 matched runs across $M \in \{2, 5, 10, 20\}$, yielding mean advantage gains of $+0.136$, $+0.119$, $+0.135$, and $+0.242$ reward per step. Paired bootstrap percentile intervals remain strictly positive across all $M$. However, performance scaling is non-monotone with respect to $M$, at $M=20$, multihead achieves a slightly higher mean score ($-2.236$ vs.\ $-2.260$).

Tab.~\ref{tab:pendulum-m-t-controls} summarizes phase-count, dwell-time, and target/reward weighting sweeps. In the dwell-time sweep at $M=10$, Multipolicy matches complete sharing at $T=500$ and secures the top score for $T \in \{1000, 2000\}$. At $M=20$, Multihead dominates short-to-medium horizons ($T \in \{500, 1000\}$), whereas Multipolicy leads at $T=2000$.

We then vary target spacing and reward weighting ($M=10, T=1000, 500\text{k}$/phase) to further illustrate these regime boundaries. The results are presented at the end of Tab. \ref{tab:pendulum-m-t-controls}. Concentrated or small-jump target configurations heavily favor shared representations, whereas maximum-jump configurations favor parameter separation (multipolicy). Assigning a dominant reward weight to Phase 0 eliminates the specialist advantage. This occurs because highly non-uniform rewards force the shared network to prioritize the high-value phase, relieving the capacity bottleneck that normally arises when balancing multiple competing task objectives.

\begin{table}[t]
\centering
\scriptsize
\caption{PhasePendulum ablation controls across phase counts ($M$), dwell times ($T$), and target/reward structures. The phase-count and dwell-time sweeps allocate $10^6$ interactions per phase; target/weighting sweeps allocate $500\text{k}$ per phase. $n$ denotes random seed count.}
\label{tab:pendulum-m-t-controls}
\begin{tabular}{llrrrr}
\toprule
Control & Setting & $n$ & Single+clock & Multihead & Multipolicy \\
\midrule
Phase count, $T=1000$ & $M=2$  & 5 & -1.888 & -1.917 & \textbf{-1.752} \\
                      & $M=5$  & 5 & -1.633 & -1.646 & \textbf{-1.514} \\
                      & $M=10$ & 5 & -1.688 & -1.640 & \textbf{-1.553} \\
                      & $M=20$ & 5 & -2.502 & \textbf{-2.236} & -2.260 \\
\midrule
Dwell time, $M=10$ & $T=500$  & 3 & \textbf{-1.567} & -1.660 & -1.572 \\
                    & $T=1000$ & 3 & -1.637 & -1.639 & \textbf{-1.564} \\
                    & $T=2000$ & 3 & -1.766 & -1.697 & \textbf{-1.548} \\
\midrule
Dwell, $B=5$M & $M=5,T=100$  & 3 & \textbf{-1.580} & -2.021 & -1.672 \\
               & $M=5,T=200$  & 3 & \textbf{-1.545} & -1.922 & -1.636 \\
               & $M=5,T=500$  & 3 & -1.567 & -1.706 & \textbf{-1.530} \\
               & $M=10,T=100$ & 3 & \textbf{-1.586} & -1.885 & -1.779 \\
               & $M=10,T=200$ & 3 & \textbf{-1.584} & -1.816 & -1.613 \\
               & $M=10,T=500$ & 3 & -1.645 & -1.624 & \textbf{-1.586} \\
\midrule
High dwell, $B_0=1$M & $M=10,T=500$  & 3 & \textbf{-1.567} & -1.660 & -1.572 \\
                      & $M=10,T=1000$ & 3 & -1.637 & -1.639 & \textbf{-1.564} \\
                      & $M=10,T=2000$ & 3 & -1.766 & -1.697 & \textbf{-1.548} \\
                      & $M=20,T=500$  & 3 & -2.421 & \textbf{-1.844} & -2.377 \\
                      & $M=20,T=1000$ & 3 & -2.461 & \textbf{-2.136} & -2.310 \\
                      & $M=20,T=2000$ & 3 & -2.570 & -2.585 & \textbf{-2.370} \\
\midrule
Targets / weights & Identical targets, uniform weights & 3 & \textbf{-1.631} & -2.419 & -6.174 \\
                & Small-jump targets, uniform weights & 3 & \textbf{-0.360} & -2.625 & -5.185 \\
                & Maximum-jump targets, uniform weights & 3 & -1.601 & -1.648 & \textbf{-1.584} \\
                & Maximum-jump targets, Phase 0 dominant & 3 & -2.030 & \textbf{-2.019} & -2.330 \\
\bottomrule
\end{tabular}
\end{table}

\subsubsection{GridWorld Experiment Results}
\label{sec:gridworld-results}

\paragraph{One-phase budget calibration.}
To establish a rigorous sample-efficiency baseline for phase specialists, we first calibrated the required interaction budget in a single-phase environment ($M=1, T=100$). Because single-phase environments offer no benefit from parameter specialization, this experiment isolates the sample complexity required by the underlying PPO learner to cross a predefined performance threshold. 

As reported in Table~\ref{tab:gridworld-one-phase-budget}, Single+clock achieves threshold performance across all three seeds at $500\text{k}$ interactions, whereas the multi-policy architecture requires $750\text{k}$ interactions. Based on this calibration, all subsequent multi-phase PhaseGridWorld experiments allocate a conservative minimum budget of $1\text{M}$ interactions per specialist ($B_0 \ge 1\text{M}$) to prevent sample starvation. At first glance, it is surprising that the multipolicy condition learns more slowly than the shared policy when $M=1$, since both use a single PPO learner. However, the comparison is not fully matched: the shared policy receives within-phase clock information, whereas the standard multipolicy condition does not. The clock does not introduce randomness directly, but it allows the policy to select different actions at the same physical state at different phase ages. This can help the agent escape repetitive behaviors, such as repeatedly attempting to move into a wall, and may partly explain the higher final performance of the clock-conditioned policy. Table~\ref{tab:gridworld-distractors} supports this interpretation: without distractors, adding the clock improves multipolicy performance relative to the standard multipolicy condition.

\begin{table}[t]
\centering
\small
\caption{Single-phase GridWorld budget calibration ($M=1, T=100$, three seeds). ``Reached'' denotes the proportion of seeds exceeding the performance threshold. Values are reported as mean reward per step.}
\label{tab:gridworld-one-phase-budget}
\begin{tabular}{rccrr}
\toprule
Budget & Method & Reached & Best mean & Final mean \\
\midrule
$50\text{k}$  & Multipolicy-style & $0/3$ & $-0.0400$ & $-0.0585$ \\
$50\text{k}$  & Single+clock      & $0/3$ & $-0.0410$ & $-0.0567$ \\
$100\text{k}$ & Multipolicy-style & $0/3$ & $-0.0417$ & $-0.0557$ \\
$100\text{k}$ & Single+clock      & $0/3$ & $-0.0420$ & $-0.0557$ \\
$200\text{k}$ & Multipolicy-style & $0/3$ & $-0.0375$ & $-0.0525$ \\
$200\text{k}$ & Single+clock      & $0/3$ & $-0.0461$ & $-0.0502$ \\
$300\text{k}$ & Multipolicy-style & $0/3$ & $-0.0310$ & $-0.0438$ \\
$300\text{k}$ & Single+clock      & $1/3$ & $-0.0003$ & $-0.0048$ \\
$500\text{k}$ & Multipolicy-style & $1/3$ & $-0.0040$ & $-0.0386$ \\
$500\text{k}$ & Single+clock      & $3/3$ & $0.1223$  & $0.1095$  \\
$750\text{k}$ & Multipolicy-style & $3/3$ & $0.1452$  & $0.0791$  \\
$750\text{k}$ & Single+clock      & $3/3$ & $0.1745$  & $0.1547$  \\
$1\text{M}$   & Multipolicy-style & $3/3$ & $0.1350$  & $0.0208$  \\
$1\text{M}$   & Single+clock      & $3/3$ & $0.1807$  & $0.1683$  \\
\bottomrule
\end{tabular}
\end{table}

\paragraph{Robustness to distractor dimensions.}
In order to evaluate how representation space capacity affects phase routing, we conducted a distractor sweep at $M=10, T=100$ by appending $Z$ task-irrelevant, independent observation dimensions. As shown in Table~\ref{tab:gridworld-distractors}, Single+clock maintains high performance up to $Z=32$, but degrades sharply at $Z=128$ ($0.0629$). Conversely, Multihead routing and warm-started clock-aware specialists demonstrate superior representation robustness at $Z=128$, achieving $0.0968$ and $0.1178$, respectively. This drop in shared policy performance is consistent with phase-signal neglect under high-dimensional noise, though distractors also increase the baseline representation burden across all modular models.

\begin{table}[t]
\centering
\small
\caption{PhaseGridWorld distractor sweep ($M=10, T=100$). Values denote mean reward per step across three seeds under $Z$ reward-irrelevant noise dimensions (higher is better).}
\label{tab:gridworld-distractors}
\begin{tabular}{rccccc}
\toprule
$Z$ & Single+clock & Multihead & Multipolicy & Multi+clock & Warm Multi+clock \\
\midrule
0   & \textbf{0.1422} & 0.0980 & 0.1210 & 0.1258 & 0.1356 \\
8   & \textbf{0.1427} & 0.1197 & 0.1048 & 0.1201 & 0.1319 \\
32  & \textbf{0.1359} & 0.1197 & 0.0836 & 0.1082 & 0.1312 \\
128 & 0.0629          & 0.0968 & 0.0544 & 0.0732 & \textbf{0.1178} \\
\bottomrule
\end{tabular}
\end{table}

\paragraph{Obstacle Density}
The fixed-total-budget density sweep in Appendix Table~\ref{tab:gridworld-density-sweep} provides a complementary control. At $D=0.40$, Single+clock leads in every tested $(M,T)$ setting when the total budget is only $3$M interactions. By contrast, clock-aware specialization is best at $D=0.40$ once each specialist receives $1.5$M interactions. Obstacle density is therefore not an architecture-selection criterion by itself: its effect depends jointly on phase-local data, clock information, and the topology induced by the sampled walls.

\begin{table}[H]
\centering
\scriptsize
\caption{Exploratory PhaseGridWorld obstacle-density sweep under a fixed $3$M-interaction total budget. Entries are mean reward per step over three matched seeds; higher is better. The specialist budget is approximately $3\text{M}/M$. Multipolicy and warm multipolicy do not receive clock features in this earlier sweep.}
\label{tab:gridworld-density-sweep}
\begin{tabular}{cccrrrr}
\toprule
Density & $M$ & $T$ & Single+clock & Multihead & Multipolicy & Warm multipolicy \\
\midrule
0.10 & 5  & 50  & 0.1169 & 0.1252 & 0.0917 & \textbf{0.1253} \\
0.10 & 5  & 100 & 0.1273 & 0.1192 & 0.0633 & \textbf{0.1331} \\
0.10 & 10 & 50  & 0.0861 & \textbf{0.0991} & 0.0798 & 0.0969 \\
0.10 & 10 & 100 & 0.0985 & 0.1031 & 0.0445 & \textbf{0.1272} \\
\midrule
0.25 & 5  & 50  & \textbf{0.1008} & 0.0704 & 0.0653 & 0.0992 \\
0.25 & 5  & 100 & 0.0901 & 0.0443 & 0.0261 & \textbf{0.1009} \\
0.25 & 10 & 50  & \textbf{0.1275} & 0.0571 & 0.0375 & 0.0960 \\
0.25 & 10 & 100 & 0.0750 & 0.0686 & 0.0326 & \textbf{0.0813} \\
\midrule
0.40 & 5  & 50  & \textbf{0.1162} & 0.0730 & 0.0502 & 0.0629 \\
0.40 & 5  & 100 & \textbf{0.1069} & 0.0440 & 0.0490 & 0.0533 \\
0.40 & 10 & 50  & \textbf{0.1254} & 0.0788 & 0.0065 & -0.0140 \\
0.40 & 10 & 100 & \textbf{0.0887} & 0.0552 & -0.0121 & 0.0113 \\
\bottomrule
\end{tabular}
\end{table}

\paragraph{Specialist controls and sample allocation.}
Table~\ref{tab:gridworld-specialist-controls} summarizes extended per-phase budget runs ($B_0 \in \{1\text{M}, 1.5\text{M}\}$) across varying grid densities and phase counts. At $B_0=1.5\text{M}$ and $M=10$, Multi+clock matches complete sharing at grid density $0.25$ ($0.0992$ vs.\ $0.0991$) and outperforms it at density $0.40$ ($0.1797$ vs.\ $0.1486$). Warm initialization (``Warm+clock'') provides gains at low density ($0.25$), whereas direct transition-aware warm starts underperform across all settings. These controls indicate that while clock availability and parameter initialization improve specialist performance, sample starvation alone does not fully explain modular failure modes in multi-phase grid worlds.

\begin{table}[t]
\centering
\small
\caption{PhaseGridWorld specialist controls at $T=500$. Values denote mean reward per step across three seeds. $B_0$ is the nominal interaction budget per phase. ``Warm+clock'' initializes specialists from a shared pretrained model prior to fine-tuning.}
\label{tab:gridworld-specialist-controls}
\setlength{\tabcolsep}{4pt}
\begin{tabular}{cccrrrrrr}
\toprule
$B_0$ & Density & $M$ & Single+clock & Multihead & Multipolicy & Multi+clock & Warm+clock & Transition warm \\
\midrule
$1\text{M}$   & 0.25 & 5  & 0.0836 & 0.0514 & 0.0200 & 0.0940 & 0.1124 & -- \\
$1\text{M}$   & 0.25 & 10 & 0.0878 & 0.1228 & 0.0916 & 0.0926 & 0.1179 & -- \\
$1\text{M}$   & 0.40 & 5  & 0.1508 & 0.2099 & 0.1924 & 0.2372 & 0.2323 & -- \\
$1\text{M}$   & 0.40 & 10 & 0.1480 & 0.1344 & 0.1772 & 0.1651 & 0.1681 & -- \\
$1.5\text{M}$ & 0.25 & 10 & 0.0991 & 0.0826 & 0.0791 & 0.0992 & \textbf{0.1155} & -0.0164 \\
$1.5\text{M}$ & 0.40 & 10 & 0.1486 & 0.1404 & 0.1606 & \textbf{0.1797} & 0.1695 & 0.0539 \\
\bottomrule
\end{tabular}
\end{table}

\paragraph{Two-phase coordination failure and the handoff trap.}
Because independent specialists optimize phase-local returns from phase-local reset states, they receive no training signal regarding downstream continuation values. In contrast, a trajectory-level shared policy can trade off immediate phase rewards to optimize downstream start states.

Fig.~\ref{fig:gridworld-handoff-schematic} formalizes this trade-off in a controlled two-phase GridWorld. In Phase 0, the agent chooses between an immediate goal ($g_{00}$, reward $1.0$) and a safe handoff goal ($g_{01}$, reward $\rho < 1.0$). Reaching $g_{00}$ leaves the agent in a disconnected component in Phase 1 behind a closed wall, rendering the Phase 1 goal unreachable. Scripted controls verify that choosing $g_{01}$ preserves Phase 1 reachability across all $79$ valid start states, improving net return by $\rho$.

As presented in Figure~\ref{fig:gridworld-handoff}. When evaluating learned policies, shared architectures adapt to downstream value trade-offs as $\rho$ increases: at $\rho=0.9$, Single+clock and Multihead choose the safe handoff in $96\%$ and $98\%$ of episodes, respectively. In contrast, independent specialists (Multipolicy+clock) select the safe handoff in only $10\%$ of episodes, continuing to greedily pursue the larger Phase 0 local reward $g_{00}$. This confirms that uncoordinated specialist ensembles suffer from myopic action selection at phase boundaries.

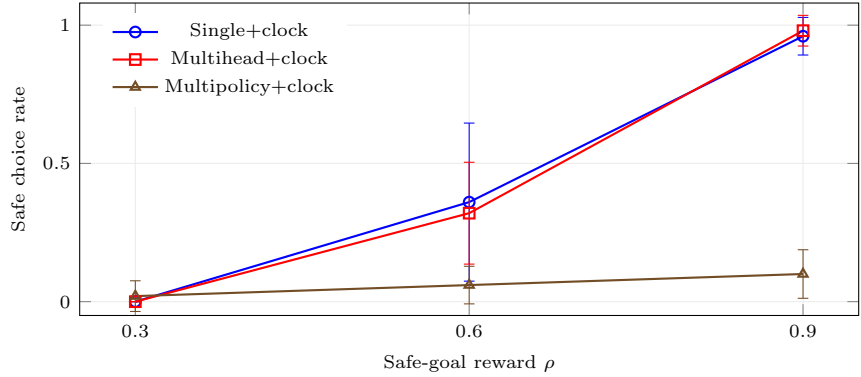
\begin{figure}[H]
\centering
\begin{tikzpicture}
\begin{axis}[
width=0.72\linewidth,height=2.25in,
xlabel={Safe-goal reward $\rho$},ylabel={Safe choice rate},
xmin=0.25,xmax=0.95,ymin=-0.05,ymax=1.08,
xtick={0.3,0.6,0.9},ytick={0,0.5,1},
legend style={draw=none,font=\scriptsize,at={(0.03,0.97)},anchor=north west},
tick label style={font=\scriptsize},label style={font=\scriptsize},
grid=both,grid style={gray!15}
]
\addplot+[mark=o,thick,error bars/.cd,y dir=both,y explicit]
table[x=rho,y=single_phase_tau,y error=single_phase_tau_ci95,col sep=comma]
{final_analysis/figure_data/gridworld_goal_choice.csv};
\addlegendentry{Single+clock}
\addplot+[mark=square,thick,error bars/.cd,y dir=both,y explicit]
table[x=rho,y=multihead_tau,y error=multihead_tau_ci95,col sep=comma]
{final_analysis/figure_data/gridworld_goal_choice.csv};
\addlegendentry{Multihead+clock}
\addplot+[mark=triangle,thick,error bars/.cd,y dir=both,y explicit]
table[x=rho,y=multipolicy_tau,y error=multipolicy_tau_ci95,col sep=comma]
{final_analysis/figure_data/gridworld_goal_choice.csv};
\addlegendentry{Multipolicy+clock}
\end{axis}
\end{tikzpicture}
\caption{Learned safe-handoff choices. Shared methods increasingly choose the
safe phase-0 goal as $\rho$ rises, whereas independently trained specialists
continue to prefer the immediate phase-local reward.
Means and two-sided $95\%$ Student-$t$ intervals over five seeds.}
\label{fig:gridworld-handoff}
\end{figure}

\paragraph{Causal boundary interventions and recovery dynamics.}
To isolate boundary state distribution shifts from intrinsic phase-local policy quality, we evaluated trained checkpoints under two distinct phase-entry conditions: physical state carryover versus resampling from reachable entry states (Fig.~\ref{fig:gridworld-handoff-topologies}).

Fig. \ref{fig:gridworld-boundary-intervention}, presents the reward gap between multipolicy and single policy in the two cases when the transition is carried on from one phase to the other or when there is a reachable state reset. Those comparisons are presented for the three already presented phase-structured gridworlds (Aligned, Bottleneck, Closed trap).
In the irrecoverable ``Closed trap'' setting, state carryover imposes a severe performance penalty on specialists (a gap of $-0.0798$ relative to shared policies). However, resetting phase entry to a common reachable state eliminates the handoff penalty, reversing the net performance gap to $+0.0060$. Conversely, in the recoverable ``Bottleneck'' environment (where a distant gate permits recovery), specialization outperforms sharing under both physical carryover ($+0.0681$) and reachable resets ($+0.0728$).

These interventions validate the regime decomposition framework in \eqref{eq:main-regime-decomposition}: recoverable transitions induce a finite transient recovery cost ($L_m$), enabling settled specialist gains ($G_m$) to dominate long-horizon returns. Unrecoverable transitions cause effective transient costs to span the entire phase duration ($L_m = T_m$), preventing boundary cost amortization. It also supports the effect of heterogeneity, as multipolicy, as a specialized solution, outperforms a single policy in ``Bottleneck," composed of two different phases, but underperforms in ``Aligned," where there is no difference between the phases.

\begin{figure}[t]
\centering
\begin{subfigure}[b]{0.56\linewidth}
\centering
\begin{tikzpicture}
\begin{axis}[
width=\linewidth,
height=1.65in,
xmin=-0.45,xmax=2.45,
ymin=-0.13,ymax=0.10,
xtick={0,1,2},
xticklabels={Aligned, Bottleneck, Closed trap},
ylabel={Specialist $-$ shared reward/step},
legend style={draw=none,font=\scriptsize,at={(0.5,1.02)},anchor=south,legend columns=2},
tick label style={font=\scriptsize},
label style={font=\scriptsize},
grid=major,
grid style={gray!15}
]
\addplot[black,thin,forget plot] coordinates {(-0.45,0) (2.45,0)};
\addplot+[only marks,mark=o,thick,xshift=-4pt,error bars/.cd,y dir=both,y explicit]
table[x=x,y=coupled,y error=coupled_ci,col sep=comma]
{figures_data/gridworld_boundary_intervention.csv};
\addlegendentry{Carried state}
\addplot+[only marks,mark=square,thick,xshift=4pt,error bars/.cd,y dir=both,y explicit]
table[x=x,y=reset,y error=reset_ci,col sep=comma]
{figures_data/gridworld_boundary_intervention.csv};
\addlegendentry{Reachable-state reset}
\end{axis}
\end{tikzpicture}
\caption{Reward gap across entry laws}
\label{fig:gridworld-intervention-a}
\end{subfigure}\hfill
\begin{subfigure}[b]{0.41\linewidth}
\centering
\begin{tikzpicture}
\begin{axis}[
width=\linewidth,
height=1.65in,
xmin=-0.45,xmax=2.45,
ymin=-0.015,ymax=0.13,
xtick={0,1,2},
xticklabels={Aligned,Bottleneck,Trap},
x tick label style={font=\scriptsize,rotate=18,anchor=east},
ylabel={Reset $-$ carried gap},
tick label style={font=\scriptsize},
label style={font=\scriptsize},
grid=major,
grid style={gray!15}
]
\addplot[black,thin,forget plot] coordinates {(-0.45,0) (2.45,0)};
\addplot+[only marks,mark=diamond*,thick,error bars/.cd,y dir=both,y explicit]
table[x=x,y=interaction,y error plus=interaction_ci_plus,
  y error minus=interaction_ci_minus,col sep=comma]
{figures_data/gridworld_boundary_intervention.csv};
\end{axis}
\end{tikzpicture}
\caption{Difference-in-differences}
\label{fig:gridworld-intervention-b}
\end{subfigure}
\caption{Controlled boundary intervention in PhaseGridWorld ($M=2, T=100$). \textbf{(a)} Specialist-minus-shared returns under physical state carryover versus reachable-state resets. \textbf{(b)} Corresponding difference-in-differences. Points represent means; error bars denote $95\%$ Student-$t$ CIs (left) and paired bootstrap CIs (right). Removing inherited boundary states reverses method rankings in the closed trap while preserving specialist dominance in the bottleneck.}
\label{fig:gridworld-boundary-intervention}
\end{figure}
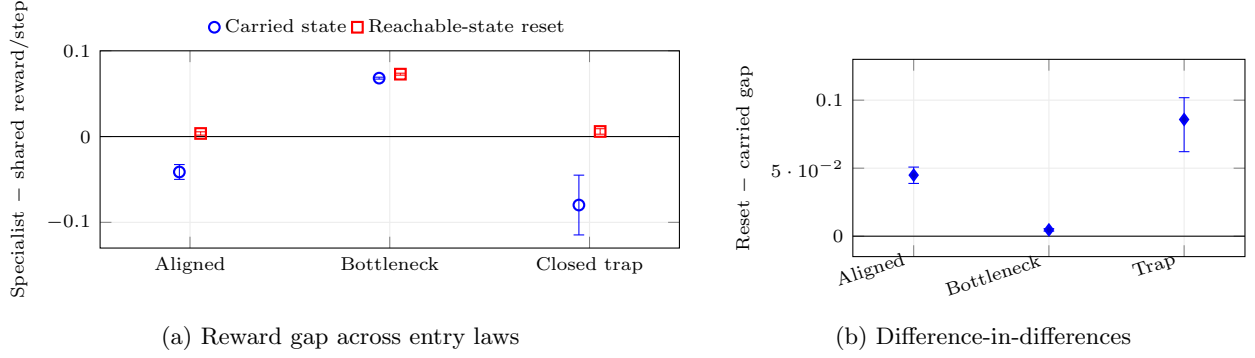

\paragraph{Empirical transient--settled decomposition.}
Finally, Table~\ref{tab:gridworld-gk-decomposition} reports the empirical decomposition into settled gains ($G$) and transient handoff costs ($K$) under a preregistered window split $L=50$ ($T=100$). Under state carryover, both the Aligned and Closed Trap topologies exhibit positive transient costs ($K > 0$) and negative late-window specialist gaps ($G < 0$). Re-evaluating policies under a common reachable reset eliminates transient costs ($K \approx 0$) and restores positive net gains ($G > 0$). %These empirical breakdowns closely reflect the theoretical bounds established in Proposition~\ref{prop:reference-regime-bounds}.

\begin{table}[t]
\centering
\small
\caption{Empirical transient--late-window decomposition in PhaseGridWorld ($L=50, T=100$). Values denote means across five matched training seeds evaluated under $\Delta\bar J = (1-\lambda)G - \lambda K$ with $\lambda=0.5$. ``Carried'' retains physical states across boundaries; ``Reset'' uses a common reachable entry law.}
\label{tab:gridworld-gk-decomposition}
\begin{tabular}{llrrr}
\toprule
Topology & Entry law & $G$ (Settled) & $K$ (Transient) & Specialist $-$ Shared \\
\midrule
Aligned    & Carried & $-0.0375$ & $0.0450$  & $-0.0413$ \\
Aligned    & Reset   & $0.0065$  & $-0.0007$ & $0.0036$  \\
Bottleneck & Carried & $0.0746$  & $-0.0616$ & $0.0681$  \\
Bottleneck & Reset   & $0.0771$  & $-0.0684$ & $0.0728$  \\
Closed trap& Carried & $-0.0754$ & $0.0842$  & $-0.0798$ \\
Closed trap& Reset   & $0.0101$  & $-0.0018$ & $0.0060$  \\
\bottomrule
\end{tabular}
\end{table}

\end{document}